\documentclass[11pt]{article} 
\def\comments{0}
\def\todo{0}

\def\coltshort{0}

\usepackage{sty/preamble}

\newcommand{\users}{\mathify{t}}
\newcommand{\samples}{\mathify{n}}
\newcommand{\datadim}{d}
\newcommand{\feat}{x}
\newcommand{\map}{z}
\newcommand{\lab}{y}
\newcommand{\featdom}{\cX}

\newcommand{\mapdom}{\cZ}

\newcommand{\pop}{P}
\newcommand{\dist}{\pop}
\newcommand{\distint}{B}
\newcommand{\metadist}{\mathfrak{Q}}

\newcommand{\sgn}{\text{sign}}
\newcommand{\growth}{\mathcal{G}}

\newcommand{\alg}{\cA}

\newcommand{\conc}{g}
\newcommand{\repclass}{\cH}
\newcommand{\rep}{h}

\newcommand{\repdim}{   k}
\newcommand{\perclass}{\cF}
\newcommand{\per}{f}
\newcommand{\realpclass}{\cR}
\newcommand{\realp}{r}
\newcommand{\wit}{m}
\newcommand{\vc}{v}
\newcommand{\pseudod}{\psi}
\newcommand{\er}{\textrm{err}}
\newcommand{\rer}{\textrm{rep-err}}
\newcommand{\erind}{q}
\newcommand{\erindclass}{\cQ}
\newcommand{\pd}{\textrm{PDim}}

\newcommand{\npers}{n_{\text{spec}}}

\newcommand{\nrcc}{non-realizability-certificate complexity}
\newcommand{\Nrcc}{Non-realizability-certificate complexity}
\newcommand{\NRCC}{Non-Realizability-Certificate Complexity}
\newcommand{\NRC}{\mathrm{NRC}}

\title{Metalearning with Very Few Samples Per Task}
\usepackage{times}

\ifnum\coltshort=0

    \author{
    Maryam Aliakbarpour\thanks{Department of Computer Science, Rice University.  Part of this work was done while at the Department of Computer Science at Boston University and the Khoury College of Computer Sciences at Northeastern University.} \and
    Konstantina Bairaktari\thanks{Khoury College of Computer Sciences, Northeastern University.} \and
    Gavin Brown\thanks{Paul G.\ Allen School of Computer Science and Engineering, University of Washington.  Part of this work was done while at the Department of Computer Science at Boston University.} \and
    Adam Smith\thanks{Department of Computer Science, Boston University.} \and
    Nathan Srebro\thanks{Toyota Technological Institute at Chicago.} \and
    Jonathan Ullman\thanks{Khoury College of Computer Sciences, Northeastern University.}
    }
\else

\coltauthor{
\Name{Maryam Aliakbarpour}\\
\addr Department of Computer Science, Rice University
\AND
\Name{Konstantina Bairaktari}\\
\addr Khoury College of Computer Sciences, Northeastern University
\AND
\Name{Gavin Brown}\\
\addr Paul G.\ Allen School of Computer Science and Engineering, University of Washington.
\AND 
\Name{Adam Smith}\\
\addr Department of Computer Science, Boston University.
\AND
\Name{Jonathan Ullman}\\
\addr Khoury College of Computer Sciences, Northeastern University
}
\fi
\date{}

\begin{document}

\ifnum\todo=1
{\color{blue} {\Large \bfseries TODO:}
\begin{itemize}
    \item
\end{itemize}
\fi

\maketitle

\begin{abstract}
    Metalearning and multitask learning are two frameworks for solving a group of related learning tasks more efficiently than we could hope to solve each of the individual tasks on their own.  In multitask learning, we are given a fixed set of related learning tasks and need to output one accurate model per task, whereas in metalearning we are given tasks that are drawn i.i.d.\ from a metadistribution and need to output some common information that can be easily specialized to new, previously unseen tasks from the metadistribution.
    
    In this work, we consider a binary classification setting where tasks are related by a \emph{shared representation}, that is, every task $P$ of interest can be solved by a classifier of the form $\per_{P} \circ \rep$ where $\rep \in \repclass$ is a map from features to some representation space that is shared across tasks, and $\per_{P} \in \perclass$ is a task-specific classifier from the representation space to labels.  The main question we ask in this work is \emph{how much data do we need to metalearn a good representation?}  Here, the amount of data is measured in terms of both the number of tasks $\users$ that we need to see and the number of samples $\samples$ per task. We focus on the regime where the number of samples per task is extremely small.

    Our main result shows that, in a distribution-free setting where the feature vectors are in $\R^d$, the representation is a linear map from $\R^d \to \R^k$, and the task-specific classifiers are halfspaces in $\R^k$, we can metalearn a representation with error $\eps$ using just $\samples = k+2$ samples per task, and $d \cdot (1/\eps)^{O(k)}$ tasks.  Learning with so few samples per task is remarkable because metalearning would be impossible with $k+1$ samples per task, and because we cannot even hope to learn an accurate task-specific classifier with just $k+2$ samples per task.  To obtain this result, we develop a sample-and-task-complexity theory for distribution-free metalearning and multitask learning, which identifies what properties of $\perclass$ and $\repclass$ make metalearning possible with few samples per task.  Our theory also yields a simple characterization of distribution-free multitask learning.  Finally, we give sample-efficient reductions between metalearning and multitask learning, which, when combined with our characterization of multitask learning, give a characterization of metalearning in certain parameter regimes.

\end{abstract}
\ifnum \coltshort = 1
\fi
\ifnum \coltshort = 0
\newpage

\tableofcontents
\newpage
\fi
\pagenumbering{arabic}

\section{Introduction}
Metalearning and multitask learning are frameworks for solving a group of related learning tasks more efficiently than we could hope to solve each of the individual tasks on their own.  It is convenient to think of tasks as corresponding to people—for example, each task could involve completing sentences of one person's email or recognizing one person's friends' faces in photographs. These tasks are clearly different from person to person, since people have different writing styles and friends, but there is also a great deal of shared structure in both examples.  Even though each person may not supply enough text messages or photographs to build an accurate model, one can leverage shared structure to solve all of the tasks more effectively.

In these frameworks (specialized to binary classification), we model each task as a distribution $P$ over labeled examples $\featdom \times \bits$. A family of tasks can be either a finite list of tasks $P_1,\dots,P_t$ (in multitask learning) or, more generally, a metadistribution $\metadist$ over tasks  (in metalearning).  There are several ways to model the relationships among a family of tasks (see Section~\ref{sec:rw}). In this paper, we assume that the tasks are well labeled by classifiers that share a common \textit{representation} $h:\featdom\to \mapdom$ for some  space $\mapdom$, usually of lower dimension than $\featdom$. For each task $P_j$, there is a specialized classifier $f_j:\mapdom\to \bits$ such that $f_j\circ h$ has high accuracy on $P_j$. For example, when recognizing faces in photographs, a common approach uses one neural network to extract facial features followed by a final layer that differs for each person and maps facial features to names of that person's friends.

Given classes $\repclass$ of possible representations and $\perclass$ of possible specializations,  we ask \emph{how much data do we need to find a representation that performs well for most of the tasks in our family of interest?} Suppose the learner has access to $\users$ data sets, where the $j$-th dataset consists of $\samples$ observations drawn i.i.d.\ from task $P_j$. Following the literature, we consider two basic objectives:

\ifnum\coltshort=0
\begin{quote}
    \textbf{(Proper%
        \footnote{For simplicity, we describe 
        the \emph{proper} version of multitask and meta learning where we have to output a representation from $\repclass$ and a set of specialized classifiers from $\perclass$.  In this paper we also consider a more general notions of improper multitask and meta learning.}) 
    Multitask learning:} Output a representation $\rep \in \repclass$ and specialized classifiers $\per_1,\dots,\per_\users \in \perclass$ such that the error of the classifier $\per_j \circ \rep$ on its task $P_j$ is low on average over the tasks. For example, this objective corresponds to learning the classifiers for a given set of face-recognition tasks. 
\end{quote}
\begin{quote}
    \textbf{(Proper\footnotemark[\value{footnote}]) Metalearning:} Assuming that the tasks $P_1,...,P_\users$ are themselves drawn i.i.d.\ from an unknown metadistribution $\metadist$, output a representation $\hat h \in \repclass$ that can then be specialized to an unseen task $P$ drawn from $\metadist$.  The benchmark is then the 
    expected error of the best representation $h^*$ on a new task $\dist\sim \metadist$, measured using the best 
    specialized classifier $\per\in\perclass$ for that task.  This objective corresponds to finding an embedding that can be used to quickly build face-recognition classifiers for new people, as opposed to the people whose photos were used for training. 
\end{quote}
\else
    \smallskip
    \textbf{(Proper%
        \footnote{For simplicity, and to highlight the relationship with metalearning, we describe 
        the \emph{proper} version of multitask/meta learning, but we also consider a more general notion of improper multitask/meta learning.}) 
    Multitask learning:} Output a representation $\rep \in \repclass$ and specialized classifiers $\per_1,\dots,\per_\users \in \perclass$ such that the error of the classifier $\per_j \circ \rep$ on its task $P_j$ is low on average over the tasks. For example, this objective corresponds to learning the classifiers for a given set of face-recognition tasks. 
    \smallskip

    \textbf{\adamtext{(Proper) }Metalearning:} Assuming that the tasks $P_1,...,P_\users$ are themselves drawn i.i.d.\ from an unknown metadistribution $\metadist$, output a representation $\hat h \in \repclass$ that can then be specialized to an unseen task $P$ drawn from $\metadist$.  The benchmark is then the 
    expected error of the best representation $h^*$ on a new task $\dist\sim \metadist$, measured using the best 
    specialized classifier $\per\in\perclass$ for that task.  This objective corresponds to finding an embedding that can be used to quickly build face-recognition classifiers for new people, as opposed to for people whose photos were used for training. 
    \smallskip
\fi

For both variants, one can consider \emph{realizable} families of tasks or a more general \emph{agnostic setting}.  In the realizable case, there is a representation $h \in \repclass$ such that, for every task $P$ in the family, we can find a classifier $f_P \in \perclass$ that perfectly labels example from $P$.  In the agnostic setting, we have no such promise; we aim to compete with the benchmark no matter how good or bad it is.

\mypar{The sample and task complexity of meta and multitask learning.} The goal of this paper is to understand the sample complexity of multitask and metalearning.  Specifically
\begin{quote}
    \emph{How many samples $\samples$ do we need per task to solve metalearning and multitask learning and, given a number of samples $\samples$ per task, how many tasks $\users$ do we need to see?}
\end{quote}

To get some intuition for this question, it is easy to see that in the multitask learning setting we need at least enough samples from each task to learn a good specialized classifier.  For realizable learning, in the distribution-free setting that we study, the number of samples per task must be at least $n \geq \VC(\perclass)/\eps$ to learn with error $\eps$.  In contrast, metalearning's only goal is to output a representation that can be specialized to new tasks, and there is no need to output a good classifier for any specific task.  
Thus we will focus on a more specific question 
\begin{quote}
    \emph{What is the minimum number of samples per task $\samples$ that enable us to solve metalearning with a finite number of tasks $\users$?}
\end{quote}

Our main result answers this question \emph{exactly} for the case of \emph{halfspace classifiers} with a shared \emph{linear representation} in the \emph{distribution-free, realizable setting}.  That is, when  examples are in $\R^d$, representations are linear functions from $\R^d \to \R^k$, and specialized classifiers are halfspaces on $\R^k$.  Our result shows that we can metalearn this class of functions with $k+2$ samples per task. Metalearning  with $k+1$ samples per task is impossible in the distribution-free setting.  For constant error, we show that the number of tasks required to learn the representation using $k+2$ samples per task is at most linear in $d$ and exponential in $k$.

To prove our main result, we develop a uniform-convergence theory for distribution-free metalearning and multitask learning.  This theory complements prior work, which developed a theory of meta- and multitask learning under specific assumptions about the task distributions (see \Cref{sec:rw}).  The main novel feature of our uniform-convergence theory is that it enables us to prove bounds on the number of tasks required for metalearning \emph{even in the regime where we lack the samples per task needed to accurately evaluate the quality of a representation.}  That is, given a representation $\rep$ and just $\samples = \VC(\perclass)+1$ samples from a task, we can learn an arbitrarily good representation, even though we do not have enough data to find a specialized classifier $\per \circ \rep$ with small error $\eps$.

Our theory identifies two key features of the classes $\repclass$ and $\perclass$ that enable metalearning in the realizable case: (1) whether $\perclass$ has small \emph{non-realizability certificates}, meaning that given an unrealizable sample of size $\samples$, there is a small subset of samples that are also unrealizable, and (2) whether the class of \emph{realizability predicates} for $\repclass$ and $\perclass$ has small VC dimension, where the realizability predicates are functions that indicate whether the sample $S$ is realizable by a function of the form $\per \circ \rep$ for $\per \in \perclass$.

Along the way, we also study the multitask setting.
We give a simple characterization of the total number of samples $\samples \users$ that we need to learn and pin down the sample complexity of learning halfspaces with a shared linear representation up to constant factors.
These results help to highlight the differences between metalearning and multitask learning. By showing that we can metalearn (improperly) if and only if we can multitask learn in the setting where we have enough samples per task to output a good classifier for a given representation, we obtain task and sample complexity bounds for improper metalearning.

Our results shed light on some of the intriguing empirical phenomena in modern machine learning. For example, the training data sets for foundation models are increasingly being augmented with  data sources from a wide variety of subpopulations, even though those sources may be very small. That's normally motivated by the desire to improve accuracy on these subpopulations. However, our work  highlights a different and complementary reason why that can be valuable: even a few examples from many subpopulations can vastly improve learning on all populations, even the ``data rich'' ones.

% \jubig{Big question: \emph{can we metalearn with very few samples per task?}  We answer this question positively and do it by developing learning theoretic foundations for meta learning in a setting that is both distribution free and allows for very few samples per task.  Contributions
%     \begin{itemize}
%         \item We develop a uniform convergence theory for distribution-free meta learning.  Prior work only studies meta learning under specific assumptions about task distributions.

%         \item \textbf{In contrast to other approaches, our approach to meta learning allows us to prove bounds in the setting where there are so few samples per task that it isn't even possible to evaluate the quality of a representation.}

%         \item Using our framework, we prove task-complexity bounds for linear classifiers with a shared linear representation, showing that with just $k+2$ samples per tasks we can still learn a good representation given enough tasks.  \textbf{What's a clean way to describe our task-complexity bounds that doesn't suck.  Polynomial in $d$ and exponential in $k$?}

%         \item Along the way we also study the related problem of multitask learning and pin down the sample complexity of learning halfspaces with a shared linear representation.
%     \end{itemize}
%}

\subsection{Our Results and Techniques}
\label{sec:results}

In this section we give a more detailed overview of our main results and techniques.  We focus both on the general conditions for meta or multitask learnability that we establish as well as our specific sample- and task-complexity results for linear classifiers over linear representations.  We begin by describing the multitask learning setting because it is simpler and serves as a contrast to highlight some of the unusual features of metalearning. 

\mypar{Multitask Learning: A General Characterization and a Tight Bound for Linear Classes.}
Our first result is a complete characterization of the conditions for distribution-free multitask learning. Given $\repclass$ and $\perclass$ and a number of tasks $\users$, one may view the collection of functions $\rep,\per_1,...,\per_\users$ output by a proper learner  as a single  function from $[\users]\times\featdom$ to $\bits$ that maps $ (j,x)$ to $\per_j(\rep(x))$. We denote by $\perclass^{\otimes \users}\circ \repclass$ the class of such composite functions: 
\[  
    \perclass^{\otimes \users}\circ \repclass 
    :=
    \Bset{
        \conc: [\users]\times \featdom \to \bits 
        \ \Big| \ 
        \exists \rep \in \repclass, \per_1, \ldots, \per_{\users} \in \perclass \textrm{ s.t. } \conc(j,x) = \per_j(\rep(x))
    }
\]
The composite class $\perclass^{\otimes \users}\circ \repclass$ has been studied in other multitask learning settings (see \Cref{sec:rw}). We show that in our setting, the VC dimension of this class characterizes distribution-free multitask learnability in both the realisable and agnostic settings. %We state and prove the following result for completeness.

\begin{thm}[Informal version of Theorem~\ref{thm:mtl-smpls}]
    \label{thm:mtl-smpls-intro}
    Distribution-free multitask learning to error $\eps$ with $\samples$ samples per task and $\users$ tasks is possible if and only if $nt \gtrsim \VC(\perclass^{\otimes \users}\circ \repclass)/\eps^2$.  In the realizable setting $1/\eps^2$ is replaced with $\log(1/\eps)/\eps$.
\end{thm}
See Definition~\ref{def:multitask_learning} for a formal definition of the setting and error parameter $\eps$. Since VC dimension characterizes distribution-free PAC learning, this result shows that multitask learning is possible if and only if the larger composed function class is PAC learnable.  See \Cref{sec:generic-VC-composite} and \Cref{sec:finite_classes} for basic upper and lower bounds on the VC dimension of the composed class.

Nevertheless, an interesting phenomenon emerges since $\users$ appears in both the definition of the function and the sample complexity bound.  As a concrete example, consider the case of a linear classifier over a linear representation
$$
    \repclass_{\datadim,\repdim} = \left\{\rep \mid \rep(\mathbf{\feat}) = \mathbf{B} \mathbf{\feat}, \mathbf{B} \in \R^{\repdim\times \datadim}\right\}
    \text{ and }
    \perclass_{\repdim} = \left\{\per\mid  \per(\mathbf{\map}) = \textrm{sign}(\mathbf{a}\cdot \mathbf{\map} - w),  \mathbf{a} \in \R^{\repdim}, w \in \R \right\},
$$ 
which will be the main example we study in this paper.  Using a bound on the number of possible sign patterns in a family of polynomials due to~\cite{Warren68}, which was also used by~\cite{Baxter2000} for meta-learning bounds on regression problems, we characterize the VC dimension of the associated composite class up to a constant.

\newcommand{\thmMtlVcHLText}{
    For all $\users, \datadim, \repdim \in \N$, we have
    \[\VC(\perclass_{\repdim}^{\otimes \users}\circ \repclass_{\datadim,\repdim})  = 
    \begin{cases}
    \users(\datadim + 1), &\textrm{ if } \users \leq \repdim\\
    \Theta(\datadim \repdim +\repdim\users),&\textrm{ if } \users > \repdim.
    \end{cases}
\]}
%\ifnum\coltshort=0
    \begin{thm}
       \label{thm:mtl-vc-hl-intro}
       \thmMtlVcHLText
    \end{thm}
\begin{comment}
\else
    \begin{thm}
    \label{thm:mtl-vc-hl-intro}
    For all $\users, \datadim, \repdim \in \N$, we have
    $\VC(\perclass_{\repdim}^{\otimes \users}\circ \repclass_{\datadim,\repdim})  = 
    \begin{cases}
    \users(\datadim + 1), &\textrm{ if } \users \leq \repdim\\
    \Theta(\datadim \repdim +\repdim\users),&\textrm{ if } \users > \repdim.
    \end{cases}$
    \end{thm}
    \fi
\end{comment}
\ifnum\coltshort=0
    We see two distinct regimes.  First, when $\users \leq \repdim$, the VC dimension is exactly $\users \cdot (d+1) = \users \cdot \VC(\perclass_\datadim)$, so there is no benefit to the existence of a common representation, and one needs enough data to solve $t$ independent tasks in the larger, $\datadim$-dimensional space. However, for $\users> \repdim$, we see that the sample requirement scales as $\samples \gtrsim  (\frac{dk}{t} + k)/\eps^2$, so there is a fixed cost of $\datadim \repdim/\eps^2$ samples that can be amortized over many tasks. This is natural given that members of $\repclass_{\datadim,\repdim}$ are described by $\repdim\times\datadim$ matrices. However, there is also an unavoidable marginal cost of $\repdim / \eps^2$, which is the number of samples needed to train the specialized classifier for a given task once $\rep$ has been determined.
\else
    We see two distinct regimes:  When $\users \leq \repdim$, the VC dimension is exactly $\users \cdot (d+1)$ which is also the sample complexity of learning $\users$ separate $d$-dimensional linear classifiers, so there is no benefit to the existence of a common representation.  When $\users > \repdim$, the sample complexity scales as $\samples \gtrsim  (\frac{dk}{t} + k)/\eps^2$, so there is a fixed cost of $\datadim \repdim/\eps^2$ samples that can be amortized over many tasks, with an unavoidable marginal cost of just $\repdim/\eps^2$ per task. See \Cref{sec:vc_mtl_app} for more details.
\fi

\mypar{Metalearning: The Case of Linear Representations.}
For metalearning, a different and more complex picture emerges. For example, one might expect that the conditions for multitask learning are also \emph{necessary} for learning a good representation. However, this is not always the case. Consider again the linear classes $\repclass_{\datadim,\repdim}$ and $\perclass_\repdim$. While multitask learning requires $\Omega(\repdim/\eps)$ samples per task, we show that one can accurately learn a good $\repdim$-dimensional linear representation for any realizable family with just $\samples = \repdim+2$ samples per task:

\begin{cor}[Informal version of Corollary~\ref{cor:met-lin-real1}]
    \label{cor:met-lin-real1-intro}
    One can metalearn $(\repclass_{\datadim,\repdim}, \perclass_\repdim)$ to error
    %accuracy 
    $\eps$ in the realizable case if
    \ifnum\coltshort=1
        $\users= \datadim \repdim^2 \cdot O(\ln(1/\eps)/\eps)^{\repdim + 2}$ and $\samples = \repdim +2$.
    \else
        \[
        \users= \datadim \repdim^2 \cdot O\left(\frac{\ln(1/\eps)}{\eps}\right)^{\repdim + 2} \quad \text{and} \quad \samples = \repdim +2\, . 
        \]
    \fi
\end{cor}
See Definition~\ref{def:metalearn} for a formal definition of the setting and notion of error. The algorithm underlying this result selects a representation that minimizes the \emph{fraction of nonrealizable data sets} among the training tasks---see the explanations after \Cref{thm:mtl-smpls-intro} for details and intuition.

\newcommand{\thmmetmon}{We can metalearn $(\repclass_{\datadim,1}, \perclass_{\text{mon}})$ to accuracy $\eps$ in the realizable case with $\users$ tasks and $\samples$ samples per task when \ifnum\coltshort=1 $\users = O\left(\datadim\ln(1/\eps)/\varepsilon^2\right)$ and $\samples=2$ \else $$\users = O\left(\frac{\datadim\ln(1/\eps)}{\varepsilon^2}\right) \quad\text{and}\quad \samples = 2.$$ \fi} 
When we restrict further to the case of specialized monotone thresholds $\perclass_{\textrm{mon}}$ applied to a 1-dimensional linear representation, we get a stronger bound than what would be obtained by applying the result above with $\repdim =1$, namely, just $n=2$ samples per task suffice.
\ifnum\coltshort=1 (see Theorem~\ref{thm:met-mon}).
\else
    \begin{thm}[see Theorem~\ref{thm:met-mon}] 
        \label{thm:met-mon-intro}
        \thmmetmon
    \end{thm}
\fi

These results consider a regime of extreme data scarcity. For linear classifiers, any nondegenerate set of $\repdim+1$ labeled examples can be perfectly fit by a halfspace and so, absent additional assumptions like a large margin, one gets no information from seeing just $\repdim+1$ points per task.\adamnote{State/cite a precise result.}  Thus $\samples = \repdim +2$ is the exact minimum number of samples per task that suffice for metalearning. However, with so few examples we cannot even reliably estimate how well a candidate representation $\rep$ performs on a task. Essentially the only signal one gets is whether the data set is realizable, so it is surprising that representation learning is possible at all in such a regime. Indeed, once we choose a representation and sample a new task $P$ from the metadistribution $\metadist$, we need a larger number of samples $\npers \approx \repdim/\eps^2$ to find a good specialized classifier for the new task. Although such a large number of samples isn't necessary for metalearning, when samples are plentiful we can show that one can learn a good representation with task complexity polynomial in $\datadim$, $\repdim$ and $1/\eps$, even in the agnostic setting.

% \begin{cor}[Informal version of Corollary~\ref{cor:met-lin-agn}]
%     \label{cor:met-lin-agn-intro}
%     One can metalearn $(\repclass_{\datadim,\repdim}, \perclass_\repdim)$ 
%     to error
%     %accuracy
%     $\eps$ with $\users$ tasks and $\samples$ samples per task when
%     \ifnum\coltshort=1
%         $\users = \Tilde{O}(\datadim \repdim^2 / \varepsilon^4)$ and  $\smash{\samples = O( \frac{\repdim+\ln(1/\varepsilon)}{\varepsilon^2})}.$
%     \else
%         $$\users = \Tilde{O}\left(\frac{\datadim \repdim^2}{\varepsilon^4}\right) \quad\text{and}\quad \samples = O\left( \frac{\repdim+\ln(1/\varepsilon)}{\varepsilon^2}\right).$$
%     \fi
% \end{cor}

\begin{cor}[following Corollary~\ref{cor:met-lin-agn} and Corollary~\ref{cor:imp_meta_hlr}]
    \label{cor:met-lin-agn-intro}
    One can metalearn $(\repclass_{\datadim,\repdim}, \perclass_\repdim)$ 
    to error
    %accuracy
    $\eps$ with $\users$ tasks and $\samples$ samples per task for $\samples = O\left( \frac{\repdim\ln(1/\varepsilon)}{\varepsilon^2}\right)$ and \emph{either}
        \(\users = \Tilde{O}\left(\frac{\datadim \repdim^2}{\varepsilon^4}\right),\) 
    via a proper learner, 
    \emph{or} 
    \(\users = O\left(\datadim \right) , \)
    via an improper learner.\adamnote{Write up calculation of parameters.}
\end{cor}

\mypar{Metalearning for realizable families with minimal sample complexity.}
To prove these results, we develop a general uniform-convergence theory for metalearning. 
Our results for realizable learning rely on two major ideas: first, we show that there is a sample size $\samples$ which depends only on $\perclass$ (not on $\repclass$ or $\eps$) that allows for training a representation given many tasks; second, we identify a set of predicates based on $\repclass$ and $\perclass$ that determine how many tasks are needed.

\begin{defn}[\Nrcc]\label{def:non-realizability-complexity} 
    % Let $\perclass$ be a class of functions $\per: \mapdom \to \bits$, $\wit \in \N$ and $S\in (\mapdom\times\bits)^*$ be a dataset. 
    % A \emph{non-realizability witness of size $m$ for $S$ with respect to $\perclass$} is a subset of $S$ of size at most $\wit$ that cannot be realized by $\perclass$. 
    A class $\per: \mapdom \to \bits$ has non-realizability certificates of size at most $\wit$ if every set $S\subseteq \mapdom \times \bits$ that is not realizable by $\perclass$ has a nonrealizable subset of size at most $\wit$. We write $\NRC(\perclass)$ for the smallest $\wit$ for which this condition holds, which we call the \emph{\nrcc}.
\end{defn}

The \nrcc{} is essentially the smallest sample size at which we are guaranteed that realizability provides a meaningful signal. 
For example, the class of monotone threshold functions $\perclass_{\textrm{mon}}$ has $\NRC(\perclass_{\textrm{mon}})= 2$, since every nonrealizable set has two points with different labels that are out of order. 
Halfspaces in $\R^\repdim$ have \nrcc{} $\repdim + 2$ by a classic duality argument~\cite{Kirchberger1903}. In general, $\NRC(\perclass)$ and $\VC(\perclass)$ dimension can differ arbitrarily in either direction, as we show in \Cref{sec:nrcc}.

Given this sample size $\wit$, we consider a class of \emph{realizability predicates}:
\begin{defn}[Realizability predicate]\label{def:realizability_predicate}
    Let $\perclass$ be a class of specialized classifiers $\per: \mapdom \to \bits$ and $\samples \in \N$. Given a representation function $\rep: \featdom \to \mapdom$, the \emph{$(\samples,\perclass)$-realizability predicate} is
    $
        \realp_{\rep}\paren{(\feat_1, \lab_1), \ldots, (\feat_{\samples}, \lab_{\samples})} \defeq \ind_{\pm}\left\{\exists \per \in \perclass \text{ s.t. } \forall i \in [\samples], \per(\rep(\feat_i))=y_i\right\}.
  $
    The class of realizability predicates, $\realpclass_{\samples, \perclass, \repclass}$, contains all realizability predicates $\realp_{\rep}$ for $\rep \in \repclass$.
\end{defn}
For example, for the class of  monotone thresholds, 
we can view the representation $\rep$ as a $1$-dimensional function that orders the samples on the real line. Then, $\realp_{\rep}$ returns $+1$ if and only if all the positively labeled samples have greater representation value than the negative ones.
\begin{comment}
Given this sample size $\wit$, we consider a class of \emph{realizability predicates}: for each representation $\rep$, the predicate $r_\rep$ indicates whether a labeled data set of size $\wit$ is realizable by  $\perclass\circ h$: 
    \begin{align*}
        \realp_{\rep}\bparen{(\feat_1, \lab_1), \ldots, (\feat_{\wit}, \lab_{\wit})} \defeq \ind_{\pm}\bset{\exists \per \in \perclass \text{ s.t. } \forall i \in [\wit], \per(\rep(\feat_i))=y_i},
    \end{align*}
 where $\ind_{\pm}\{p\}$ is an indicator which takes value $+1$ if $p$ is true and $-1$ otherwise. Let $\realpclass_{\wit, \perclass, \repclass} = \set{r_h \mid h \in \repclass}$ denote the class of  all such realizability predicates. For example, for the class of mototone thresholds we can view the representation $\rep$ as a $1$-dimensional function that orders the samples on the real line. Then, $\realp_{\rep}$ returns $+1$ if and only if all the positively labeled samples have greater representation value than the negative ones.
 \end{comment}
%
%\adamnote{maybe say more about what these predicates capture.; maybe what they look like for thresholds. We have text in the tech sections that we can use.}

Our key result is that realizable metalearning with just $\wit$ samples per task is possible whenever the VC dimension of this predicate class is finite---in fact, it scales  linearly with its VC dimension and at most exponentially in $\NRC(\perclass)$---by  minimizing the fraction of nonrealizable samples.

\begin{thm}[Informal version of Theorem~\ref{thm:met-samples}]
\label{thm:met-samples-intro}
    %If $\perclass$ has \nrcc{} $\wit$
    We can metalearn $(\repclass, \perclass)$ with error $\eps$ in the realizable case with  $\NRC(\perclass)$ samples per task and task complexity that, when $\VC(\perclass)=O(\NRC(\perclass))$, scales as
    \[ \users =  \VC\paren{\realpclass_{\NRC(\perclass),\perclass,\repclass}}
    \cdot 
            O \paren{ \frac{\log(1/\eps)}{ \eps }}^{\NRC(\perclass)}
    \, .  
    %\bparen{\VC(\realpclass_{\wit,\perclass,\repclass})+\ln(1/\delta)} 
    %\cdot
    %\paren{ \frac{O(\VC(\perclass)\ln(1/\varepsilon))} {\wit \varepsilon }}^\wit    
    \]
\end{thm}

For natural classes (such as halfspaces), we expect the VC dimension of $\perclass$ and the \nrcc{} to be comparable, and so we get a complexity of roughly $\VC(\realpclass_{\wit,\perclass,\repclass}) / \eps^m$ for $\wit=\NRC(\perclass)$. %(Indeed, this is the case when $\perclass$ is the class of halfspaces.) 

To get some intuition for  Theorem~\ref{thm:met-samples-intro}, consider a fixed representation $\hat\rep$ and a task $\dist$ on $\featdom\times\bits$. 
The heart of our argument is a relationship between two important quantities: the \emph{population error} of $\perclass \circ \hat\rep$ on $P$, and the \emph{probability of nonrealizability} of a  dataset from $P^{\otimes \wit}$ by functions in $\perclass \circ \hat\rep$. 
Letting $\distint$ denote 
an arbitrary distribution on $\mapdom \times \bits$ (in the rest of the paper, this distribution will generally be $(\hat \rep(X),Y)$ for $(X,Y)\sim P$ and some candidate representation $\hat h$),
%the distribution of $\hat \rep(X)$ for $X\sim P$, 
we define
\begin{align*}
    \er(\distint,\perclass)  \defeq \min_{\per \in \perclass}\Pr[(\feat,\lab)\sim \distint]{\per(\feat)\neq \lab} 
    \quad\textrm{and}\quad p_{\mathrm{nr}}(\distint, \perclass, \wit) \defeq \Pr[S \sim \distint^\wit]{S\text{ is not realizable by }\perclass}.   
\end{align*}

In Lemma~\ref{lem:real-to-error}, we give a general lower bound on the probability of nonrealizability $p_{\mathrm{nr}}(\distint, \perclass, \wit)$ in terms of the error of the best classifier $\er(\distint,\perclass)$ and the complexity $\NRC(\perclass)$: for any class of specialization functions $\perclass$ and any distribution $\distint$ on the intermediate space $\mapdom\times \bits$,
\begin{equation}
    p_{\mathrm{nr}}(\distint, \perclass, m) \gtrsim \paren{\tfrac{\wit}{\wit + \VC(\perclass)} \cdot \er(\distint,\perclass)}^{\wit} \quad \text{for}\ \wit=\NRC(\perclass)\,.
\label{eq:pnr-bound-intro}    
\end{equation}

Bounding $p_{\mathrm{nr}}$ is nontrivial since the sample size $\wit = \NRC(\perclass)$ is typically much smaller than $N = \VC(\perclass)/\er(\distint,\perclass)$, which is what standard concentration arguments require. To get around this, we consider a thought experiment in which a larger data set size $N$ is drawn and then a random subsample of size $\wit$ within it is observed. Since the larger set will almost certainly be unrealizable, by definition it must contain an unrealizable  subset of size $\wit$. The random subsample will find such a witness with probability at least $1/\binom{N}{\wit}$, which is bounded by the expression in \eqref{eq:pnr-bound-intro}.
%\adamtext{The particular form of \eqref{eq:pnr-bound-intro} is chosen to be convex in $\er(\distint,\perclass)$ for later convenience.}

The lower bound on $p_{\mathrm{nr}}$ implies that a representation $\hat \rep$ which is poor for typical tasks from a meta distribution $\metadist$ will also  lead to a typical sample of size $m$ drawn from a typical task being unrealizable. When $\realpclass_{\wit,\perclass,\repclass}$ has low VC dimension, the number of unrealizable samples in the training data will  concentrate uniformly across representations $\hat \rep$, and so the representation which minimizes the number of unrealizable training data sets will also, up some loss, minimize the expected population level loss $\er(\distint,\perclass)$. This line of argument leads to Theorem~\ref{thm:met-samples-intro}.

\mypar{Metalearning with More Data Per Task.} 
Checking the realizability of very small data sets
works when we are guaranteed to find a representation for which all the training data sets will be realizable. 
In the agnostic setting, this approach makes less sense.
Additionally, the bounds we obtain scale poorly with the $\NRC(\perclass)$.

To obtain polynomial scaling and cope with the agnostic setting, we consider larger training data sets—large enough to assess the quality of a given representation on the task at hand—and a different set of function classes associated to the pair $(\repclass,\perclass)$. Instead of realizability, we consider a function that returns the best achievable empirical error for a given representation on a dataset: 

 \begin{defn}[Empirical error function]\label{def:empirical_error_function}
    Let $\perclass$ be a class of specialized classifiers $\per: \mapdom \to \bits$ and $n \in \N$. Given a representation function $\rep: \featdom \to \mapdom$, the {\em $(n,\perclass)$-empirical error function} of $\rep$ is
    \ifnum\coltshort=1
        $\erind_{\rep}(\feat_1, \lab_1, \ldots, \feat_{n}, \lab_{n}) \defeq \min_{\per \in \perclass}\frac{1}{n}\sum_{i=1}^n \ind\{\per(\rep(\feat_i)) \neq \lab_i\}.$
    \else
    \begin{align*}
       \erind_{\rep}(\feat_1, \lab_1, \ldots, \feat_{n}, \lab_{n}) \defeq \min_{\per \in \perclass}\frac{1}{n}\sum_{i=1}^n \ind\{\per(\rep(\feat_i)) \neq \lab_i\}.
    \end{align*}
    \fi
    The class of empirical error functions $\erindclass_{\samples, \perclass, \repclass}$ contains all empirical error functions $\erind_{\rep}$ for $\rep \in \repclass$.
\end{defn}
%\ktext{Note that in this definition indicator $\ind\{p\}$ takes value $1$ if $p$ is true and $0$ otherwise.}

We show that the task complexity of metalearning with data sets of size $\samples$ can be bounded via the complexity of $\erindclass_{\samples, \perclass, \repclass}$. Specifically, we consider the \emph{pseudodimension} of $\erindclass_{\samples, \perclass, \repclass}$, which is essentially the VC dimension of the class of all binary thresholdings of the class. The pseudodimension of $\erindclass_{\samples, \perclass, \repclass}$ is thus at least VC dimension of the realizability class $\realpclass_{\wit, \perclass, \repclass} $, but can in general be larger.

\newcommand{\thmagnmet}{Let $\repclass$ be a class of representation functions $\rep: \featdom \to \mapdom$ and $\perclass$ be a class of specialized classifiers $\per: \mapdom \to \bits$. Then, for every $\eps$ and $\delta \in (0,1)$ we can $(\eps, \delta)$-properly metalearn $(\repclass, \perclass)$ with $\users$ tasks and $\samples$ samples per task when \ifnum \coltshort = 1
\users = O\left(\frac{\pd(\erindclass_{\samples, \perclass, \repclass })\ln(1/\varepsilon)+\ln(1/\delta)}{\varepsilon^2}\right)$ and $\samples = O\left( \frac{\VC(\perclass)+\ln(1/\varepsilon)}{\varepsilon^2}\right).$  \else $$\users = O\left(\frac{\pd(\erindclass_{\samples, \perclass, \repclass })\ln(1/\varepsilon)+\ln(1/\delta)}{\varepsilon^2}\right)\quad \text{ and}\quad\samples = O\left( \frac{\VC(\perclass)+\ln(1/\varepsilon)}{\varepsilon^2}\right).$$
\fi}

\begin{thm} [see Theorem~\ref{thm:agn-met-samples}]
   \label{thm:agn-met-samples-intro}
    \thmagnmet
    
    %\adaminline{Copy \Cref{thm:agn-met-samples} on psuediomension and metalearning.}
\end{thm}

There exist other ways of measuring the complexity of real-valued functions---fat-shattering dimension and Rademacher complexity are more general than pseudodimension---but we show that for the linear representations and halfspaces, the pseudodimension lends itself directly to bounds based on the signs of polynomials.

%In particular, we show that we can metalearn $(\repclass, \perclass)$ with $\Tilde{O}(\pd((\erindclass_{\samples, \perclass, \repclass })/\eps^2)$ tasks and $\Tilde{O}(\VC(\perclass)/\eps^2)$ samples per task.
%In essence, this number of tasks allows us to infer that a representation $\hat{\rep}$ that minimizes the empirical error has an overall small error as well.

\mypar{Generic reductions between metalearning and multitask learning.}
% A reduction from metalearning to multitask learning. \knote{this needs to change} 
%
To complement our uniform convergence theory for metalearning, we also show a generic equivalence between 
metalearning and multitask learning in \Cref{sec:red}.
% reduction \ktext{from} metalearning \ktext{to} multitask learning \ktext{} in \ifnum \coltshort = 1\Cref{sec:models_red}\else \Cref{sec:red}.  
This equivalence leads to
% Combining this reduction with our uniform convergence results for multitask learning, we obtain 
improved sample complexity for linear classes in the regime where we have enough samples per task to solve multitask learning.

\ifnum\coltshort=0
    \begin{thm}[see Theorems ~\ref{thm: meta-to-mtl} and~\ref{thm: mtl-to-meta}]
        If there is a multitask learner for $(\repclass, \perclass)$ with error $\eps$ using $\users+1$ tasks and $\samples$ samples per task, then there is an improper metalearner for the same class with error $O(\eps)$ using $\users$ tasks and $\samples$ samples per task. If there is a metalearner for $(\repclass, \perclass)$ with error $\eps$ using $\users$ tasks and $\samples$ samples per task, then there is a multitask learner for the same class with error $O(\eps)$ using $\users$ tasks and $O(\samples \ln(\users))$ samples per task.
    \end{thm}
\fi
In contrast to our earlier results, these reductions \emph{do not} yield any bounds for metalearning when we have just $\samples = \VC(\perclass)+1$ samples per task. Interestingly, the reduction gives metalearners that are \textit{improper}, even when the original multitask learner is proper.  Instead of a succinct representation, the reduction produces a specialization algorithm whose description is no simpler than the entire training data set.
%Another interesting of this reduction is that it yields an \emph{improper} metalearning algorithm, which does not return any explicit shared representation. 
We leave it as an open question to determine if one can match the sample complexity bounds we obtain from our reduction via  proper metalearning.

\mypar{Analyzing the classes associated to linear representations.}
To apply our general results to linear classes, we produce bounds on the VC dimension of the realizability predicates and, using similar tools, the pseudodimension of the empirical error function, yielding the following result:
\newcommand{\thmpdimacc}{For all $\datadim, \repdim, \samples$ with $\samples\ge \repdim+2$,
    $\mathrm{PDim}(Q_{n,\perclass_{\repdim}, \repclass_{\datadim,\repdim}}) = \tilde{O}(dkn)\, .$}
\begin{thm}[see Theorem~\ref{thm:pdim-acc}]
    \thmpdimacc
    \label{thm:pdim-acc-intro}
\end{thm}
At a high level, our approach shows that the predicate $r_h$ can be expressed in low-complexity language.  Concretely, we first derive a list of low-degree polynomials $p^{(1)}(h),\ldots,p^{(w)}(h)$ which are positive exactly when certain conditions on the data and representation are met. We then show that the realizability predicate can be written as a Boolean function over these conditions: the function's $\ell$-th input is $\sign(p^{(\ell)}(h))$. A result of~\cite{Warren68} allows us to bound the VC dimension of signs of low-degree polynomials, and we prove a novel extension of Warren's result applies to functions over these objects.

This approach involves a subtle conceptual shift.  Let $D=(\feat_1,\lab_1,\ldots,\feat_\samples,\lab_\samples)$ be a data set. Our notation for the realizability predicate $r_h(D)$ suggests that it is ``parameterized'' by representations and receives as input a data set. In this section of analysis, these roles are reversed: we construct and analyze polynomials $p_D^{(1)}(h),\ldots,p_D^{(w)}(h)$. The \emph{coefficients} of these polynomials depend on $D$, but we measure their complexity via their degree as polynomials in $h$.

\begin{comment}
\adaminline{Old Text Below}
We prove a bunch of sample complexity bounds for learning linear classifiers on top of linear representations.

\begin{table}[h!]
\centering
\begin{tabular}{@{}lcccc@{}}
\toprule
Result       & Multi vs. Meta & Realizable? & Samples-per-Task $n$ & Number of Tasks $t$            \\ \midrule
Thm \ref{..} & Multi           & No         & -                & - \\
Thm \ref{..} & Multi           & Yes         & -                & -\\
Thm \ref{..} & Meta           & Yes         & $k+2$                & $\tilde{O}(dk^2/\alpha^{k+2})$ \\
Thm \ref{..} & Meta           & Yes         & $O(k/\alpha)$        & $\tilde{O}(dk^2/\alpha^2)$     \\
Thm \ref{..} & Meta           & No          & $O(k/\alpha^2)$      & $\tilde{O}(dk^3/\alpha^4)$     \\ \bottomrule
\end{tabular}
\caption{Sample complexity bounds we can prove.}
\end{table}

We have a bunch of intermediate results that are interesting:
\begin{enumerate}
    \item If a class has certificates of size $c$ then a sample of size $c$ will be non-realizable with probability $\Omega(\alpha^c)$.
    \textit{NB:} Linear classifiers have certificates of size $k+2$ (Kirchberger's Theorem).
    \item If the VC dimension of the realizability predicate is bounded then a small number of tasks suffice to learn.
\end{enumerate}
\end{comment}

\subsection{Related Work} \label{sec:rw}

There is a large body of literature on both multitask learning and metalearning, including related concepts or alternative names such as \emph{transfer learning}, \emph{learning to learn}, and \emph{few-shot learning}, going back at least as far as the 90s (see~\cite{caruana:1998ml,ThrunP1998} for early surveys). 
Although this line of work is too broad to survey entirely, we  summarize here some of the main themes and highlight points of difference with our work.

The most closely related work to ours is that of~\cite{Baxter2000}, which proves sample complexity bounds for both multitask learning and metalearning in a framework that is equivalent to our shared representation framework.  Our main advances compared to \cite{Baxter2000} are: (1) we prove sample complexity bounds in the regime where the number of samples per task is too small to learn an accurate specialized classifier, whereas the generalization arguments in \cite{Baxter2000} crucially rely on having enough samples per task to learn an accurate specialized classifier for a given representation and task, (2) our bounds hold in a distribution-free, classification setting and do not rely on the sort of margin assumptions that are essential for the covering arguments used in \cite{Baxter2000} for the analogous setting.  Other works differ from ours along one or more of the following axes.

% \mypar{Different approaches to modeling relatedness of tasks.}
% Many works consider qualitatively different ways of modeling relatedness between different tasks than the shared representation setting we consider here. One popular approach to modeling relatedness is to assume that the optimal classifiers for each task are close according to some metric, either on the predictions they make or on their natural parametrizations.\ju{Need to add cites here.}

\mypar{Stronger assumptions on the tasks or meta-distribution.}
Several works \cite{SrebroB06,Maurer2009,maurer2012,pontil2013,MaurerPR16} prove dimension-free generalization bounds for various forms of linear classifiers over linear representations under stronger 
%distributional
assumptions on the input, such as margin assumptions.  In contrast, our work proves distribution-free bounds, which are necessarily dimension-dependent.

Another line of work \cite{TripuraneniJJ20,DuHKLL20,TripuraneniJJ21} considers multitask/meta learning under \emph{task diversity} assumptions that make it easier to identify the optimal representation.
Other ways to model task relatedness based on Kolmogorov/information complexity include \cite{juba2006,ben2008,mahmud2009}.  Lastly \cite{HannekeK22} consider a setting where the tasks share the same optimal classifier, not just a common representation that can be specialized to different optimal classifiers.

\mypar{Computationally efficient algorithms.}
While our work focuses on the intrinsic information-theoretic sample complexity of multitask and metalearning, there is a long line of work \cite{balcan2019, kong2020a, kong2020b,saunshi2020,fallah2020,ThekumparampilJNO21,chen2021,collins2022} that studies the sample complexity of specific practical heuristics such as \emph{model-agnostic metalearning (MAML)} under more restrictive assumptions where these heuristics are provably effective.  \cite{BairaktariBTUZ23} also give efficient algorithms for multitask learning sparse low-weight halfspaces.

\mypar{PAC-Bayes bounds for multitask and metalearning.}
Another important line of work~\cite{pentina2014,amit2018,lucas2020,rezazadeh2021,rothfuss2021,chen2021,rezazadeh2022} extends the PAC-Bayes and other information-theoretic generalization bounds to the multitask and metalearning problems.  These works typically consider properties of both a data distribution and an algorithm that bound generalization error, but do not address the optimal sample complexity of learning specific families, which is the focus of our work.

\section{Preliminaries}

\ifnum\coltshort=0
\subsection{Notation}
We use bold lowercase letters for vectors, e.g. $\mathbf{a}$, and bold uppercase letters for matrices, e.g. $\mathbf{A}$. Given a $\datadim$-dimensional vector $\mathbf{a}$ and a value $b$, $(\mathbf{a}\|b)$ is a $\datadim+1$-dimensional vector, where we have appended value $b$ to vector $\mathbf{a}$. For distributions over labeled data points we use the regular math font, e.g. $D$, whereas for metadistributions (distributions over task data distributions) we use fraktur letters, e.g. $\mathfrak{D}$. We define the indicator function of a predicate $p$ as
\begin{align*}
    \ind\{p\} = \begin{cases}
        1 & \text{if }p\text{ is true} \\
        0 & \text{o.w.}
    \end{cases}.
\end{align*}
Occasionally, we will use the alternative indicator $\ind_{\pm}\{p\}$, which takes value $+1$ if $p$ is true and $-1$ otherwise.
We also use the following sign function 
\begin{align*}
    \sgn(x) = \begin{cases}
        +1 & \text{if }x \geq 0 \\
        -1 & \text{o.w.}
    \end{cases}.
\end{align*}
Contrary to the convention in Boolean analysis, when necessary we interpret $+1$ as logical ``true'' and $-1$ as logical ``false.'' 
\fi

In this work we consider several different types of objects---representations and classifiers---and several different types of error---training and test error over different distributions.  The following table summarizes these error measures. \ifnum \coltshort = 1\ktext{For further details about our notation see \Cref{app:back}}.\fi
\begin{table}[ht]
    \centering
    \def\arraystretch{1.25}
    \begin{tabular}{|c|c|}
    \hline
     $\er(S, \per) = \frac{1}{|S|} \sum_{i=1}^{|S|}\ind\{\per(\feat_i) \neq \lab_i\}$ &training error of classifier $\per$ \\
    \hline
     $\er(\dist, \per ) = \Pr[(\feat,\lab)\sim \dist]{\per(\feat)\neq \lab}$ & test error of classifier $\per $\\
     \hline
    $\er(\dist, \perclass) = \min_{\per \in \perclass}\Pr[(\feat, \lab) \sim \dist]{ \per(\feat)\neq \lab}$&  test error of class $\perclass$ \\
     \hline
     $\rer(S,\rep, \perclass) = \min_{\per \in \perclass} \er(S, \per \circ \rep)$ 
     & training error of representation $\rep$\\
     \hline
     $\rer(\dist,\rep,\perclass) = \min_{\per \in \perclass} \er(\dist, \per \circ \rep) =\er(\rep(\dist), \perclass)$
     & test error of representation $\rep$\\
     \hline
     $\rer(\metadist, \rep, \perclass) = \Exp_{\dist \sim \metadist} \left[\rer(\dist,\rep, \perclass)\right]$& meta-error of representation $\rep$ \\
     \hline
     $p_{nr}(\dist, \perclass, \wit) = \Pr[S \sim \dist^\wit]{S\text{ is not realizable by }\perclass}$& probability of non-realizability\\
     \hline
    \end{tabular}
    %\caption{Error definitions.}
    \label{tab:errors}
\end{table}

\ifnum \coltshort = 0
\subsection{Background}
\fi

Our sample complexity results are based on classical generalization bounds via VC dimension and pseudodimension. \ifnum \coltshort = 1 These definitions and results are included in \Cref{app:back} and in standard references~\cite{shalev2014understanding,AnthonyB1999}. \ktext{The VC dimension of a function class $\perclass$ is denoted $\VC(\perclass)$ and the pseudodimension is $\textrm{PDim}(\perclass)$}. \else These definitions and results are included in standard references~\cite{shalev2014understanding,AnthonyB1999}.\fi 
\ifnum \coltshort=0
\mypar{VC dimension.} We start by recalling the standard definition of VC dimension and the Sauer-Shelah Lemma.
\begin{defn}\label{def:vc_dimension}
    Let $\mathcal{F}$ be a set of functions mapping from a domain $\mathcal{X}$ to $\{\pm 1\}$ and suppose that $X=(x_1,\ldots,x_n)\subseteq \mathcal{X}$.
    We say that $\mathcal{F}$ \emph{shatters} $X$ if, for all $b\in \{\pm 1\}^n$, there is a function $f_b\in \mathcal{F}$ with $f_b(x_i)=b_i$ for each $i\in[n]$.  
    
    The \emph{VC dimension} of $\mathcal{F}$, denoted $\mathrm{VC}(\mathcal{F})$, is the size of the largest set $X$ that is shattered by $\mathcal{F}$.
\end{defn}

VC dimension is closely related to the \emph{growth function}, which bounds the number of distinct labelings a hypothesis class can produce on any fixed data set.
\begin{defn}
    Let $\perclass$ be a class of functions $\per: \mathcal{Z}\to\{\pm 1\}$ and $S = \{z_1, \ldots,z_n\}$ be a set of points in $\mathcal{Z}$. The \emph{restriction} of $\perclass$ to $S$ is the set of functions 
    $$\perclass_S = \{(\per(z_1), \ldots, \per(z_n))\mid\per \in \perclass\}.$$
    The \emph{growth function} of $\perclass$, denoted $\growth_\perclass(n)$, is 
    $
        \growth_\perclass(n) := \sup_{S: |S|=n} \left|\perclass_S \right|.
    $
\end{defn}

The Sauer-Shelah Lemma, which bounds the growth function in terms of the VC dimension.
\begin{lem}%[\cite{sauer1972,shelah1972}]
    If $\VC(\perclass) = \vc$, then for every $\samples > \vc$, $\growth_{\perclass}(\samples) \leq (e\samples/\vc)^\vc.$
    \label{lem:sauer}
\end{lem}

\mypar{VC dimension and PAC learning.}  Now we recall the relationship between VC dimension and the sample complexity of distribution-free PAC learning.  Here we refer to the textbook notion of PAC learning without giving a formal definition. 
\begin{thm}
\label{fact:pac-vc}
    Let $\class$ be a hypothesis class of functions $f: \featdom \to \bits$ with $\VC(\class) = \vc < \infty$, then for every $\eps,\delta \in (0,1)$
    \begin{enumerate}
        \item $\class$ has uniform convergence with sample complexity $O(\frac{\vc+\ln(1/\delta)}{\varepsilon^2})$
        \item $\class$ is agnostic PAC learnable with sample complexity $O(\frac{\vc+\ln(1/\delta)}{\varepsilon^2})$
        \item $\class$ is PAC learnable with sample complexity $O(\frac{\vc\ln(1/\varepsilon)+\ln(1/\delta)}{\varepsilon})$.
    \end{enumerate}
\end{thm} 
While uniform convergence requires $O(1/\eps^2)$ samples, with just $O(\ln(1/\eps)/\eps)$ samples we will have the property that every hypothesis that has error $\eps$ on the distribution has non-zero error on the sample.
\begin{thm}
    Let $\class$ be a hypothesis class with VC dimension $\vc$. Let $D$ be a probability distribution over $\featdom\times \bits$. For any $\varepsilon,\delta > 0$ if we draw a sample $S$ from $D$ of size $n$ satisfying 
    \begin{align*}
        n \geq \frac{8}{\varepsilon}\left(\vc \ln{\left(\frac{16}{\varepsilon}\right)+\ln{\left(\frac{2}{\delta}\right)}}\right)
    \end{align*}
    then with probability at least $1-\delta$, all hypotheses $\per$ in $\perclass$ with $\er(D,f)>\varepsilon$ have $\er(S,f)>0$.
    \label{fact: vc-gen}
\end{thm}

\mypar{VC dimension of halfspaces.} For a significant part of this paper, we work with linear classifiers of the form $$\perclass_{\datadim} = \left\{\per \left| \per(\textbf{x}) = \sign(\mathbf{a}\cdot \mathbf{x}-w), \mathbf{a} \in \R^\datadim, w \in \R \right.\right\}.$$ We also consider the class of linear classifiers that pass through the origin $$\tilde{\perclass}_{\datadim} = \left\{\per \left| \per(\textbf{x}) = \sign(\mathbf{a}\cdot \mathbf{x}), \mathbf{a} \in \R^\datadim \right.\right\}.$$

\begin{thm}
    \label{thm:vc-hs}
    We have $\VC(\perclass_{\datadim}) = \datadim+1$ and $\VC(\tilde{\perclass}_{\datadim}) = \datadim$.
\end{thm}

\mypar{Pseudodimension.} For real-valued functions we use generalization bounds based on a generalization of VC dimension called the \emph{pseudodimension.}
\begin{defn}\label{def:pseudodimension}
    Let $\mathcal{F}$ be a set of functions mapping from a domain $\mathcal{X}$ to $\mathbb{R}$ and suppose that $X = (x_1,\ldots,x_n) \subseteq \mathcal{X}$.
    We say that $X$ is \emph{pseudoshattered} by $\mathcal{F}$ if there are real numbers $r_1,\ldots,r_n$ such that for each $b\in \{\pm 1\}^n$ there is a function $f_b\in\mathcal{F}$ with $\sign(f_b(x_i)-r_i)=b_i$ for each $i\in [n]$.
    
    The \emph{pseudodimension} of $\mathcal{F}$, denoted $\mathrm{Pdim}(\mathcal{F})$, is the size of the largest set $X$ that is pseudoshattered by $\mathcal{F}$.
\end{defn}

\begin{thm}%[\cite{AnthonyB1999}]
\label{fact:pac-pd}
    Let $\class$ be a hypothesis class of functions $f: \featdom \to \R$ with $\textrm{PDim}(\class) = \pseudod < \infty$, then $\class$ has uniform convergence with sample complexity $O(\frac{\pseudod\ln(1/\varepsilon)+\ln(1/\delta)}{\varepsilon^2})$.
\end{thm}
\fi

\ifnum \coltshort = 0
\section{Multitask Learning and Metalearning Models}
\label{sec:models}
In this section we introduce the learning models we consider in this work.

\subsection{The Multitask Learning Model}
In multitask learning, we pool data together from $\users$ related tasks with the aim of finding one classifier per task so that the average test error per task is low. When these tasks are related, we may need fewer samples per task than if we learn them separately, because samples from one task inform us about the distribution of another task. In this work, we consider classifiers that use a shared representation to map features to an intermediate space in which the specialized classifiers are defined.  We want to achieve low error on tasks that are related by a single shared representation that can be specialized to obtain a low-error classifier for most tasks.

\begin{defn}[Multitask learning]\label{def:multitask_learning}
Let $\repclass$ be a class of representation functions $\rep: \featdom \to \mapdom$ and $\perclass$ be a class of specialized classifiers $\per: \mapdom \to \bits$. We say that $(\repclass,\perclass)$ is \emph{distribution-free $(\varepsilon, \delta)$-multitask learnable for $\users$ tasks with $\samples$ samples per task} if there exists an algorithm $\alg$ such that for every $\users$ probability distributions $\dist_1, \ldots, \dist_{\users}$ over $\featdom \times \bits$, for every $\rep \in \repclass$ and every $\per_1, \ldots, \per_{\users} \in \perclass$, given $\samples$ i.i.d. samples from each $\dist_i$, returns  hypothesis $\conc:[\users]\times \featdom \to \bits$ such that with probability at least $1-\delta$ over the randomness of the samples and the algorithm 
\begin{align*}
    \frac{1}{\users}\sum_{j\in [\users]}\er(\dist_j, g(j,\cdot)) \leq \min_{\rep \in \repclass, \per_1, \ldots, \per_{\users} \in \perclass} \frac{1}{\users}\sum_{j\in [\users]}\er(\dist_j, \per_j\circ \rep)+ \varepsilon.
\end{align*}

If there exists an algorithm $\alg$ such that the same guarantee holds except that we only quantify over all distributions $\dist_1,\dots,\dist_{\users}$ such that
\begin{align*}
    \min_{\rep \in \repclass, \per_1, \ldots, \per_{\users} \in \perclass} \frac{1}{\users}\sum_{j\in [\users]}\er(\dist_j, \per_j\circ \rep) = 0
\end{align*}
then we say that $(\repclass,\perclass)$ is \emph{distribution-free $(\varepsilon, \delta)$-multitask learnable for $\users$ tasks with $\samples$ samples per task in the realizable case}.

For brevity, we will typically omit the term ``distribution-free,'' which applies to all of the results in this paper.
\end{defn}

The natural approach to learning $(\repclass,\perclass)$ is to learn a classifier that, given the index of a task and the features of a sample, computes the representation of the sample and then labels it using the task-specific specialized classifier. We call the class of these classifiers $\perclass^{\otimes \users}\circ \repclass$.

\begin{defn}
Let $\repclass$ be a class of representations $\rep: \featdom \to \mapdom$ and $\per$ be a class of specialized classifiers $\per: \mapdom \to \bits $. We define the class of 
% specialized classifiers 
composed classifiers
for multitask learning with $\users$ tasks as 
$$\perclass^{\otimes \users}\circ \repclass = \left\{\conc: [\users]\times \featdom \to \bits \mid \exists \rep \in \repclass, \per_1, \ldots, \per_{\users} \in \perclass \textrm{ s.t. } \conc(j,x) = \per_j(\rep(x))\right\}.$$
\end{defn}

\subsection{The Proper Metalearning Model}
\label{sec:meta-model}
Suppose there exists a distribution $\metadist$ of different but related binary classification tasks. For each task we can draw a dataset of labeled examples from its specific data distribution $\dist$ over $\featdom \times \bits$. In proper metalearning, we draw $\users$ tasks from distribution $\metadist$ and pool together the data from these tasks. We exploit the relatedness between the tasks to learn a common representation $\rep: \featdom \to \mapdom$ in $\repclass$ that maps the features to a new domain. The hope is that the learned representation facilitates learning new tasks from distribution $\metadist$ meaning that we require fewer samples to learn their classifiers compared to the case where no representation is provided. 

The accuracy of the representation is measured based on how well we can use it to solve a new binary classification task drawn from the same distribution as the ones we have seen in the metalearning phase. The data efficiency of the metalearning algorithm is measured based on two parameters, the number of tasks and the number of samples per task. The metalearning setting is particularly interesting when the number of samples per task is too small to learn a good classifier for each task independently, so data must be pooled across tasks to find a representation that can be specialized to each task using less data than we would need to learn from scratch.

There is another aspect of efficiency that comes into effect after metalearning: given the learned representation, how many samples do we need to learn a new task? The number of samples, in this case, is merely determined by the choice of the specialized classifier class. In this paper, we focus on the efficiency of the metalearning algorithm that finds a good representation.

\begin{defn}[Proper Metalearning]\label{def:metalearn}
Let $\repclass$ be a class of representation functions $\rep: \featdom \to \mapdom$ and $\perclass$ be a class of specialized classifiers $\per: \mapdom \to \bits$. We say that $(\repclass,\perclass)$
is \emph{$(\varepsilon, \delta)$-\ktext{properly} meta-learnable for metadistribution $\metadist$ over distributions on $\featdom \times \bits$, for $\users$ tasks and $\samples$ samples per task}, if there exists an algorithm $\alg$  that, given $\users$ distributions $\dist_1, \ldots, \dist_{\users}$ drawn i.i.d.\ from $\metadist$ and $\samples$ i.i.d.\ samples from each data distribution $\dist_i$,  returns representation $\hat{\rep} \in \repclass$ such that with probability at least $1-\delta$ over the randomness of the samples and the algorithm 
\begin{align}\label{eq:metalearn}
\rer\paren{\metadist, \hat{\rep}, \perclass} \leq \min_{\rep \in \repclass}\rer\paren{\metadist, \rep, \perclass} + \varepsilon.
\end{align}

If there exists an algorithm $\alg$ such that the same guarantee holds only when the metadistribution $\metadist$ is such that
$$ \min_{\rep \in \repclass} \rer(\metadist, \rep, \perclass) = 0$$
then we say that $(\repclass, \perclass)$ is \emph{$(\eps, \delta)$-properly meta-learnable for meta-distribution $\metadist$, $\users$ tasks and $\samples$ samples per task in the realizable case}.
\end{defn}

\Cref{def:metalearn} makes no assumptions on the process used to adapt the learned representation $\hat\rep$ to a new task. It only bounds the number of samples per task $n$ needed to find $\hat \rep$. Given a new task $\dist'$ drawn from $\metadist$ and a representation $h$ that satisfies \Cref{eq:metalearn}, $n_{\text{spec}}=\tilde O(\VC(\perclass)/\eps^2)$ suffice in order to find an $\per\in\perclass$ such that $\per\circ \hat\rep$ has error $\rer(\metadist, \hat{\rep}, \perclass) + 2\eps$ (in expectation over $\dist'$ and the data drawn from $\dist'$). The size $n_{\text{spec}}$ of the data used for specialization could be much larger or smaller than the per-task training set size $n$.

\subsection{The General (Improper) Metalearning Model}
We can also think of metalearning as a more general process than the proper metalearning model we defined which finds a good representation. At a high level the metalearning process pools together datasets from multiple tasks and returns an algorithm that we can use in the future to obtain good classifiers for new tasks drawn from the same metadistribution. In proper metalearning we did not just specify the form of the output of the specialization algorithm, but also the specialization algorithm itself. In that case, the metalearning process finds a representation $\hat{\rep} \in \repclass$ and everytime the specialization algorithm is called it performs ERM using the samples from the new task to find a $\per \in \perclass$ that minimizes the training error of $\per \circ \hat{\rep}$.

\begin{defn}[General Metalearning]
    Let $\repclass$ be a class of representation functions $\rep : \featdom \to \mapdom$ and $\perclass$ be a class of specialized classifiers $\per: \mapdom \to \bits$. We say that $(\repclass, \perclass)$ is \emph{$(\varepsilon, \delta)$-meta-learnable for $\users$ tasks, $\samples$ samples per task and $\samples_{\text{spec}}$ specialization samples}, if there exists an algorithm $\alg$ that for all metadistributions $\metadist$ over distributions on $\featdom \times \bits$ has the following property:
    Given $\samples$ i.i.d.\ samples from each data distribution $\dist_j$, where $j \in [\users]$, that was drawn independently from $\metadist$, $\alg$ returns a specialization algorithm $\alg'$. Algorithm $\alg'$, given a set $S$ of $\samples_{\text{spec}}$ samples from $\dist$ which was drawn from $\metadist$, returns hypothesis $g: \featdom \to \bits$ such that with probability at least $1-\delta$ over the randomness of the samples of the $\users$ tasks and algorithms $\alg$ and $\alg'$
    \[
    \Exp_{\dist , S }\left[ \er(\dist, \alg'(S)\right] \leq \min_{\rep \in \repclass} \Exp_{\dist } \left[\min_{\per \in \perclass}\er(\dist, \per \circ \rep) \right]+\eps.
    \]
\end{defn}

\subsection{Reductions between improper metalearning and multitask learning}
\label{sec:red}
We show how we can reduce metalearning to multitask learning. The algorithm we construct for the reduction stores the given $\users$ datasets. For every new task it performs multitask learning on the $\users+1$ datasets and returns the hypothesis that corresponds to the new task. We need to highlight that this is not a proper metalearning algorithm, since it does not return one representation that is used for all new tasks, but the algorithm builds the classifier from scratch every time it is given a new task.

\begin{thm}
\label{thm: meta-to-mtl} Suppose $(\repclass, \perclass)$ is $(\eps, \eps)$-multitask learnable for $\users + 1$ tasks and $\samples$ samples per task. Then, for all $c>0$, it is (improperly) $(c\eps, 2/c)$-meta-learnable, for $\users$ tasks, $\samples$ samples per task and $\samples$ specialization samples. 
\end{thm}

\begin{proof}
    Given a multitask learning algorithm $\mathcal{A}_{\text{multi}}$, we construct a metalearning algorithm $\mathcal{A}_{\text{meta}}$ that stores the $\users$ data sets and, when given the new dataset, runs $\mathcal{A}_{\text{multi}}$ on the collected $\users+1$ datasets. Formally, recall that a multitask learning algorithm sees the datasets of $\users+1$ tasks and outputs one hypothesis per task minimizing the average test error, while a metalearning algorithm sees $\users$ datasets from tasks drawn from the same metadistribution and outputs an algorithm $\alg_{\text{spec}}$ that can be used to train a model for a new task.
    
   In our reduction, the metalearning algorithm $\alg_{\text{meta}}$, on input data sets $S_1,...,S_\users$, returns an algorithm $\alg_{\text{spec}}$ which, given a set $S_{\text{new}}$ of $\samples$ samples from distribution $\dist_{\text{new}}$ for a new task, executes the following steps. First, it draws an index $j_{\text{new}}$ uniformly at random from $[\users+1]$ and inserts the new task in this position by setting $S_{j_{\text{new}}}' = S_{\text{new}}$. For $j < j_{\text{new}}$ it sets $S_j' = S_j$ and for $j > j_{\text{new}}$ $S_j' = S_{j+1}$. For notation purposes, let $P_j'$ be the distribution that $S_j'$ was drawn from. Second, it runs the multitask algorithm $\alg_{\text{multi}}$ for $\users+1$ tasks and $\samples$ samples for each of these tasks. At the end, it outputs the classifier it has computed for the $j_{\text{new}}$-th task.
   
   By our assumption, the guarantee that multitask learning gives is that $\alg_{\text{multi}}$ outputs $g_1, \ldots, g_{\users+1}$ such that with probability at least $1-\eps$ over the randomness of $\alg_{\text{multi}}$ and the $\users+1$ datasets $S_1', \ldots, S_{\users+1}'$
   \begin{align}
    \label{eq: multi-ineq}
       \frac{1}{\users+1}\sum_{j \in [\users+1]}\er(\dist_j', g_j) \leq \min_{\rep \in \repclass; \per_1, \ldots, \per_{\users+1}\in \perclass} \frac{1}{\users+1} \sum_{j \in [\users+1]} \er(\dist_j', \per_j \circ \rep)+\eps.
 \end{align}
Therefore, for all distributions $\dist_1', \ldots, \dist_{\users+1}'$ in the support of $\metadist$
\begin{align*}
    \Exp_{\alg_{\text{multi}}, S_1,\ldots,S_{\users+1}}\left[\frac{1}{\users+1}\sum_{j \in [\users+1]}\er(\dist_j', g_j)\right]\leq \min_{\rep \in \repclass; \per_1, \ldots, \per_{\users+1}\in \perclass} \frac{1}{\users+1} \sum_{j \in [\users+1]} \er(\dist_j', \per_j \circ \rep)+2\eps.
\end{align*}

We now want to compute the expected error of the metalearning algorithm over the randomness of the algorithm and the stored datasets. Since we draw index $j_{\text{new}}$ uniformly at random, 
%$j$ where we insert the new task independently, we get that
\begin{align*}
    &\Exp_{\alg_{\text{spec}}, P_1, \ldots, P_{\users}, S_1, \ldots, S_{\users}, }\left[\Exp_{\dist, S}\left[\er(\dist, \alg_{\text{spec}}(S))\right]\right] = \\
    &\Exp_{\alg_{\text{multi}}, j \sim \text{Unif}([\users+1]),  P_1, \ldots, P_{\users}, S_1, \ldots, S_{\users}}\left[\Exp_{\dist, S}\left[\er(\dist, \alg_{\text{multi}}(S'_1, \ldots, S'_{\users+1})_j)\right]\right]= \\
    &\Exp_{\alg_{\text{multi}}, \dist_1, \ldots, \dist_{\users}, S_1, \ldots, S_{\users}}\left[\frac{1}{\users+1}\sum_{j \in [\users+1]} \Exp_{\dist, S}\left[\er(\dist, \alg_{\text{multi}}(S'_1, \ldots, S'_{\users+1})_j)\right]\right]
         = \\&\Exp_{\alg_{\text{multi}}, \dist'_1, \ldots, \dist'_{\users+1}, S'_1, \ldots, S'_{\users+1}}\left[\frac{1}{\users+1}\sum_{j \in [\users+1]} \er(\dist'_j, \alg_{\text{multi}}(S'_1, \ldots, S'_{\users+1})_j)\right].
\end{align*}
Combining the above results we obtain the following inequality  
    \begin{align*}
        &\Exp_{\alg_{\text{spec}},P_1, \ldots, P_{\users}, S_1, \ldots, S_{\users}}\left[\Exp_{\dist, S}\left[\er(\dist, \alg_{\text{spec}}(S))\right]\right] 
         \leq\\& \Exp_{\dist'_1, \ldots, \dist'_{\users+1}}\left[\min_{\rep \in \repclass, \per_1, \ldots, \per_{\users+1}\in \perclass} \frac{1}{\users+1} \sum_{j \in [\users+1] }\er(\dist_j', \per_j\circ\rep)\right] + 2\eps
         \leq\\& \min_{\rep \in \repclass}\Exp_{\dist'_1, \ldots, \dist'_{\users+1}}\left[\frac{1}{\users+1}\sum_{j \in [\users+1] }\min_{\per_j \in \perclass}  \er(\dist_j', \per_j\circ\rep)\right] + 2\eps
          = \\&\min_{\rep \in \repclass}\Exp_{\dist}\left[\min_{\per \in \perclass}  \er(\dist, \per\circ\rep)\right] + 2\eps.
    \end{align*}
    Therefore, we have shown that \[\Exp_{\alg_{\text{spec}}, P_1, \ldots, P_{\users} S_1, \ldots, S_{\users}}\left[\Exp_{\dist, S}\left[\er(\dist, \alg_{\text{spec}}(S))\right]-\min_{\rep \in \repclass}\Exp_{\dist}\left[\min_{\per \in \perclass}  \er(\dist, \per\circ\rep)\right]\right] < 2\eps.\]
    We apply Markov's inequality to obtain the following probability
    \begin{align*}
        &\Pr[\alg_{\text{spec}}, P_1,\ldots, P_{\users}, S_1, \ldots, S_{\users}]{\Exp_{\dist, S}\left[\er(\dist, \alg_{\text{spec}}(S))\right]-\min_{\rep \in \repclass}\Exp_{\dist}\left[\min_{\per \in \perclass}  \er(\dist, \per\circ\rep)  \right]> c\eps}\\
        &\leq \frac{\Exp_{\alg_{\text{spec}},P_1,\ldots, P_{\users}, S_1, \ldots, S_{\users}}\left[\Exp_{\dist, S}\left[\er(\dist, \alg_{\text{spec}}(S))\right]-\min_{\rep \in \repclass}\Exp_{\dist}\left[\min_{\per \in \perclass}  \er(\dist, \per\circ\rep)\right]\right]}{c\eps} < \frac{2}{c}.
    \end{align*}

    To conclude, we showed that with probability at least $1-\frac{2}{c}$ over algorithm $\alg_{\text{spec}}$ 
    \[
    \Exp_{\dist, S}\left[\er(\dist, \alg_{\text{spec}}(S))\right]\leq\min_{\rep \in \repclass}\Exp_{\dist}\left[\min_{\per \in \perclass}  \er(\dist, \per\circ\rep)\right]+c\eps.
    \qedhere\]
\end{proof}

Corollary~\ref{cor:met-lin-agn} captures one application of this general reduction to the case of linear classes with $n$ roughly $k/\eps^2$ samples per task.

We can also show the opposite direction, how to reduce multitask learning to metalearning. We first use some of the samples from the tasks to learn a specialization algorithm and then we run this algorithm to learn one classifier per task. Since metalearning assumes the existence of a metadistribution, we set it to be the uniform distribution over the distributions of the seen tasks. As a result, when we draw tasks from the metadistribution we might draw the same task more than once, which means we need additional samples from that task.  We use a standard balls-and-bins argument (Lemma~\ref{lem:balls-and-bins}) to argue that we have enough samples per task to generate independent samples for metalearning.

\newcommand{\lemballsnbins}{Suppose we draw $t$ samples from a uniform distribution over $[t]$. Let $F_i$ denote the frequency of each $i \in [t]$ in the sample set. With probability $1-\delta$, $\max_i F_i$ is at most $\frac{\ln ((t/\delta)^2) }{\ln \ln ((t/\delta)^2)}$.}
\begin{lem}
\label{lem:balls-and-bins}
\lemballsnbins
\end{lem}

\begin{thm}
\label{thm: mtl-to-meta}
    Suppose $(\repclass, \perclass)$ is (improperly) $(\eps, \eps)$-metalearnable for $\users$ tasks, $\samples$ samples per task and $\samples_{\text{spec}}$ specialization samples. Then, for all $c > 0$, it is $(c \eps, 3/c)$-multitask learnable for $\users$ tasks and $\samples \lceil2\ln(c\users)\rceil + \samples_{\text{spec}}$ samples per task.
\end{thm}
\begin{proof}
     Given a metalearning algorithm $\alg_{\text{meta}}$ that returns a specialization algorithm $\alg_{\text{spec}}$, we construct a multitask learning algorithm $\alg_{\text{multi}}$.

     In our reduction, $\alg_{\text{multi}}$ gets as input $\users$ datasets $S_1, \ldots, S_\users$, where each $S_j$ has $n \lceil\ln(c\users)\rceil+\samples_{\text{spec}}$ samples drawn i.i.d.\ from distribution $\dist_j$, and executes the following steps:
        \begin{enumerate}
            \item Draw $\users$ indices of tasks $i_1, \ldots, i_{\users}$ uniformly at random from $[\users]$ with replacement. A pair of tasks $i_j, i_{j'}$ might correspond to the same original task.
            \item\label{step:trainsets} Build datasets $S_1', \ldots, S_{\users}'$ as follows.  For each $j$, if there are no samples remaining in $S_{i_{j}}$, fail and output $\bot$, otherwise let $S_{j}'$ be the next $\samples$ samples from $S_{i_{j}}$ and delete those samples from $S_{i_j}$.
            \item Run $\alg_{\text{meta}}(S_1', \ldots, S_\users')$ which returns $\alg_{\text{spec}}$.
            \item For every $j \in [\users]$ define dataset $S^{\text{spec}}_{j}$ as a set of $\samples_{\text{spec}}$ datapoints from $S_j$. As in Step \ref{step:trainsets}, avoid using the same datapoint twice (and output $\bot$ if that's impossible).
            \item For every task $j$, set $g_j = \alg_{\text{spec}}(S^{\text{spec}}_{j})$.
        \end{enumerate}

    The goal in multitask learning is to bound the probability 
    \[
    \Pr[S_1, \ldots, S_\users, \alg_{\text{multi}}]{\frac{1}{\users}\sum_{j \in [\users]}\er(\dist_j,g_j) \leq \min_{\rep \in \repclass, \per_1, \ldots, \per_{\users} \in \perclass} \frac{1}{\users}\sum_{j \in [\users]}\er(\dist_j, \per \circ \rep) + c\eps}
    \]
    for all $\dist_1, \ldots, \dist_\users$.
    
    Fix $\users$ tasks $\dist_1, \ldots, \dist_\users$. The multitask learning algorithm receives a dataset per task $S_j$ of $\samples \lceil \ln(c\users)\rceil + \samples_{\text{spec}}$ samples as input, and in steps 1 and 2 constructs an appropriate input for the metalearning algorithm. Let $\mathcal{B}$ be the event that the algorithm outputs $\bot$. In other words, it's the event it drew some task $j$ more than $\lceil2\ln(c\users)\rceil$ times in step 1. By Lemma~\ref{lem:balls-and-bins} we know that $\Pr[S_1,\ldots, S_\users,\alg_{\text{multi}}]{\mathcal{B}}\leq \frac{1}{c}$.

     By our assumption, if $\mathcal{B}$ does not hold, then $\alg_{\text{meta}}$ returns specialization algorithm $\alg_{\text{spec}}$ with the guarantee that with probability at least $1-\eps$ over the randomness of the samples and the algorithms 
        \[
        \Exp_{\dist, S}\left[\er(\dist, (\alg_{\text{meta}}(S_1', \ldots, S_{\users}'))(S))\right]\leq\min_{\rep \in \repclass}\Exp_{\dist}\left[\min_{\per \in \perclass}  \er(\dist, \per\circ\rep)\right]+\eps,
        \]
    where $\dist$ is $\dist_j$ for $j$ drawn uniformly at random from $[\users]$.

    Thus, 
     \[
    \Exp_{i_1, \ldots, i_\users, S_1', \ldots, S_\users', \alg_{\text{meta}}, \alg_{\text{spec}}}\left[\Exp_{j, S}\left[\left.\er(\dist_j, (\alg_{\text{meta}}(S_1', \ldots, S_{\users}'))(S))\right|\text{ not }\mathcal{B}\right]\right]\leq\min_{\rep \in \repclass}\Exp_{j}\left[\min_{\per \in \perclass}  \er(\dist_j, \per\circ\rep)\right]+2\eps.
    \]

    Conditioning on $\mathcal{B}$ not happening, we see that 
    \begin{align*}
        &\Exp_{\substack{S_1, \ldots, S_{\users},\\ \alg_{\text{multi}}}}\left[\left.\frac{1}{\users}\sum_{j \in [\users]}\er(\dist_j,g_j) \right|\text{ not }\mathcal{B}\right] = \\
        & \Exp_{\substack{S_1, \ldots, S_{\users},\\ i_1 ,\ldots, i_\users,\\\alg_{\text{meta}}, \alg_{\text{spec}}}}\left[\left.\frac{1}{\users}\sum_{j \in [\users]}\er(\dist_j,\alg_{\text{meta}}(S_1', \ldots, S_\users')(S_j^{\text{spec}}) \right|\text{ not }\mathcal{B}\right] =\\
        & \Exp_{\substack{S_1\setminus S_1^{\text{spec}}, \ldots, S_{\users}\setminus S_{\users}^{\text{spec}},\\ i_1 ,\ldots, i_\users ,\\\alg_{\text{meta}}, \alg_{\text{spec}}}}\left[\left.\frac{1}{\users}\sum_{j \in [\users]}\Exp_{S_j^{\text{spec}}}\left[\er(\dist_j,\alg_{\text{meta}}(S_1', \ldots, S_\users')(S_j^{\text{spec}})\right] \right|\text{ not }\mathcal{B}\right] =\\
         & \Exp_{\substack{S_1\setminus S_1^{\text{spec}}, \ldots, S_{\users}\setminus S_{\users}^{\text{spec}},\\ i_1,\ldots, i_\users,\\\alg_{\text{meta}}, \alg_{\text{spec}}}}\left[\left.\Exp_{j,S}\left[\er(\dist_j,\alg_{\text{meta}}(S_1', \ldots, S_\users')(S)\right] \right|\text{ not }\mathcal{B}\right] =\\
        & \Exp_{\substack{i_1, \ldots, i_\users,\\ S_1', \ldots, S_\users',\\ \alg_{\text{meta}}, \alg_{\text{spec}}}}\left[\left.\Exp_{j,S}\left[\er(\dist_j,\alg_{\text{meta}}(S_1', \ldots, S_\users')(S)\right] \right |\text{ not }\mathcal{B}\right] \leq \\
        & \min_{\rep \in \repclass}\Exp_{j}\left[ \min_{\per \in \perclass} \er(\dist_J, \per \circ \rep)\right] + 2\eps =\\
        & \min_{\rep \in \repclass, \per_1, \ldots, \per_\users \in \perclass} \frac{1}{\users} \sum_{j \in [\users]}\er(\dist_j, \per_j \circ \rep)+2\eps.
    \end{align*}
    By Markov's inequality we can bound the corresponding probability
    \begin{align*}
        &\Pr[S_1, \ldots, S_\users, \alg_{\text{multi}}]{\left.\frac{1}{\users}\sum_{j \in [\users]} \er (\dist_j, g_j) - \min_{\rep \in \repclass, \per_1, \ldots, \per_{\users}\in \perclass} \frac{1}{\users}\sum_{j \in [\users]} \er(\dist_j, \per \circ \rep) > c\eps \right| \text{ not }\mathcal{B}} \leq\\
        &\frac{1}{c\eps}\cdot \Exp_{S_1, \ldots, S_\users, \alg_{\text{multi}}}\left[\left.\frac{1}{\users}\sum_{j \in [\users]} \er (\dist_j, g_j) - \min_{\rep \in \repclass, \per_1, \ldots, \per_{\users}\in \perclass} \frac{1}{\users}\sum_{j \in [\users]} \er(\dist_j, \per \circ \rep) > c\eps \right| \text{not }\mathcal{B}\right] \leq \frac{2}{c}.
    \end{align*}
    Combining the results above we can show that for all $\dist_1, \ldots, \dist_\users$
    \begin{align*}
        &\Pr[S_1, \ldots, S_\users, \alg_{\text{multi}}]{\frac{1}{\users}\sum_{j \in [\users]} \er(\dist_j, g_j) > \min_{\rep \in \repclass, \per_1, \ldots, \per_\users \in \perclass} \frac{1}{\users}\sum_{j \in [\users]}\er(\dist_j, \per \circ \rep) + c\eps} \leq \\
        &\Pr[S_1, \ldots, S_\users, \alg_{\text{multi}}]{\left.\frac{1}{\users}\sum_{j \in [\users]} \er (\dist_j, g_j) - \min_{\rep \in \repclass, \per_1, \ldots, \per_{\users}\in \perclass} \frac{1}{\users}\sum_{j \in [\users]} \er(\dist_j, \per \circ \rep) > c\eps \right| \text{ not }\mathcal{B}} \\
        &+ \Pr[S_1, \ldots, S_\users, \alg_{\text{multi}}]{\mathcal{B}} \leq \frac{3}{c}.
    \end{align*}
    This concludes the proof.
\end{proof}
\fi

\ifnum \coltshort = 0
\section{Multitask Learning}\label{sec:mtl}
In this section, we bound the number of tasks and samples per task we need to multitask learn using the VC dimension. Moreover, we provide upper and lower bounds of this VC dimension for general classes of representations and specialized classifiers.

\subsection{Sample and Task Complexity Bounds for General Classes}

Here we prove that the VC dimension of the composite class $\perclass^{\otimes \users}\circ \repclass$ determines the total number of samples required for multitask learning.

\newcommand{\thmmtlsmpls}{For any $(\repclass,\perclass)$, any $\eps,\delta > 0$, and any $\users, \samples$,

    \begin{enumerate}
        \item In the realizable case, $(\repclass, \perclass)$ is $(\eps,\delta)$-multitask learnable with $\users$ tasks and $\samples$ samples per task when $\samples \users = O(\frac{\VC(\perclass^{\otimes \users}\circ \repclass) \cdot \ln(1/\varepsilon)+\ln\left(1/\delta\right)}{\varepsilon})$.
        % \item If $\samples \users = O(\frac{\VC(\perclass^{\otimes \users}\circ \repclass) \cdot \ln(1/\varepsilon)+\ln\left(1/\delta\right)}{\varepsilon})$, then $(\repclass, \perclass)$ is $(\eps,\delta)$-multitask learnable in the realizable case with $\users$ tasks and $\samples$ samples per task.

        \item In the agnostic case, $(\repclass, \perclass)$ is $(\eps,\delta)$-multitask learnable with $\users$ tasks and $\samples$ samples per task when $\samples \users = O(\frac{\VC(\perclass^{\otimes \users}\circ \repclass) \cdot \ln\left(1/\delta\right)}{\varepsilon^2})$.
        % \item If $\samples \users = O(\frac{\VC(\perclass^{\otimes \users}\circ \repclass) \cdot \ln\left(1/\delta\right)}{\varepsilon^2})$, then $(\repclass, \perclass)$ is $(\eps,\delta)$-multitask learnable with $\users$ tasks and $\samples$ samples per task.

        \item If $ \samples \users \leq \frac14 \cdot \VC(\perclass^{\otimes \users}\circ \repclass)$ then $(\repclass,\perclass)$ is \emph{not} $(1/8,1/8)$-multitask learnable with $\users$ tasks and $\samples$ samples per task, even in the realizable case.
    \end{enumerate}
}

\begin{thm} \label{thm:mtl-smpls}
   \thmmtlsmpls
\end{thm}

The proof of this theorem\ifnum \coltshort = 0, presented in \Cref{sec:mtl-smples-proof},\fi\ closely follows the standard proofs on VC dimension which upper and lower bound on the sample complexity of distribution-free classification.  The main difference is that the samples we analyze, while independent, are not identically distributed. 
This is because each sample comes from a specific task distribution and these distributions are potentially different.
\ifnum \coltshort = 1
\begin{proof}
We’ll start by proving part 1, which covers the realizable case.  We will omit the proof of part 2, which generalizes the realizable case in a standard way. Finally, we will prove part 3.

\medskip\emph{Proof of 1.}
For every task $j \in [\users]$ we have $\samples$ i.i.d.\ samples $S_j = \{(\feat_i^{(j)}, \lab_i^{(j)})\}_{i \in [\samples]}$ drawn from $\dist_j$. Our dataset is equivalent to dataset $S = \{(j, \feat_i^{(j)}, \lab_i^{(j)})\}_{i \in [\samples], j \in [\users]}$, where the $j$s have fixed values. In standard PAC learning we assume that all the datapoints are i.i.d. However, in our case the samples are independent but not identically distributed because different tasks have (potentially) different distributions.

As in standard PAC learning, our proof follows the ``double-sampling trick''. We want to bound the probability of bad event 
\begin{align*}
    B: \exists g \in \perclass^{\otimes}\circ\repclass \textrm{ s.t. } \er(S, g)=0,\textrm{ but }\frac{1}{\users} \sum_{j \in [\users]} \er(\dist_j, g(j,\cdot)) > \varepsilon.
\end{align*}

We consider an auxiliary dataset $\hat{S} = \{(j, \hat{\feat}_i^{(j)}, \hat{\lab}_i^{(j)})\}_{i \in [\samples], j \in [\users]}$, where $(\hat{\feat}_i^{(j)},  \hat{\lab}_i^{(j)})$ are drawn independently from $\dist_j$. In our proof we will bound the probability of $B$ by bounding the probability of event 
\begin{align*}
    B': \exists g \in \perclass^{\otimes}\circ\repclass \textrm{ s.t. } \er(S, g)=0,\textrm{ and }  \er(\hat{S}, g) > \frac{ \varepsilon}{2}.
\end{align*}

We can show that if $\samples \users > \frac{8}{\varepsilon}$, then $\Pr[S]{B} \leq 2 \Pr[S, \hat{S}]{B'}$. We can show this by applying a multiplicative Chernoff bound on $\er(\hat{S},g)$, which is an average of independent random variables. Due to this step, it suffices to bound $\Pr[S, \hat{S}]{B'}$.

We also define a third event $B''$ as follows. We give $S$ and $\hat{S}$ as inputs to randomized process \textit{Swap} which iterates over $j \in [\users]$ and $i \in [\samples]$ and at every step it swaps $(\feat_i^{(j)}, \lab_i^{(j)})$ with $(\hat{\feat}_i^{(j)}, \hat{\lab}_i^{(j)})$ with probability $1/2$. Let $T$ and $\hat{T}$ be the two datasets this process outputs. We define event $$B'': \exists g \in \perclass^{\otimes \users}\circ \repclass \textrm{ s.t. } \er(T,g) = 0 \textrm{ and } \er(\hat{T}, g) > \frac{\varepsilon}{2}.$$

We see that $\Pr[S, \hat{S}, \textit{Swap}]{B''} = \Pr[S, \hat{S}]{B'}$. This happens because $T$, $\hat{T}$, $S$ and $\hat{S}$ are identically distributed. Thus, what we need to do now is bound $\Pr[S, \hat{S}, \textit{Swap}]{B''}$. 

 We start by showing that for a fixed $g$ $$\Pr[\textit{Swap}]{ \er(T,g) = 0 \textrm{ and } \er(\hat{T},g) > \frac{\varepsilon}{2}\mid S, \hat{S}} \leq 2^{-nt\varepsilon/2}.$$ Given $S$ and $\hat{S}$, $B''$ happens if for every $j \in [\users]$ and $i \in [\samples]$ $g$ predicts the label of $\feat_i^{(j)}$ or $\hat{\feat}_i^{(j)}$ correctly and makes $m > \varepsilon \samples \users/2$ mistakes overall. Additionally, all $m$ mistakes $g$ makes are in dataset $\hat{T}$. This means that \textit{Swap} assigns all these points to $\hat{T}$, which happens with probability $1/2^m \leq 1/2^{\varepsilon \samples \users/2}$. 

 Let $\left(\perclass^{\otimes \users}\circ \repclass\right) (S \cup \hat{S}) \subset \perclass^{\otimes \users}\circ \repclass$ be a set of hypotheses which contains one hypothesis for every labeling of $S\cup\hat{S}$. Then, 

\begin{align*}
    \Pr[S, \hat{S}, \textit{Swap}]{B''}& = \Exp_{S,\hat{S}}\left[\Pr[\textit{Swap}]{ \exists g \in  \perclass^{\otimes \users}\circ \repclass \textrm{ s.t. }\er(T,g) = 0 \textrm{ and } \er(\hat{T},g) > \frac{\varepsilon}{2}\mid S, \hat{S}}\right]\\
    &\leq \Exp_{S,\hat{S}}\left[\sum_{g \in\perclass^{\otimes \users}\circ \repclass (S \cup \hat{S}) }\Pr[\textit{Swap}]{ \er(T,g) = 0 \textrm{ and } \er(\hat{T},g) > \frac{\varepsilon}{2}\mid S, \hat{S}}\right]\\
    & \leq \growth_{\perclass^{\otimes \users}\circ \repclass}(2\samples \users) 2^{-\varepsilon \samples \users/2}.
\end{align*}

For the probability of bad event $B$ happening to be at most $\delta$, by the steps above we need $\samples \users \geq 2 \frac{\log_2 \growth_{\perclass^{\otimes \users}\circ \repclass}(2\samples \users) + \log_2(2/\delta)}{\varepsilon}$ samples in total. 

By Lemma \ref{lem:sauer},  for $\samples\users > \VC(\perclass^{\otimes \users}\circ \repclass)$,
$$\log_2\paren{\growth_{\perclass^{\otimes \users}\circ \repclass}(2nt)}
\leq 
\VC(\perclass^{\otimes \users}\circ \repclass)
\log_2\left(\frac{e 2\samples \users}{\VC(\perclass^{\otimes \users}\circ \repclass)}\right)
.$$
One can show that if $\samples \users \geq \frac{\VC(\perclass^{\otimes \users}\circ \repclass)}{\varepsilon}\log_2\left(\frac{2e}{\varepsilon}\right)$, then $\samples \users \geq \log_2\left(\frac{e 2\samples \users}{\VC(\perclass^{\otimes \users}\circ \repclass)}\right)\frac{\VC(\perclass^{\otimes \users}\circ \repclass)}{\varepsilon}.$ Therefore, we get that for $\samples \users \geq \frac{\VC(\perclass^{\otimes \users}\circ \repclass)\log_2(2e/\varepsilon)+\log_2(2/\delta)}{\varepsilon}$ samples in total the probability of bad event $B$ is at most $\delta$. 

\medskip\emph{Proof of 3.}
    Our goal is to construct distributions $\dist_1,\ldots, \dist_{\users}$ over $\featdom \times \{\pm 1\}$. We will first build their support and then define the probability distributions. 
    
    Let $\vc = \VC(\perclass^{\otimes \users}\circ \repclass)$. Since
    $4\samples \users \leq \vc$, there exists a dataset $S = \{(j_i,x_i)\}_{i \in [4\samples \users]}$ that can be shattered by $\perclass^{\otimes \users}\circ \repclass$. 
    Let $S_j = \{(j_i,x_i) \in S\mid j_i = j\}$. In general for every task $j$ the size of $S_j$, i.e.\ $|S_j|$, will be different. 
    We want to use $S$ to get a new dataset $S'$ that can be shattered by $\perclass^{\otimes \users}\circ \repclass$, but also has at least $2\samples$  
    points for every task. 
    For every $j \in [\users]$ we throw away points from $S_j$ until we have a multiple of $2\samples$. After doing this, we have thrown away at most $2\samples$ per task, which is at most $2\samples \users$ points in total. We can redistribute points so that we have at least $2\samples$ points per task as follows. For every task $j \in [\users]$, if there are no points in this task, there must be another task $k$ with at least $4\samples \users$ points. In this case, we move $\samples \users$ points from task $k$ to task $j$ by replacing $k$ with $j$ in $(k,x_i)$. 
    We denote
    this new dataset by $S'$ and the subset for task $j$, by $S_j'$. 
    
    We claim that $S'$
    , which has at least $2\samples$ points per task, can also be shattered by $\perclass^{\otimes \users}\circ \repclass$. To see this, suppose that $S$ was shattered by $\tilde{g}$. Now, if task $j$ did not lose all its points, the remaining points can be shattered by $\tilde{g}(j,\cdot)$. Otherwise, we are in the case where $j$'s points were initially assigned to task $k$, so they can be shattered by $\tilde{g}(k,\cdot)$.

    Let $\paren{\perclass^{\otimes \users}\circ \repclass}(S')$ be a set of hypotheses that contains one function for each labeling of $S'$. We choose a labeling function $g$ uniformly at random from $\paren{\perclass^{\otimes \users}\circ \repclass}(S')$. For all tasks $j \in [\users]$ we define $\dist_j$ 
    as the distribution of $(x,y)$ obtained by sampling $x$ uniformly from $S'_j$ (ignoring the task index in the sample) and labeling it according to $g$.

    Suppose that the (potentially randomized) learning algorithm $\alg$ returns $\hat{g}$ after seeing datasets $\hat{S}_1, \ldots, \hat{S}_\users$, where every $\hat{S}_j$ has $\samples$ points drawn i.i.d.\ from $\dist_j$. For a fixed task $j$, the probability that $\hat{g}$ makes a mistake on a new point $(x,y)$ drawn from $\dist_j$ is 
    \begin{align*}
        \Pr[g,\hat{S}_j, (x,y) \sim \dist_j]{\hat{g}(j,x) \neq y} & \geq \Pr[g,\hat{S}_j, (x,y)\sim \dist_j]{\hat{g}(j,x) \neq y \textrm{ and } x \notin \hat{S}_j}\\
        & = \Pr[g,\hat{S}_j, (x,y)\sim \dist_j]{x \notin \hat{S}}\Pr[g,\hat{S}, (x,y)\sim \dist_j]{\hat{g}(j,x) \neq y\mid x \notin \hat{S}_j}.
    \end{align*}
    
    We know that $\Pr[\hat{S}_j, (x,y)\sim \dist_j]{x \notin \hat{S}}\geq 1/2$ because $x$ is chosen uniformly at random out of more than $2\samples$ points that are in $S_j'$ and $\hat{S}_j$ has only $\samples$
    points. When $x \notin \hat{S}_j$ the algorithm has not seen the label that corresponds to this $x$ and, thus, $y=g(j,x)$ is independent of $\hat{g}(j,x)$. We have picked $g$ uniformly at random from a class with exactly one function per labeling, which means that for each $(j,x)$ we see $+1$ and $-1$ with equal probability. Therefore, $\Pr[g,\hat{S}_j, (x,y)\sim \dist_j]{\hat{g}(j,x) \neq y\mid x \notin \hat{S}_j} = 1/2$. Thus, $\Pr[g,\hat{S}_j, (x,y) \sim \dist_j]{\hat{g}(j,x) \neq y} \geq 1/4$.

    The average error is 
    \begin{align*}
        \frac{1}{\users} \sum_{j \in [\users]} \Pr[g,\hat{S}_j, (x,y) \sim \dist_j]{\hat{g}(j,x) \neq y} \geq \frac{1}{4}.
    \end{align*}
    In expectation over the labeling functions $g$ and the randomness of algorithm $\alg$, we have 
    \begin{align*}
        \Exp_{g, \alg}\left[\frac{1}{\users} \sum_{j \in [\users]} \Pr[\hat{S}_j, (x,y) \sim \dist_j]{\hat{g}(j,x) \neq y}\right] \geq \frac{1}{4}
    \end{align*}
    Hence, there exists a labeling function $g$ in $\perclass^{\otimes \users}\circ \repclass(S')$ such that \[\Exp_{\alg}\left[\frac{1}{\users} \sum_{j \in [\users]} \Pr[\hat{S}_j, (x,y) \sim \dist_j]{\hat{g}(j,x) \neq y}\right] \geq \frac{1}{4}.\] For this $g$ we have that 
    \begin{align*}
        &\Exp_{\alg}\left[\frac{1}{\users} \sum_{j \in [\users]} \Pr[\hat{S}_j, (x,y) \sim \dist_j]{\hat{g}(j,x) \neq y}\right]\\
        &= \Exp_{\alg}\left[\frac{1}{\users} \sum_{j \in [\users]} \Exp_{\hat{S}_j}[\er(\dist_j,\hat{g}(j,\cdot))]\right]\\
        & = \Exp_{\alg,\hat{S}_1, \ldots, \hat{S}_\users}\left[\frac{1}{\users} \sum_{j \in [\users]}\er(\dist_j,\hat{g}(j,\cdot))\right]\\
        & = \Pr[\alg,\hat{S}_1, \ldots, \hat{S}_\users]{\frac{1}{\users} \sum_{j \in [\users]}\er(\dist_j,\hat{g}(j,\cdot))>\frac{1}{8}}\Exp_{\alg,\hat{S}_1, \ldots, \hat{S}_\users}\left[\frac{1}{\users} \sum_{j \in [\users]}\er(\dist_j,\hat{g}(j,\cdot)) \left| \sum_{j \in [\users]}\er(\dist_j,\hat{g}(j,\cdot)) > \frac{1}{8}\right.\right]\\
        &+\Pr[\alg,\hat{S}_1, \ldots, \hat{S}_\users]{\frac{1}{\users} \sum_{j \in [\users]}\er(\dist_j,\hat{g}(j,\cdot))<\frac{1}{8}}\Exp_{\alg,\hat{S}_1, \ldots, \hat{S}_\users}\left[\frac{1}{\users} \sum_{j \in [\users]}\er(\dist_j,\hat{g}(j,\cdot)) \left| \sum_{j \in [\users]}\er(\dist_j,\hat{g}(j,\cdot)) < \frac{1}{8}\right.\right]\\
        &\leq \Pr[\alg,\hat{S}_1, \ldots, \hat{S}_\users]{\frac{1}{\users} \sum_{j \in [\users]}\er(\dist_j,\hat{g}(j,\cdot))>\frac{1}{8}} + \frac{1}{8}
    \end{align*}
    Thus, we showed that there exist $\dist_1, \ldots, \dist_\users$ such that $\Pr[\alg,\hat{S}_1, \ldots, \hat{S}_\users]{\frac{1}{\users} \sum_{j \in [\users]}\er(\dist_j,\hat{g}(j,\cdot))>\frac{1}{8}} \geq \frac{1}{8}$.
\end{proof}
\fi

\subsection{Generic Bounds on \texorpdfstring{$\VC(\perclass^{\otimes \users}\circ \repclass)$}{Lg}}
\label{sec:generic-VC-composite}

Ideally, given $\repclass$ and $\perclass$ we would like to be able to determine the VC dimension of the composite class $\perclass^{\otimes \users}\circ \repclass$. 
For general classes, we can state a few simple upper and lower bounds.
For one lower bound, consider the \emph{easier} problem where someone gives us the representation $\rep$ and we only need to find the specialized classifiers $\per_1,\dots,\per_\users$ for each task.  Even in this simpler setting we need $\VC(\perclass)$ samples per task, and thus the total sample complexity is $t \VC(\perclass)$.  For another lower bound, consider the easier problem where all specialized classifiers are the same, so the data is labeled by a single concept $\per \circ \rep \in \perclass \circ \repclass$, which requires total sample complexity $\VC(\perclass \circ \repclass)$.

For an upper bound, there is a na\"ive strategy for multitask learning where we treat the data for each task in isolation and simply try to learn a classifier from $\perclass \circ \repclass$. This requires sample complexity $\VC(\perclass \circ \repclass)$ per task.

Lemma~\ref{lem:vc-bounds} formalizes these results.\ifnum \coltshort = 0
Its proof appears in \Cref{sec:vc-bounds-proof}.
\fi
% In \Cref{lem:vc-bounds} we show that these simple upper and lower bounds can be captured by the VC dimension of the composed class. See \Cref{sec:vc-bounds-proof} for the proof of this lemma.

\newcommand{\lemvcbounds}{For any representation class $\repclass$ that contains a surjective function, any class of specialized classifiers $\perclass$, and any $\users$,
\[
    \max \left\{\users \cdot \VC(\perclass), \VC(\perclass \circ \repclass) \right\} \leq \VC(\perclass^{\otimes \users}\circ \repclass) \leq \users \cdot \VC(\perclass \circ \repclass).
\]
The upper bound still holds even if $\repclass$ does not contain a surjective function.}

\begin{lem}
\label{lem:vc-bounds}
\lemvcbounds
\end{lem}
\ifnum \coltshort = 1
\begin{proof}
    We will show the two parts of the statement separately.

    \medskip\emph{Proof of the upper bound on VC dimension.}
    Assume that we have a dataset $X= \left( (j_1, \feat_1), \ldots, (j_n,\feat_n)\right)$ of size $n = \VC(\perclass^{\otimes \users}\circ \repclass)$ which can be shattered by $\perclass^{\otimes \users}\circ \repclass$. We split it into $\users$ disjoint datasets $X_1,\ldots,X_\users$, where $X_j = \{ \feat_i: (j,\feat_i) \in X\}$. Each one of these datasets can be shattered by $\perclass\circ \repclass$.
    
    Let $n_j$ be the size of dataset $X_j$. Then, we have that $n_j \leq \VC(\perclass \circ \repclass)$. As a result, we obtain $\VC(\perclass^{\otimes \users}\circ \repclass) = \sum_{j \in [\users]}n_j \leq \users\VC(\perclass \circ \repclass)$.
    
    \medskip\emph{Proof of the lower bound on VC dimension.}
    We will first show that $\VC(\perclass^{\otimes \users}\circ \repclass) \geq \VC(\perclass \circ \repclass)$. Suppose $X = (\feat_1, \ldots, \feat_n)$ is a dataset of size $n$ that can be shattered by class $\perclass\circ \repclass$. 
    % This means that for any labeling $(\lab_1, \ldots, \lab_n)$ there exist $\per^* \in \perclass$ and $\rep^* \in \repclass$ such that $\forall i \in [n]$ $\lab_i = \per^*(\rep^*(\feat_i))$. 
    Then, for any $j_1, \ldots, j_n \in [\users]$, the dataset $((j_1, \feat_1), \ldots, (j_n, \feat_n))$ can be shattered by $\perclass^{\otimes \users}\circ \repclass$. 
    To see this, fix a labeling $(\lab_1, \ldots, \lab_n)$ of $((j_1, \feat_1), \ldots, (j_n, \feat_n))$.
    Since we assumed $X$ could be shattered by $\perclass\circ \repclass$, there exists $\per^* \in \perclass$ and $\rep^*\in\repclass$ such that $\forall i \in [n]$, $\lab_i = \per^*(\rep^*(\feat_i))$. 
    Thus, there is a function in $\perclass^{\otimes t}\circ \repclass$ (namely, the one with  representation $\rep^*$ and all $t$ personalization functions equal to $\per^*$) that realizes this labeling.
    Therefore, $\VC(\perclass^{\otimes \users}\circ \repclass) \geq \VC(\perclass \circ \repclass)$.

    Next, we will prove that $\VC(\perclass^{\otimes \users}\circ \repclass) \geq \users  \VC(\perclass)$. Let $(\map_1, \ldots, \map_n) \in \mapdom^n$ be a dataset that $\perclass$ can shatter. Since there exists an $\rep \in \repclass$ whose image is $\mapdom$, there exist $(\feat_1, \ldots, \feat_n) \in \featdom^n$ such that $\forall i \in [n]$ $\rep(\feat_i)=\map_i$. 
    We now construct a new dataset $\cup_{j \in [\users]}\left\{(j,\feat_1)\ldots, (j,\feat_n)\right\}$, which has $n\users$ datapoints. 
    Our function class $\perclass^{\otimes \users}\circ \repclass$ can shatter this dataset. 
    To see this: for any labeling we split the dataset to $\users$ parts according to the value of $j$, use $\rep$ to get $(\map_1, \ldots, \map_n)$ for each part. We then label each part using an $f \in \perclass$. This means that there exists a dataset of size $\users  \VC(\perclass)$ that $\perclass^{\otimes \users}\circ \repclass$ can shatter. Thus, $\VC(\perclass^{\otimes \users}\circ \repclass) \geq \users  \VC(\perclass)$.
\end{proof}
\fi
See \Cref{sec:finite_classes} for more precise bounds that hold for finite classes.

\fi
\section{Metalearning}
\label{sec:meta}
\ifnum \coltshort = 0
In this section, we bound the number of tasks and samples per task needed to metalearn general classes of representations and specialized classifiers. After laying out the basic definitions (\Cref{sec:meta-model}), we illustrate the techniques used to achieve these results via the special case of monotone thresholds applied to 1-dimensional representations (\Cref{sec:real-tech}).  This special case corresponds to a natural setting where the representation assigns a real-valued score to each example, but the threshold for converting that score into a binary label may vary from task to task.  We then provide the main theorems and their proofs for the realizable (\Cref{sec:metalearn-realizable}) and the agnostic cases (\Cref{sec:metalearning_agnostic}), respectively.
\else In this section, we define the metalearning model and outline our arguments bounding the number of tasks and samples per task needed to properly metalearn general classes of representations and specialized classifiers. %After laying out the basic definitions (\Cref{sec:meta-model}), we illustrate the techniques used to achieve these results and provide the main theorems for the realizable and the agnostic cases (\Cref{sec:metalearn-realizable}). 
For complete proofs and additional discussion see \Cref{sec:meta_app}.
\fi 

\ifnum \coltshort = 1

\subsection{The Proper Metalearning Model}
\label{sec:meta-model}
Suppose there exists a distribution $\metadist$ of different but related binary classification tasks. For each task, specified by a distribution $\dist$, we can draw a dataset of labeled examples in $\featdom \times \bits$. In proper metalearning, we draw $\users$ tasks from distribution $\metadist$ and pool together the data from these tasks. We exploit the relatedness between the tasks to learn a common representation $\rep: \featdom \to \mapdom$ in $\repclass$ that maps the features to a new domain, in the hope that the learned representation facilitates learning new tasks from the same distribution $\metadist$ with fewer samples compared to the baseline where we had to learn those tasks from scratch.

The accuracy of the representation is measured based on how well we can use it to solve a new binary classification task drawn from the same distribution as the ones we have seen in the metalearning phase. The data efficiency of the metalearning algorithm is measured based on two parameters, the number of tasks and the number of samples per task. The \ktext{proper} metalearning setting is particularly interesting when the number of samples per task is too small to learn a good classifier for each task independently, so data must be pooled across tasks to find a representation that can be specialized to new tasks using less data than we would need to learn from scratch.

\begin{defn}[Proper Metalearning]\label{def:metalearn}
Let $\repclass$ be a class of representation functions $\rep: \featdom \to \mapdom$ and $\perclass$ be a class of specialized classifiers $\per: \mapdom \to \bits$. We say that $(\repclass,\perclass)$
is \emph{$(\varepsilon, \delta)$-\ktext{properly} meta-learnable for $\users$ tasks and $\samples$ samples per task}, if there exists an algorithm $\alg$ that for all metadistributions $\metadist$ over distributions on $\featdom \times \bits$ has the following property:
Given $\samples$ i.i.d.\ samples per task distribution $\dist_j$, where $j \in [\users]$, that was drawn indepedently from $\metadist$, $\alg$  returns representation $\hat{\rep} \in \repclass$ such that with probability at least $1-\delta$ over the randomness of the samples and the algorithm 
\begin{align}\label{eq:metalearn}
\rer\paren{\metadist, \hat{\rep}, \perclass} \leq \min_{\rep \in \repclass}\rer\paren{\metadist, \rep, \perclass} + \varepsilon.
\end{align}

If there exists an algorithm $\alg$ such that the same guarantee holds only for all metadistributions $\metadist$ such that
$ \min_{\rep \in \repclass} \rer(\metadist, \rep, \perclass) = 0$
then we say that $(\repclass, \perclass)$ is \emph{$(\eps, \delta)$-\ktext{properly} meta-learnable for $\users$ tasks and $\samples$ samples per task in the realizable case}.
\end{defn}

%There is another aspect of efficiency that comes into effect after \ktext{proper} metalearning: given the learned representation, how many samples do we need to learn a new task? The number of samples, in this case, is merely determined by the choice of the specialized classifier class. In this paper, we focus on the efficiency of the \ktext{proper} metalearning algorithm that finds a good representation.

Definition~\ref{def:metalearn} only requires us to learn a good representation $\hat\rep$, and does not explicitly discuss the number of samples required to learn a good specialized classifier for a new task.  However, the sample complexity for specialization is essentially determined by the complexity of $\perclass$.  That is, given a new task $\dist'$ drawn from $\metadist$ and a representation $\hat{h}$ that satisfies \eqref{eq:metalearn}, $n_{\text{spec}}=O(\VC(\perclass)/\eps^2)$ suffice in order to find an $\per\in\perclass$ such that $\per\circ \hat\rep$ has error $\rer(\metadist, \hat{\rep}, \perclass) + 2\eps$. The size $n_{\text{spec}}$ of the data used for specialization could be much larger or smaller than the per-task training set size $n$.
\fi

\ifnum \coltshort = 0
\subsection{Warm Up and Techniques Overview}
\label{sec:real-tech}

In this subsection, we describe the technical tools we use to bound the number of samples and tasks needed to find a good representation $\hat{\rep}$.
As a warm-up, we first prove Theorem~\ref{thm:met-mon}, which covers metalearning for the class of monotone thresholds over linear representations in realizable case with minimal sample complexity.  For this example, we prove a stronger result than what is implied by our general bounds. This example considers the class of linear representations that map from $\datadim$ dimensions to $1$ dimension, $\repclass_{\datadim,1} = \left\{\rep \mid  \rep(\mathbf{\feat}) = \mathbf{b} \cdot \mathbf{\feat}, \mathbf{b} \in \R^{\datadim}\right\}$, and the class of specialized monotone thresholds 
$\perclass_{\textrm{mon}} = \{\per \mid  \per(\map) = \textrm{sign}(\map - w),$\ $ w \in \R \}$.   Then, we sketch the techniques for metalearning with more samples per task. Finally, we discuss extending our techniques to the agnostic case.

\mypar{Metalearning in the realizable case.} A data set $S \in (\mapdom \times \bits)^*$ is \emph{realizable} by a concept class $\perclass$ if there is an $f\in \perclass$ such that $y_i = f(\map_i)$ for every $(\map_i,y_i)\in S$. A distribution $\distint$ over $\mapdom \times \bits$ is realizable by $\perclass$ if there exists  $f\in \perclass$ such that $y=f(\map)$ with probability 1 over $(\map,y)\sim\distint$. In our metalearning model, realizability has two ``layers,'' one for the representation and one for the specialized classifiers. Let $\cP$ be a family of data distributions $\dist$ over $\featdom \times \bits$. We say that $\cP$ is {\em meta-realizable} by $\perclass \circ \repclass$ if there exists a shared representation function $\rep^*$ for which, for all $\dist \in \cP$, the test error of  $\rep^*$, $\rer(\dist,\rep^*,\perclass)$, is zero. In other words, there exists an $\rep^*$ such that for all $\dist \in \cP$ there exists a specialized classifiers $\per \in \perclass$ such that $\er(\dist, \per \circ \rep^*) = 0$. We say a meta distribution is {\em meta-realizable} if its support is meta-realizable.

By these definitions we see that in the realizable case there exists a representation $\rep^*$ in $\repclass$ such that all tasks drawn from distribution $\metadist$ are realizable by $\per \circ \rep^*$ for an $\per$ in $\perclass$. Therefore, during the training process we want to find a representation that allows for perfect classification of all the points of all the seen tasks. Our technical approach to achieve this while getting generalization for unseen tasks depends on the number of samples per task we have.

 The key component of our analysis is a class of binary functions that we call \emph{realizability predicates} (recall Definition~\ref{def:realizability_predicate}). Given a representation $\rep$, a dataset of $n$ points and a set of specialized classifiers $\perclass$, the realizability predicate returns $+1$ if and only if the mapping of the dataset using $\rep$ is realizable by a function in $\perclass$.

\begin{comment}
\begin{defn}[Realizability predicate]\label{def:realizability_predicate}
    Let $\perclass$ be a class of specialized classifiers $\per: \mapdom \to \bits$ and $\samples \in \N$. Given a representation function $\rep: \featdom \to \mapdom$, we define the \emph{$(\samples,\perclass)$-realizability predicate}, denoted by $\realp_{\rep}: (\featdom \times \{\pm1\})^\samples \rightarrow \{\pm1\}$, to be
    \begin{align*}
        \realp_{\rep}\paren{(\feat_1, \lab_1), \ldots, (\feat_{\samples}, \lab_{\samples})} \defeq \ind_{\pm}\left\{\exists \per \in \perclass \text{ s.t. } \forall i \in [\samples], \per(\rep(\feat_i))=y_i\right\}.
    \end{align*}
    The class of realizability predicates, $\realpclass_{\samples, \perclass, \repclass}$, contains all realizability predicates $\realp_{\rep}$ for $\rep \in \repclass$.
\end{defn}
\end{comment}
For instance, in class $\realpclass_{2, \perclass_{\textrm{mon}}, \repclass}$ the realizability predicate $\realp_\rep$ takes two labeled datapoints $(\mathbf{x}_1, y_1)$ and $(\mathbf{x}_2, y_2)$ and returns $+1$ if $\rep$ orders them correctly. More formally, given an $\rep \in \repclass$ 
\[
\realp_\rep(\mathbf{x}_1, y_1, \mathbf{x}_2, y_2) = 
\begin{cases}
-1, & y_1 = +1, y_2 = -1 \textrm{ and } \rep(\mathbf{x}_1) < \rep(\mathbf{x}_2)\\
-1, & y_1 = -1, y_2 = +1 \textrm{ and } \rep(\mathbf{x}_1) > \rep(\mathbf{x}_2)\\
+1, & \textrm{otherwise.}
\end{cases}
\]
We can think of $\realp_\rep$ as a classifier on ``data points'' of the form $(\mathbf{x}_1, y_1, \mathbf{x}_2, y_2)$,  indicating whether $h$ is a good representation or not.  

Since $\metadist$ is meta-realizable by $\perclass\circ \repclass$, there exists a representation $\rep^* \in \repclass$ for which all tasks have a specialized classifier that perfectly classifies all samples. This implies that the corresponding realizability predicate $\realp_{\rep^*} \in \realpclass_{n, \perclass, \repclass}$ outputs $+1$ for any $n$ samples drawn from a distribution $\dist \sim \metadist$. 
\begin{comment}
\begin{figure}[ht] 
    \centering
        \resizebox{\textwidth}{!}{
            \input{img/fig_new_dataset_mt}
        }
        \caption{Constructing a datapoint drawn from distribution $\cD$.}
    
    \label{fig:newDataset}
\end{figure}
\end{comment}

At a high level, we exploit the above fact about $\realp_{\rep^*}$ and class $\realpclass_{n, \perclass, \repclass}$. 
We consider a new data distribution $\cD$ that generates samples of the form $\zeta_j := ((x_1^{(j)}, y_1^{(j)}, \ldots, x_n^{(j)}, y_n^{(j)}),+1)$. The feature part of the sample is a task training set of size $n$ generated by a data distribution $\dist_j \sim \metadist: \{(x_i^{(j)}, y_i^{(j)})\}_{i \in [n]}$. The label part of the sample is always $+1$. See \Cref{fig:newDataset}. Note that the realizability predicate $\realp_{\rep^*}$ also labels the task training sets of size $n$ with $+1$:$$\realp_{\rep^*}((x_1^{(j)},y_1^{(j)}),\ldots, (x_n^{(j)},y_n^{(j)})) = +1.$$
Hence, we can say that the new data distribution $\cD$ is realizable by the hypothesis class $\realpclass_{n, \perclass, \repclass}$. Now, if we have $\users$ samples from $\cD$ for a sufficiently large $\users$ (that is $\users$ tasks $\dist_1, \ldots, \dist_{\users}$ drawn from $\metadist$ and $\samples$ samples from each), by the fundamental theorem of PAC learning, Theorem~\ref{fact:pac-vc}, we can PAC learn the class $\realpclass_{n, \perclass, \repclass}$ with respect to $\cD$. More precisely, we can show that any $\realp_{\hat{\rep}} \in \realpclass_{n, \perclass, \repclass}$ that labels samples $\zeta_j$'s correctly has low mislabeling error under $\cD$. Recall that all the labels according to $\cD$ were always $+1$. Thus, mislabeling a $\zeta_j$ implies that the realizability predicate $\realp_{\hat{\rep}}$ of the training set corresponding to $\zeta_j$ is false. Thus, the representation function we have found, which has a low error probability for mislabeling a $\zeta$, allows the realizability of the dataset corresponding to $\zeta$ with high probability. Hence, $\hat{\rep}$ is an accurate shared representation function which we need to metalearn $\metadist$ by $\perclass \circ \repclass$.

Going back to our example of monotone thresholds over linear representations, suppose that for every task $j \in [\users]$ we see two points $(\mathbf{x}_1^{(j)}, y_1^{(j)})$ and $(\mathbf{x}_2^{(j)}, y_2^{(j)})$ drawn from $P_j$, which itself is drawn from $\metadist$. Since we are in the realizable case, there is a representation $\rep^* \in \repclass_{\datadim,1}$ that always orders points from the same task correctly on the real line. Therefore, for every task $j \in [\users]$ we have $\realp_{\rep^*}(\mathbf{x}_1^{(j)}, y_1^{(j)},\mathbf{x}_2^{(j)}, y_2^{(j)}) = +1$. 
Intuitively we want to learn an $\hat{\rep} \in \repclass_{d,1}$ which has low error: 
$$  \Exp_{P \sim \metadist}\left[\Pr[(\mathbf{x}_1,y_1), (\mathbf{x}_2, y_2) \sim \dist^2]{\realp_{\hat{\rep}}(\mathbf{x}_1, y_1, \mathbf{x}_2, y_2)\neq+1}\right] = \Pr[\zeta \sim \cD]{\realp_{\hat{h}}(\zeta_{\textrm{feat}}) \neq +1}.$$
In other words, we want to find a representation that will allow two points of a new task drawn from $\metadist$ to be classified correctly using a monotone threshold in $\perclass_{\textrm{mon}}$.
By Theorem~\ref{fact:pac-vc}, if we have $\users = O\left(\frac{\VC(\realpclass_{2, \perclass_{\textrm{mon}}, \repclass_{\datadim,1}})\ln(1/\eps)+\ln(1/\delta)}{\varepsilon^2}\right)$ tasks, each with a datapoint $\zeta_j = ((\mathbf{x}_1^{(j)}, y_1^{(j)},\mathbf{x}_2^{(j)}, y_2^{(j)}),+1)$ drawn from $\cD$, then by choosing  $\hat{\rep} \in \repclass_{d,1}$ which minimizes $\frac{1}{\users}\sum_{j \in [\users]}\ind\{\realp_\rep(\zeta_{j,\textrm{feat}})\neq+1\}$ we get that with probability at least $1-\delta$ over the dataset 
$$\Pr[\zeta \sim \cD]{\realp_{\hat{h}}(\zeta_{\textrm{feat}}) \neq +1} = \Exp_{\dist \sim \metadist}\left[\Pr[(\mathbf{x}_1,y_1), (\mathbf{x}_2, y_2) \sim \dist^2]{\realp_{\hat{h}}(\mathbf{x}_1, y_1, \mathbf{x}_2, y_2)\neq+1}\right] < \varepsilon^2.$$ 
In Lemma~\ref{lem:vc_mt} we show that the VC dimension of $\realpclass_{2, \perclass_{\textrm{mon}}, \repclass_{\datadim,1}}$ is at most $\datadim$. Thus, we can learn with $\users = O\left(\frac{\datadim\ln(1/\eps)+\ln(1/\delta)}{\varepsilon^2}\right)$ tasks.

\begin{comment}
\begin{fact}
\label{fct:vc-hf}
\knote{citation?}
\gavinnote{Radon's theorem, mentioned but not cited below.}
Let $\perclass = \{\per \mid \per(\mathbf{x}) = \sign (\mathbf{a}\cdot \mathbf{x}), \mathbf{a} \in \R^d\}$ be the class of $d$-dimensional halfspaces that pass through zero. Then $\VC(\perclass) = d$.
\end{fact}
\end{comment}

\begin{lem}
\label{lem:vc_mt}
   We have $\VC(\realpclass_{2, \perclass_{\textrm{mon}}, \repclass_{\datadim,1}}) \leq d.$
   % The VC dimension of $\realpclass_{2, \perclass_{\textrm{mon}}, \repclass_{\datadim,1}}$ is 
   % \[
   % \VC(\realpclass_{2, \perclass_{\textrm{mon}}, \repclass_{\datadim,1}}) \leq 2d.
   % \]
\end{lem}
\begin{proof}
    Fix a set of inputs to $\realp_\rep$:
    % \begin{align*}
        $((\mathbf{x}_1^{(1)}, y_1^{(1)}, \mathbf{x}_2^{(1)}, y_2^{(1)}), \ldots, (\mathbf{x}_1^{(\samples)}, y_1^{(\samples)}, \mathbf{x}_2^{(\samples)}, y_2^{(\samples)}))$.
    % \end{align*}
    % We will show that this set cannot be shattered by $\realpclass_{2, \perclass_{\textrm{mon}}, \repclass_{\datadim,1}}$.
    Assume for all $i$ we have $y_1^{(i)}\neq y_2^{(i)}$, since otherwise we have $\realp_\rep(\mathbf{x}_1^{(i)}, y_1^{(i)}, \mathbf{x}_2^{(i)}, y_2^{(i)})=+1$ for all $\rep$, which means the overall set cannot be shattered.

    Because the thresholds are monotone, if $y_1^{(i)} > y_2^{(i)}$ then $\realp_\rep(\mathbf{x}_1^{(i)}, y_1^{(i)}, \mathbf{x}_2^{(i)}, y_2^{(i)})=+1$ when $\rep$ places $\rep(\mathbf{x}_1^{(i)})$ above $\rep(\mathbf{x}_2^{(i)})$.
    Equivalently, associating $\rep$ with $\mathbf{b}\in \mathbb{R}^d$, this holds when
    $\sign(\mathbf{b}\cdot(\mathbf{x}_1^{(i)} - \mathbf{x}_2^{(i)}))=+1$.
    The opposite holds for $y_1^{(i)} < y_2^{(i)}$.

    We can achieve a labeling $\mathbf{s}\in\{\pm 1\}^{\samples}$ if there exists $\mathbf{b}$ such that, for all $i$, 
    \begin{align*}
        \sign(\mathbf{b}\cdot(\mathbf{x}_1^{(i)} - \mathbf{x}_2^{(i)})) = \mathbf{s}_i \cdot \sign(y_1^{(i)} - y_2^{(i)}).
    \end{align*}
    From this, we see that the number of achievable labelings is determined by the number of possible signs of $\mathbf{b}\cdot(\mathbf{x}_1^{(i)} - \mathbf{x}_2^{(i)})$.
    This is at most $2^d$ by Theorem~\ref{thm:vc-hs}, which bounds the capacity of halfspaces passing through the origin.
    Thus, $\realpclass_{2, \perclass_{\textrm{mon}}, \repclass_{\datadim,1}}$ cannot shatter a set of size $d+1$.
\end{proof}

So far, we have described how to find a representation $\hat{\rep}$ that, evaluated on data sets of size $n$ drawn from new tasks, is likely to produce something which $\perclass$ can perfectly classify, i.e., the following expressions are small:
% that has low probability of producing a representation of a dataset of size $n$ drawn from a new task that cannot be perfectly classified by $\perclass$, i.e.
\begin{align*}
    &\Exp_{\dist \sim \metadist}\left[ \Pr[(x_1, y_1),\ldots, (x_n,y_n) \sim \dist^n]{ \realp_{\hat{\rep}}((x_1, y_1),\ldots, (x_n,y_n)) \neq +1}\right]=\\
    &\Exp_{\dist \sim \metadist}\left[ \Pr[(x_1, y_1),\ldots, (x_n,y_n) \sim \dist^n]{(\hat{\rep}(x_1), y_1),\ldots, (\hat{\rep}(x_n),y_n) \text{ is not realizable by }\perclass}\right].
\end{align*}
However, in metalearning we want to bound the following error 
$$\rer(\metadist, \hat{\rep}, \perclass) = \Exp_{\dist \sim \metadist}\left[\min_{\per \in \perclass} \Pr[(x,y) \sim \dist]{\per(\hat{\rep}(x))\neq y}\right].$$
The next step is to derive a bound on this error. 
% $\rer(\metadist, \hat{\rep}, \perclass)$ for representation $\hat{\rep}$. 
We consider two cases based on the number of samples per task we have.

\textbf{Metalearning with $\NRC(\perclass)$ samples per task.} 
Often, for a given concept class $\perclass$ we can see that every dataset that is not realizable by $\perclass$ has a subset of size at most $\wit$ that is still not realizable. We call this subset a non-realizability certificate. The smallest $\wit$ for which this condition holds is the non-realizability-certificate complexity, which we denote by $\NRC(\perclass)$. 

\begin{comment}
\begin{defn}[Non-realizability witness]\label{def:nonrealizability}
    Let $\perclass$ be a class of functions $\per: \mapdom \to \bits$, $\wit \in \N$ and $S\in (\mapdom\times\bits)^*$ be a dataset. 
    A \emph{non-realizability witness of size $m$ for $S$ with respect to $\perclass$} is a subset of $S$ of size at most $\wit$ that cannot be realized by $\perclass$. 
\end{defn}
\kb{Do we want to define \nrcc here?}
\jubig{Yeah I think we should to match the definition in the intro.}
\end{comment}
For example, the non-realizabity-certificate complexity of the class of monotone thresholds is $2$. Every dataset with more than $2$ points that is not realizable by $\perclass_{\textrm{mon}}$ has a non-realizability witness of size $2$. To see this, suppose all subsets of two points are realizable by $\perclass_{\textrm{mon}}$. This means that any point with a positive label is greater than a point with a negative label. In this case, there exists a monotone threshold that labels these points correctly by placing the threshold between two consecutive points in the real line with opposite labels, which leads to a contradiction. Just seeing one point of a dataset does not suffice to show non-realizability cause every one-sample dataset is realizable.

For a fixed data distribution $\distint$ over $\mapdom \times \bits$ that is realizable by $\perclass$, we define the \textit{probability of nonrealizability} of a dataset from $\distint$ by functions in 
$\perclass$:
 \begin{align*}
    p_{\text{nr}}(\distint, \perclass, \wit)&\defeq \Pr[(\map_1,y_1),\ldots,(\map_\wit,y_\wit) \sim \distint^\wit]{(\map_1,y_1),\ldots,(\map_\wit,y_\wit)\textrm{ is not realizable by }\perclass},
 \end{align*}
and the \textit{population error} of $\perclass$ on $\distint$: 
 \begin{align*}
        \er(\distint, \perclass)&\defeq \min_{\per \in \perclass}\left\{ \Pr[(\textrm{\map},y)\sim \distint]{\per(\map)\neq y} \right \}.
 \end{align*}
  Notice that here distribution $\distint$ is over labeled points in the intermediate space, the codomain of the representation function. Given $\NRC(\perclass)$ points from $\distint$, we derive a bound of the form $\phi(\er(\distint, \perclass)) \leq  p_{\text{nr}}(\distint, \perclass, \NRC(\perclass))$, for a strictly increasing and convex $\phi$. 

 For the class of monotone thresholds $\perclass_{\textrm{mon}}$ with $2$ samples drawn from $\distint$, we show in Lemma~\ref{lem:mon-bnd} that $(\er(\distint, \perclass_{\textrm{mon}}))^2 \leq p_{\text{nr}}(\distint, \perclass_{\textrm{mon}}, 2)$. 
 The proof uses Fact~\ref{fact:auc} and a geometric argument linking $\er(\distint, \perclass_{\textrm{mon}})$ to $ p_{\text{nr}}(\distint, \perclass_{\textrm{mon}}, 2)$.

For a fixed distribution $\dist$, we analyze the ROC curve of classifiers with threshold $w$: $\sign(z-w)$.
The ROC curve plots the true positive rate (TPR) against the false positive rate (FPR), which are defined as
\begin{align*}
    \textrm{TPR}(w) &:= \Pr[(\map,y)\sim \distint]{\map\geq w\mid y=+1} \\
    \textrm{FPR}(w) &:= \Pr[(\map,y)\sim \distint]{\map\geq w\mid y=-1}.
\end{align*}
We need the following standard fact about ROC curves.

\begin{fact}
\label{fact:auc}
    For any distribution $\distint$,
    % Fix a distribution $\distint$ that is realizable by $\perclass_{\textrm{mon}}$. 
    % Let $\textrm{TPR}(w) := \Pr[(\map,y)\sim \distint]{\map> w\mid y=+1}$ be the true positive rate of classifier $\textrm{sign}(\map-w)$ and $\textrm{FPR}(w):= \Pr[(\map,y)\sim \distint]{\map\geq w\mid y=-1}$ be the false positive rate of the same classifier. The plot of $\textrm{TPR}(w)$ against $\textrm{FPR}(w)$ for all values of $w \in \R$ is called the ROC curve. 
    the area under the ROC curve is equal to 
    $$\Pr{\map_1>\map_2 \mid  y_1=+1, y_2=-1}.$$
    % $\Pr[(\map_1,y_1),(\map_2,y_2)\sim \distint^2]{\map_1>\map_2 \mid  y_1=+1, y_2=-1}$.
\end{fact}

\begin{figure}
    \centering
    \includegraphics[width = 0.5\textwidth]{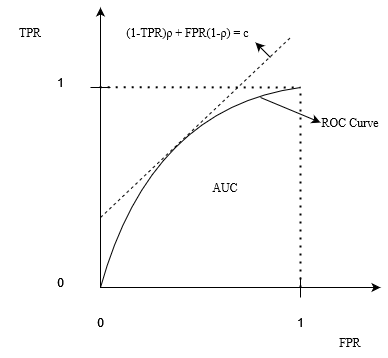}
    \caption{Example of an ROC curve and the line $(1-\textrm{TPR})\rho+\textrm{FPR}(1-\rho) = c$, which corresponds to points with error $c$. The line and curve intersect when $c=\er(\distint, \perclass_{\textrm{mon}})$.}
    \label{fig:auc}
\end{figure}

\begin{lem}
\label{lem:mon-bnd}
Fix a distribution $\distint$ over $ \mapdom \times \{\pm 1\}$.
We have
$ (\er(\distint, \perclass_{\textrm{mon}}))^2 \leq p_{\text{nr}}(\distint, \perclass_{\textrm{mon}}, 2)$.
\end{lem}
\begin{proof}
Let $\rho= \Pr[(\map,y)\sim \distint]{y=1}$ and $S$ be a dataset $\{(\map_1, y_1), (\map_2, y_2)\}$, where $(\map_1, y_1)$ and $(\map_2, y_2)$ were drawn independently from distribution $B$. We start by showing that
\begin{equation}
\label{eq:AUC}
\er(\distint, \perclass_{\textrm{mon}})^2 \leq 2\rho(1-\rho)(1-\Pr[S\sim \distint^2]{\map_1>\map_2 \mid  y_1=+1, y_2=-1)}).
\end{equation}
% Let $\textrm{TPR}(w) := \Pr[(\map,y)\sim \distint]{\map> w\mid y=+1}$ be the true positive rate of threshold classifier $\textrm{sign}(\map-w)$ and let $\textrm{FPR}(w):= \Pr[(\map,y)\sim \distint]{\map\geq w\mid y=-1}$ be the false positive rate. 
By the law of total probability, we have
\begin{align*}
\Pr[(\map,y)\sim \distint]{\textrm{sign}(\map-w)\neq y} 
&= \Pr[(\map,y)\sim \distint]{\map< w\mid y=+1}\rho + \Pr[(\map,y)\sim \distint]{\map\geq w\mid y=-1}(1-\rho)\\
&= (1-\textrm{TPR}(w))\rho+\textrm{FPR}(w)(1-\rho).
\end{align*}
% The plot of $\textrm{TPR}(w)$ against $\textrm{FPR}(w)$ for all values of $w \in \R$ is called the ROC curve. 
% In the same coordinate system, the points corresponding to error $c$ are described by the line $(1-\textrm{TPR})\rho+\textrm{FPR}(1-\rho) = c$.
Fact~\ref{fact:auc} says that the area under the curve equals $\Pr[S\sim \distint^2]{\map_1>\map_2 \mid  y_1=+1, y_2=-1}$.
Therefore, we can geometrically upper bound this quantity by the area of the space which is (i) under the isoerror line for error $c=\er(\distint, \perclass_{\mathrm{mon}})$ and (ii) under the line $\mathrm{TPR}=1$.
See \Cref{fig:auc} for an illustration.
Note that $\er(\distint, \perclass_{\mathrm{mon}})$ is exactly $\min_{w \in \R}\left\{ \Pr[(\map,y)\sim \distint]{\textrm{sign}(\map-w) \neq y} \right \}$.
This approach gives us the following bound:
\begin{align*}
    \Pr[S\sim \distint^2]{\map_1>\map_2 \mid  y_1=+1, y_2=-1} \leq 1-\frac{\left(
    % \min_{w \in \R}\left\{ \Pr[(\map,y)\sim \distint]{\textrm{sign}(\map-w) \neq y} \right \}
    \er(\distint, \perclass_{\mathrm{mon}})
    \right)^2}{2\rho(1-\rho)}.
\end{align*}
% intersects with the ROC curve have error $\Pr[(\map,y)\sim \distint]{\textrm{sign}(\map-w)\neq y}=c$. By \Cref{fact:auc} the area under the ROC curve is equal to $\Pr[S\sim \distint^2]{\map_1>\map_2 \mid  y_1=+1, y_2=-1}$. 
% Therefore, we can geometrically upper bound the AUC by the area under the isoerror line for error $\min_{w \in \R}\left\{ \Pr[(\map,y)\sim \distint]{\textrm{sign}(\map-w)\neq y} \right \}$ 
% and line $\textrm{TPR} = 1$. \ktext{See \Cref{fig:auc}.}\knote{I added this figure. I can make changes if you think it's not clear.} This approach gives us the following bound
% \begin{align*}
%     \Pr[S\sim \distint^2]{\map_1>\map_2 \mid  y_1=+1, y_2=-1} \leq 1-\frac{\left(\min_{w \in \R}\left\{ \Pr[(\map,y)\sim \distint]{\textrm{sign}(\map-w) \neq y} \right \}\right)^2}{2\rho(1-\rho)}.
% \end{align*}
Rearranging yields \Cref{eq:AUC}.
% We see that 
% \[
% \left(\min_{w \in \R}\left\{ \Pr[(\map,y)\sim \distint]{\textrm{sign}(\map-w)\neq y} \right \}\right)^2 \leq 2\rho (1-\rho)(1- \Pr[S\sim \distint^2]{\map_1>\map_2 \mid  y_1=+1, y_2=-1}).
% \]
Next, we show that the right-hand side of \Cref{eq:AUC} is exactly $p_{\text{nr}}(\distint, \perclass_{\textrm{mon}}, 2)$, the probability of non-realizability.

Let $E$ be the event that $(\map_1, y_1), (\map_2, y_2) \textrm{ is not realizable by }\perclass_{\textrm{mon}}.$
By the law of total probability, we have that 
\begin{align*}
    \Pr[S\sim \distint^2]{E}
    & =\sum_{i\in \bits} \sum_{j \in \bits}\Pr[S\sim \distint^2]{E\mid y_1=i, y_2=j } \Pr[S\sim \distint^2]{y_1 = i, y_2=j}  \\
    &=2\rho(1-\rho)\Pr[S\sim \distint^2]{\map_1\leq \map_2 \mid y_1=+1,y_2=-1} \\
    &=2\rho(1-\rho)(1-\Pr[S\sim \distint^2]{\map_1> \map_2 \mid y_1=+1,y_2=-1}).
\end{align*}
Therefore, $\er(\distint, \perclass_{\textrm{mon}})^2 \leq  p_{\text{nr}}(\distint, \perclass_{\textrm{mon}}, 2)$. 
% \begin{comment}Since both $\er(\distint,\perclass_{\textrm{mon}})$ and $ p_{\text{nr}}(\distint, \perclass_{\textrm{mon}}, \wit)$ are non-negative, $\er(\distint, \perclass_{\textrm{mon}}) \leq \sqrt{ p_{\text{nr}}(\distint, \perclass_{\textrm{mon}}, \wit)}$\end{comment}.
\end{proof}

In our example, since both $\er(\distint,\perclass_{\textrm{mon}})$ and $ p_{\text{nr}}(\distint, \perclass_{\textrm{mon}}, 2)$ are non-negative, $\er(\distint, \perclass_{\textrm{mon}}) \leq$ $ \sqrt{ p_{\text{nr}}(\distint, \perclass_{\textrm{mon}}, 2)}$. Therefore, we conclude that with probability at least $1-\delta$ over the data (that we get by seeing $\users = O\left(\frac{\datadim \ln(1/\eps) + \ln(1/\delta)}{\eps^2}\right)$ tasks with data distributions $\dist_j$ drawn from $\metadist$ and $2$ samples from $\dist_j$ for every task $j \in [\users]$)
\begin{align*}
&\Exp_{P \sim Q}\left[\min_{w \in \R}\left\{ \Pr[(\textrm{x},y)\sim P]{\textrm{sign}(\hat{h}(\mathbf{x})-w)\neq y\}} \right \}\right] \leq\\&\sqrt{\Exp_{\dist \sim \metadist}\left[\Pr[(\mathbf{x}_1, y_1),(\mathbf{x}_2,y_2) \sim \dist^2]{\realp_{\hat{\rep}}(\mathbf{x}_1, y_1,\mathbf{x}_2,y_2)} \neq +1\right]} 
\leq  \varepsilon.
\end{align*}
This proves Theorem~\ref{thm:met-mon}. 

\begin{thm}
\label{thm:met-mon}
    \thmmetmon
\end{thm}

\paragraph{From Monotone Thresholds to Arbitrary Function Classes.}
Although the argument above relies on the particular structure of monotone function classes in several ways, the next few sections will show how the results can be generalized to arbitrary representations and specialized classifiers.  First, we study the realizable case with very few samples per task (\Cref{sec:metalearn-realizable}), and then we study the agnostic case (\Cref{sec:metalearning_agnostic}) with more samples per task. The results for the agnostic case also apply to the realizable case, yielding incomparable statements.  After establishing these general bounds in terms of properties of the representation and specialized classifiers, we use them to prove specific sample complexity bounds for linear classes in \Cref{sec:lin}.

\fi
\ifnum \coltshort = 0
\subsection{Sample and Task Complexity Bounds for the Realizable Case}
\else
\subsection{Sample and Task Complexity Bounds}
\fi
\label{sec:metalearn-realizable}
\ifnum \coltshort = 0
When the metadistribution is meta-realizable, \Cref{thm:met-samples} and \Cref{thm:met-real-samples} bound the number of tasks and samples per task we need to metalearn. Their proofs follow the structure described in \Cref{sec:real-tech}, using both the VC dimension of the realizability predicate class to bound the number of tasks needed to metalearn. When we have roughly $\VC(\perclass)/\eps$ samples per task, we bound the number of tasks via \Cref{thm:met-real-samples}. 
 % , we can use \Cref{thm:met-samples}, as long as our specialized classifiers have a non-realizability certificate of that size.
For a smaller number of samples per task, \Cref{thm:met-samples} shows that $\NRC(\perclass)$ examples per task suffice to metalearn $(\repclass, \perclass)$, where $\NRC(\perclass)$ is the \nrcc\ for $\perclass$. 
\else 
In this section we give generic sample complexity bounds in terms of two quantities, the VC dimension of the realizability predicate (see Definition~\ref{def:realizability_predicate}) and the pseudodimension of the empirical error predicate (see Definition~\ref{def:empirical_error_function}).  For the realizable case, Theorem~\ref{thm:met-samples} covers the case where we have very few samples per task, which we consider to be the most interesting case, and Theorem~\ref{thm:met-real-samples} covers the case where we have more samples per task.  The agnostic case is covered in Theorem~\ref{thm:agn-met-samples}.
%In the realizable case, Theorems~\ref{thm:met-samples} and \ref{thm:met-real-samples} bound the number of tasks and samples per task we need to metalearn. Both of their proofs use the VC dimension of the realizability predicate class (see Definition~\ref{def:realizability_predicate}) to bound the number of tasks needed to metalearn. When we have roughly $\VC(\perclass)/\eps$ samples per task, we bound the number of tasks via Theorem~\ref{thm:met-real-samples}. 
 % , we can use \Cref{thm:met-samples}, as long as our specialized classifiers have a non-realizability certificate of that size.
%For a smaller number of samples per task, Theorem~\ref{thm:met-samples} shows that $\NRC(\perclass)$ examples per task suffice to metalearn $(\repclass, \perclass)$, where $\NRC(\perclass)$ is the \nrcc\ for $\perclass$. 
\fi
\newcommand{\thmMetSamples}{Let $\perclass$ be a class of specialized classifiers $\per: \mapdom \to \bits$ with non-realizability certificate complexity $\NRC(\perclass) = \wit$
%for which we have a positive integer $\wit$ such that every dataset $S$ of size at least $\wit$ with datapoints in $\mapdom \times \bits$ that is not realizable by $\perclass$ has a non-realizability witness of size $\wit$. Let
and $\repclass$ be a class of representation functions $\rep: \featdom \to \mapdom$.
Then, for every $\varepsilon$ and $\delta$ in $(0,1)$,  we can $(\varepsilon, \delta)$-properly metalearn $(\repclass, \perclass)$ in the realizable case for $\users$ tasks and $\samples$ samples per task when
$$\users  
    =
    \bparen{\VC(\realpclass_{\wit,\perclass,\repclass})+\ln(1/\delta)} 
    \cdot
    \paren{ \frac{O(\max(\VC(\perclass),\wit)\ln(1/\varepsilon))} {\wit \varepsilon }}^\wit
 \quad \text{ and } \samples = \wit   \, . 
$$
}
\begin{thm}[Complete statement of Theorem~\ref{thm:met-samples-intro}]
\label{thm:met-samples}
\thmMetSamples
\end{thm}
\if \coltshort = 0
We first prove \Cref{lem:real-to-error} and \Cref{lem:num-tasks}, which we use in the proof of \Cref{thm:met-samples}.
For a fixed distribution $\dist$ and class $\perclass$, \Cref{lem:real-to-error} allows us to relate the probability of drawing a nonrealizable dataset to the error of the class. 
\else
We use Lemmas~\ref{lem:real-to-error} and \ref{lem:num-tasks} to prove Theorem~\ref{thm:met-samples}.
For a fixed distribution $\dist$ and class $\perclass$, Lemma~\ref{lem:real-to-error} allows us to relate the probability of drawing a nonrealizable dataset, $p_{nr}(\dist, \perclass, \NRC(\perclass))$, to the test error of class $\perclass$. 
\fi
\newcommand{\lemRealtoError}{ Let $\perclass$ be the class of specialized classifiers $\per: \mapdom \to \bits$ with $\NRC(\perclass) = \wit$. 
    %Assume there exists a positive integer $\wit$ such that every data set $S$ of size at least $\wit$ \adaminline{Use $\NRC(\perclass)$ notation} that is not realizable by $\perclass$ has a non-realizability witness of size $\wit$.
    Fix an arbitrary distribution $\distint$ over $\mapdom\times\bits$.
    If $\er(\distint, \perclass) > 0$, then $$ p_{nr}(\distint, \perclass, \wit)  \ge  \frac 1 2
    \left(\left[\frac{\wit \cdot \er(\distint, \perclass)}{16e\cdot \vc \ln{(16/\er(\distint, \perclass))}}\right]^{\wit}\right),$$ where $v = \max(\VC(\perclass),\wit)$.}
\begin{lem}
     \label{lem:real-to-error}
   \lemRealtoError
\end{lem}
\ifnum \coltshort = 0
\begin{proof}
Let $g(\eps)  = \frac{16}{\eps}\VC(\perclass)\ln(16/\eps)$. 
% Let $g(\eps)  = \frac{8}{\eps}[\VC(\perclass)\ln(16/\eps)+\ln(4)]$. 
Function $g$ is continuous and strictly decreasing in $\eps$ for $\eps \in (0,1]$. %Therefore, its inverse $g^{-1}$ exists.
The proof analyzes two cases

In case one, if %$\er(\distint,\perclass) < g^{-1}(\wit)$, then 
$\wit > g(\er(\distint,\perclass))$, then by Fact \ref{fact: vc-gen} the probability that a dataset of $\wit$ points drawn from $\distint$ is not realizable by $\perclass$ is  
\begin{align*}
    p_{nr}(\distint, \perclass, \wit) 
    &= \Pr[S_\wit\sim \distint^{\wit}]{S_\wit\text{ is not realizable by }\perclass} \\
    &= \Pr[S_\wit\sim \distint^{\wit}]{\min_{\per \in \perclass}\er(S_ \wit,\per)>0} \geq \frac{1}{2},
\end{align*}
because $\er(\distint,\perclass) >0$ and $\wit\ge \frac{8}{\eps}\left[\mathrm{VC}(\mathcal{F})\ln (16/\eps)+\ln (2/\delta)\right]$ for $\eps=\er(\distint,\perclass)$ and $\delta=\frac 1 2$.
This is stronger than the claimed lower bound.
To see this, recall $\vc \ge \wit$ and observe
\begin{align*}
    \frac 1 2 \left(\frac{\wit \cdot \er(\distint, \perclass)}{16e\cdot \vc \ln{(16/\er(\distint, \perclass))}}\right)^{\wit}
    % &\le  \left[\frac{\wit \cdot \er(\distint, \perclass)}{16e\cdot \textcolor{black}{\wit} \ln{(16/\er(\distint, \perclass))}}\right]^{\wit} \\
    &\le \frac 1 2
    \left(\frac{\er(\distint, \perclass)}{16e\cdot \ln{(16/\er(\distint, \perclass))}}\right)^{\wit} \\
    &\le \frac 1 2 \left(\frac{1}{16e\cdot \ln{(16)}}\right)^{\wit},
    % &\le \frac 1 2.
\end{align*}
which is less than $\frac{1}{2}$.

%If $\er(\distint,\perclass) \geq g^{-1}(\wit)$
In case two, suppose $\wit \leq g(\er(\distint,\perclass))$.
Drawing $\wit$  i.i.d.\ samples from $\distint$ is equivalent to drawing a larger dataset $S_n =\{(\map_i,\lab_i)\}_{i \in [n]}$ from $\distint$ of some size $n\geq m$, set later in the proof, and picking a uniformly random subset $S_\wit\subseteq S_n$ of size $\wit$.
% $\wit$ points uniformly at random without replacement from $S_n$. 
More formally, 
\begin{align*}
     p_{nr}(\distint, \perclass, \wit) &=
    \Pr[S_\wit \sim \distint^{\wit}]{S_\wit \text{ is not realizable}} \\
    % &= \Pr[\substack{(\map_1,\lab_1),\ldots ,(\map_{n},\lab_{n})\sim \distint^{n}\\(i_1, \ldots, i_{\wit}) \sim \text{Unif}{\left(\binom{n}{m}\right)}}]{(\map_{i_1},\lab_{i_1}), \dots, (\map_{i_\wit},\lab_{i_\wit}) \text{ is not realizable} }.
    &= \Pr[\substack{S_n\sim \distint^n \\ S_\wit \sim \binom{S_n}{\wit}}]{S_\wit \text{ is not realizable} }.
\end{align*}
Furthermore, we notice that  if $S_n$ is labeled correctly by some $\per \in \perclass$, then $\per$ labels $S_\wit$ correctly, too.
Hence, 
\begin{align}
    p_{nr}(\distint, \perclass, \wit) &= \Pr{S_\wit \text{ is not realizable} } \nonumber \\
        &= \Pr{S_\wit \text{ is not realizable}  \mid S_n \text{ is not realizable} } \cdot \Pr{S_n \text{ is not realizable}}. \label{eq:not_realizable_breakdown}
\end{align}
We know that if $S_n$ is not realizable then there exists a non-realizability certificate of size $m$.
Since there are $\binom{n}{\wit}$ subsets, $S_\wit$ is exactly this certificate with probability at least $1/\binom{n}{\wit}$.
 % We notice that  if $S_n$ is labeled correctly by an $\per \in \perclass$, then the same $\per$ labels $(\map_{i_1},\lab_{i_1}), \dots, (\map_{i_\wit},\lab_{i_\wit})$ correctly, too. Hence%, by the law of total probability
% \begin{align*}
%     &\Pr[\substack{(\map_{1},\lab_{1}),\ldots ,(\map_{i_n},\lab_{n})\sim \distint^{n}\\(i_1, \ldots, i_{\wit}) \sim \text{Unif}{\left(\binom{n}{m}\right)}}]{(\map_{i_1},\lab_{i_1}), \dots, (\map_{i_\wit},\lab_{i_\wit})\text{ is not realizable}} =\\
%     &\Pr[\substack{(\map_{1},\lab_{1}),\ldots ,(\map_{n},\lab_{n})\sim \distint^{n}\\(i_1, \ldots, i_{\wit}) \sim \text{Unif}{\left(\binom{n}{m}\right)}}]{(\map_{i_1},\lab_{i_1}), \dots, (\map_{i_\wit},\lab_{i_\wit})\text{ is not realizable} \Big| S_n \text{ is not realizable}} \Pr[S_n \sim \distint^n]{S_n \text{ is not realizable}}.
% \end{align*}

We now provide a lower bound on the probability that $S_n$ is not realizable.
We set 
$$n := \frac{16}{\er(\distint,\perclass)}\VC(\perclass)\ln(16/\er(\distint,\perclass))$$,
% $n := \frac{8}{\er(\distint,\perclass)}[\VC(\perclass)\ln(16/\er(\distint,\perclass))+\ln(4)]$,
which satisfies $n\geq m$ by hypothesis. Then, by Fact \ref{fact: vc-gen} the probability that dataset $S_n$ is not realizable by $\perclass$ is  
\begin{align*}
    \Pr[S_n\sim \distint^{n}]{S_n\text{ is not realizable by }\perclass} = 
    \Pr[S_n\sim \distint^{n}]{\min_{\per \in \perclass}\er(S_n,\per)>0} \geq \frac{1}{2},
\end{align*}
again because $\er(\distint,\perclass) >0$ and $n$ is sufficiently large.

Thus, continuing from \Cref{eq:not_realizable_breakdown} and using a bound on the binomial coefficient, we see that $p_{nr}(\distint, \perclass, \wit) \ge \frac{1}{2\binom{n}{\wit}} \geq \frac{1 }{2(\frac{en}{\wit})^{\wit}}$.
%
% We know that if $S_n$ is not realizable then there exists a non-realizability certificate of size $m$. Thus,
% \begin{align*}
%     &\Pr[\substack{S_n\sim \distint^{n}\\(i_1, \ldots, i_{\wit}) \sim \text{Unif}{\left(\binom{n}{m}\right)}}]{(\map_{i_1},\lab_{i_1}), \dots, (\map_{i_\wit},\lab_{i_\wit})\text{ is not realizable by }\perclass \Big| S_n\text{ is not realizable by }\perclass} \geq\\
%     &\Pr[\substack{S_n\sim \distint^{n}\\(i_1, \ldots, i_{\wit}) \sim \text{Unif}{\left(\binom{n}{m}\right)}}]{\{(\map_{i_1},\lab_{i_1}), \dots, (\map_{i_\wit},\lab_{i_\wit}) \text{ is a n-r certificate } \Big| S_n\text{ is not realizable by }\perclass}\geq \frac{1}{\binom{n}{\wit}}.
% \end{align*}
% Combining the steps above, we see that
% \begin{align*}
%     &p_{\text{nr}}(\distint,\perclass,\wit)  =\\&\Pr[\substack{(\map_{1},\lab_{1}),\ldots ,(\map_{n},\lab_{n})\sim \distint^{n}\\(i_1, \ldots, i_{\wit}) \sim \text{Unif}{\left(\binom{n}{m}\right)}}]{(\map_{i_1},\lab_{i_1}), \dots, (\map_{i_\wit},\lab_{i_\wit})\text{ is not realizable}
%     \Big|
%     S_n \text{ is not realizable}} \Pr[S_n \sim \distint^n]{S_n \text{ is not realizable}}\geq\\
%     & \frac{1}{2\binom{n}{\wit}} \geq \frac{1 }{2(\frac{en}{\wit})^{\wit}}.
% \end{align*}
Plugging in $n=\frac{16}{\er(\distint,\perclass)}\VC(\perclass)\ln(16/\er(\distint,\perclass))$, we get that
\begin{align}
    p_{\text{nr}}(\distint,\perclass,\wit) \geq \frac{1}{2}\left[\frac{m \cdot \er(\distint,\perclass)}{16e\cdot \VC(\perclass)\ln{(16/\er(\distint,\perclass))}}\right]^m.
\end{align}
As $\VC(\perclass)\le \vc$, this is stronger than the claim in the lemma.
This concludes the proof.
\end{proof}
\fi
\if \coltshort = 0
The prior lemma established a quantitative relationship between error and the probability.
The following lemma shows how we can exploit this relationship for provable metalearning.
\else 
\begin{figure}[ht] 
    \centering
        \resizebox{\textwidth}{!}{
            \tikzset{every picture/.style={line width=0.75pt}} %set default line width to 0.75pt        

\begin{tikzpicture}[x=0.75pt,y=0.75pt,yscale=-1,xscale=1]
%uncomment if require: \path (0,236); %set diagram left start at 0, and has height of 236

%Rounded Rect [id:dp4686430581840396] 
\draw  [line width=1.5]  (16,97.4) .. controls (16,87.79) and (23.79,80) .. (33.4,80) -- (87.6,80) .. controls (97.21,80) and (105,87.79) .. (105,97.4) -- (105,149.6) .. controls (105,159.21) and (97.21,167) .. (87.6,167) -- (33.4,167) .. controls (23.79,167) and (16,159.21) .. (16,149.6) -- cycle ;

%Straight Lines [id:da3750788988193219] 
\draw [line width=1.5]    (112,122) -- (166,122) ;
\draw [shift={(169,122)}, rotate = 180] [color={rgb, 255:red, 0; green, 0; blue, 0 }  ][line width=1.5]    (14.21,-4.28) .. controls (9.04,-1.82) and (4.3,-0.39) .. (0,0) .. controls (4.3,0.39) and (9.04,1.82) .. (14.21,4.28)   ;
%Rounded Rect [id:dp4435815186381684] 
\draw  [line width=1.5]  (175,97.4) .. controls (175,87.79) and (182.79,80) .. (192.4,80) -- (246.6,80) .. controls (256.21,80) and (264,87.79) .. (264,97.4) -- (264,149.6) .. controls (264,159.21) and (256.21,167) .. (246.6,167) -- (192.4,167) .. controls (182.79,167) and (175,159.21) .. (175,149.6) -- cycle ;

%Rounded Rect [id:dp30983238640831057] 
\draw  [line width=1.5]  (341,61) .. controls (341,50.51) and (349.51,42) .. (360,42) -- (417,42) .. controls (427.49,42) and (436,50.51) .. (436,61) -- (436,181) .. controls (436,191.49) and (427.49,200) .. (417,200) -- (360,200) .. controls (349.51,200) and (341,191.49) .. (341,181) -- cycle ;

%Straight Lines [id:da5487016515746438] 
\draw [line width=1.5]    (272,121) -- (326,121) ;
\draw [shift={(329,121)}, rotate = 180] [color={rgb, 255:red, 0; green, 0; blue, 0 }  ][line width=1.5]    (14.21,-4.28) .. controls (9.04,-1.82) and (4.3,-0.39) .. (0,0) .. controls (4.3,0.39) and (9.04,1.82) .. (14.21,4.28)   ;
%Straight Lines [id:da06933398494211018] 
\draw [line width=1.5]    (446,119) -- (556,120) ;

%Rounded Rect [id:dp33255128078978835] 
\draw  [color={rgb, 255:red, 65; green, 117; blue, 5 }  ,draw opacity=1 ][line width=1.5]  (8,62.4) .. controls (8,40.64) and (25.64,23) .. (47.4,23) -- (516.6,23) .. controls (538.36,23) and (556,40.64) .. (556,62.4) -- (556,180.6) .. controls (556,202.36) and (538.36,220) .. (516.6,220) -- (47.4,220) .. controls (25.64,220) and (8,202.36) .. (8,180.6) -- cycle ;
%Shape: Rectangle [id:dp9734065596391649] 
\draw  [color={rgb, 255:red, 65; green, 117; blue, 5 }  ,draw opacity=1 ][fill={rgb, 255:red, 255; green, 255; blue, 255 }  ,fill opacity=1 ][line width=1.5]  (45,4) -- (176,4) -- (176,44) -- (45,44) -- cycle ;

%Straight Lines [id:da5132714101876962] 
\draw [color={rgb, 255:red, 65; green, 117; blue, 5 }  ,draw opacity=1 ][line width=1.5]    (557,120) -- (581,120) ;
\draw [shift={(584,120)}, rotate = 180] [color={rgb, 255:red, 65; green, 117; blue, 5 }  ,draw opacity=1 ][line width=1.5]    (14.21,-4.28) .. controls (9.04,-1.82) and (4.3,-0.39) .. (0,0) .. controls (4.3,0.39) and (9.04,1.82) .. (14.21,4.28)   ;

% Text Node
\draw (27.23,103.5) node [anchor=north west][inner sep=0.75pt]   [align=left] {\begin{minipage}[lt]{46.94pt}\setlength\topsep{0pt}
\begin{center}
$\metadist$\\meta dist.
\end{center}

\end{minipage}};
% Text Node
\draw (185.23,103.5) node [anchor=north west][inner sep=0.75pt]   [align=left] {\begin{minipage}[lt]{49.27pt}\setlength\topsep{0pt}
\begin{center}
$\dist_j$\\task $\displaystyle j$ dist.
\end{center}

\end{minipage}};
% Text Node
\draw (58,14) node [anchor=north west][inner sep=0.75pt]   [align=left] {\textcolor[rgb]{0.25,0.46,0.02}{Distribution $\cD$}};
% Text Node
\draw (358.66,57) node [anchor=north west][inner sep=0.75pt]   [align=left] {\begin{minipage}[lt]{41.86pt}\setlength\topsep{0pt}
\begin{center}
Labeled \\samples\\$\displaystyle (\feat_{1}^{(j)},\, \lab_{1}^{(j)})$\\$\displaystyle (\feat_{2}^{(j)},\, \lab_{2}^{(j)})$\\$\displaystyle \vdots $\\$\displaystyle (\feat_{n}^{(j)},\,\lab_{n}^{(j)})$
\end{center}

\end{minipage}};
% Text Node
\draw (596,108) node [anchor=north west][inner sep=0.75pt]   [align=left] {$\zeta _{j} := ((\feat_1^{(j)},\lab_1^{(j)},\ldots, \feat_{\wit}^{(j)}, \lab_{\wit}^{(j)}), +1)$};
% Text Node
\draw (273,129) node [anchor=north west][inner sep=0.75pt]   [align=left] {$n$ times};
% Text Node
\draw (274,96) node [anchor=north west][inner sep=0.75pt]   [align=left] {sample};
% Text Node
\draw (114,97) node [anchor=north west][inner sep=0.75pt]   [align=left] {sample};
% Text Node
\draw (449,127) node [anchor=north west][inner sep=0.75pt]   [align=left] {\& add label +1};
% Text Node
\draw (450,94) node [anchor=north west][inner sep=0.75pt]   [align=left] {Concatenate};

\end{tikzpicture}
        }
        \caption{Constructing a datapoint drawn from distribution $\cD$.}
    
    \label{fig:newDataset}
\end{figure}
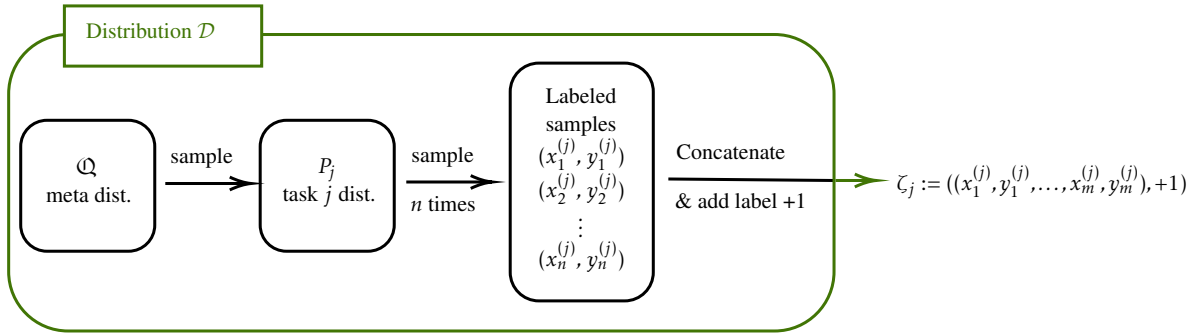
The prior lemma established a quantitative relationship between the test error of $\perclass$ and the probability of non-realizability. The following lemma shows how we can exploit this relationship for provable metalearning. For its proof we view the dataset of every task $j$ as a single ``data point'' $\zeta_j$ constructed as depicted in Figure~\ref{fig:newDataset}. Our goal is to learn a realizability predicate that labels these points correctly. We prove that such a predicate corresponds to a representation with bounded meta-error.
\fi
\begin{comment}
\newcommand{\lemNumTasks}{Let $\repclass$ be a class of representation functions $\rep: \featdom \to \mapdom$, $\perclass$ be a class of specialized classifiers $\per: \mapdom \to \bits$ and $\wit$ be a positive integer.
Assume metadistribution $\metadist$ is defined over a family of data distributions $\cP$ that is meta-realizable by $\perclass \circ \repclass$. Suppose there exists a strictly increasing convex function $\phi: (0,1]\to[0,1]$ such that for all data distributions $\dist \in \cP$ 
and all representations $\rep \in \repclass$, if $\rer(\dist, \rep, \perclass) > 0$ the following holds:
\begin{equation*} 
    \phi\left(\rer(\dist, \rep, \perclass)\right) \leq   \Pr[(\feat_1,\lab_1),\ldots ,(\feat_{\wit},\lab_{\wit})\sim \dist^{\wit}]{(\rep(\feat_1),\lab_1), \dots, (\rep(\feat_{\wit}),\lab_{\wit})\text{ is not realizable by }\perclass} \, .
\end{equation*}
 Then, we can $(\eps,\delta)$-\ktext{properly} metalearn  $(\repclass,\perclass)$ with $\users = O\left(\frac{\VC(\realpclass_{\wit, \perclass, \repclass})+\ln(1/\delta)}{\phi(\varepsilon)}\right)$ tasks and $m$ samples per task in the realizable case.}
 \end{comment}
 \newcommand{\lemNumTasks}{Let $\repclass$ be a class of representation functions $\rep: \featdom \to \mapdom$, $\perclass$ be a class of specialized classifiers $\per: \mapdom \to \bits$ and $\wit$ be a positive integer.
 Suppose there exists a strictly increasing convex function $\phi: (0,1]\to[0,1]$ such that for all metadistributions $\metadist$ satisfying $\min_{\rep \in \repclass} \rer(\metadist, \rep, \perclass) = 0$, all data distributions $\dist$ in the support of $\metadist$ and all representations $\rep \in \repclass$, if $\rer(\dist, \rep, \perclass) > 0$ the following holds:
\begin{equation*} 
    \phi\left(\rer(\dist, \rep, \perclass)\right) \leq   \Pr[(\feat_1,\lab_1),\ldots ,(\feat_{\wit},\lab_{\wit})\sim \dist^{\wit}]{(\rep(\feat_1),\lab_1), \dots, (\rep(\feat_{\wit}),\lab_{\wit})\text{ not realizable by }\perclass} \, .
\end{equation*}
 Then, we can $(\eps,\delta)$-properly metalearn  $(\repclass,\perclass)$ with $\users = O\left(\frac{\VC(\realpclass_{\wit, \perclass, \repclass})+\ln(1/\delta)}{\phi(\varepsilon)}\right)$ tasks and $m$ samples per task in the realizable case.}
\begin{lem}
 \lemNumTasks
 \label{lem:num-tasks}
\end{lem}
\ifnum \coltshort = 0
\begin{proof}
%Recall that $\metadist$ is a metadistribution over $\cP$. 
Our goal is to show that there exists a shared representation that achieves small representation error for $\metadist$.
%For the rest of this proof, we formalize the argument we mentioned above. 
The proof proceeds in two stages.
First, we use our bound on the VC dimension of $\realpclass_{\wit, \perclass, \repclass}$ to show that we can find a representation $\hat h$ that, when applied to data from a new task, admits a perfect specialized classifier with high probability.
Second, we connect this to $\rer(\metadist,\hat{\rep}, \perclass)$, the error of $\hat \rep$ on the meta-distribution.

Recall our approach for monotone thresholds in \Cref{sec:real-tech}:
for each task, we receive a data set $S^{(j)}=(\feat_1^{(j)},\lab_1^{(j)},\ldots, \feat_{\wit}^{(j)}, \lab_{\wit}^{(j)})$ and construct a single ``data point'' $\zeta_j$:
\begin{align*}
    \zeta_j \defeq (S^{(j)}, +1)\,.
\end{align*}
We set the label to ``$+1$'' because, by the meta-realizability assumption, there exists an $h^*$ such that $r_{h^*}(S^{(j)})=+1$ for all $j$.
Since each task distribution $\dist_j$ is drawn independently from $\metadist$, these $\zeta_j$ observations are drawn i.i.d.\ from some distribution $\cD$.
For an illustration of this process, see Figure~\ref{fig:newDataset}.

We find an $\hat h$ such that $r_{\hat h}$ has zero error on the dataset $\zeta_1,\ldots,\zeta_t$.
Because we set $t = \Theta\left(\frac{\VC(\realpclass_{\wit, \perclass, \repclass})+\ln(1/\delta)}{\phi(\varepsilon)}\right)$, by \Cref{fact:pac-vc} we know that $r_{\hat h}$ generalizes.
That is, with probability at least $1-\delta$, we have
\begin{equation}
    \Pr[\zeta=(S,+1)\sim \cD]{\realp_{\hat{\rep}}(S)\neq +1}\leq \phi(\varepsilon)\,.
    \label{eq:pr_nr_err_erm}
\end{equation}

For a data set $S=((\feat_1,\lab_1),\ldots ,(\feat_{\wit},\lab_{\wit}))$ and representation $\rep$, define  $S_h =((\rep(\feat_1),\lab_1),\ldots ,$ $(\rep(\feat_{\wit}),\lab_{\wit}))$.
Recall that by the assumption of this lemma, for $\hat{\rep}$ and all $\dist$ in the support of $\metadist$, if $\rer(\dist, \hat{\rep}, \perclass)>0$ we have:
\begin{equation}
\phi\left(\rer(\dist, \hat{\rep}, \perclass)\right) \leq  \Pr[S \sim \dist^{\wit}]{S_{\hat h} \text{ is not realizable by }\perclass}. 
% \phi\left(\rer(\dist, \hat{\rep}, \perclass)\right) \leq  \Pr[(\feat_1,\lab_1),\ldots ,(\feat_{\wit},\lab_{\wit})\sim \dist^{\wit}]{(\hat{\rep}(\feat_1),\lab_1), \dots, (\hat{\rep}(\feat_{\wit}),\lab_{\wit})\text{ is not realizable by }\perclass}. 
\end{equation}
Note that $\phi$ is a strictly increasing function and, thus, it has an inverse function $\phi^{-1}$ that is also strictly increasing. Therefore, the above bound implies that
\begin{equation} \label{eq:phi_assumption}
\rer(\dist, \hat{\rep}, \perclass)\leq  \phi^{-1}\left( \Pr[S\sim \dist^{\wit}]{S_{\hat h}\text{ is not realizable by }\perclass}\right). 
% \rer(\dist, \hat{\rep}, \perclass)\leq  \phi^{-1}\left( \Pr[(\feat_1,\lab_1),\ldots ,(\feat_{\wit},\lab_{\wit})\sim \dist^{\wit}]{(\hat{\rep}(\feat_1),\lab_1), \dots, (\hat{\rep}(\feat_{\wit}),\lab_{\wit})\text{ is not realizable by }\perclass}\right). 
\end{equation}

Now we are ready to bound the meta-error of $\hat{\rep}$: $\rer(\metadist, \hat{\rep}, \perclass)$. We start by bounding $\phi\left(\rer(\metadist, \hat{\rep}, \perclass)\right)$. Since $\phi$ is convex, $\phi^{-1}$ is concave and we can apply Jensen's inequality. We get that:
\begin{align*}
    \rer(\metadist,\hat{\rep}, \perclass) 
    &= \Exp_{\dist \sim \metadist}\left[\rer(\dist, \hat{\rep}, \perclass)\right]\\
    & \leq \Exp_{\dist \sim \metadist}\left[\rer(\dist, \hat{\rep}, \perclass)\mid\rer(\dist, \hat{\rep}, \perclass)>0\right]\\
    &\leq \Exp_{\dist \sim \metadist}\left[\phi^{-1}\left(\Pr[S\sim \dist^{\wit}]{S_{\hat h}\text{ is not realizable by }\perclass}\right)\right] \tag{by Eq. \ref{eq:phi_assumption}}\\
    & \leq \phi^{-1} \left(\Exp_{\dist \sim \metadist}\left[\Pr[S\sim \dist^{\wit}]{S_{\hat h}\text{ is not realizable by }\perclass}\right]\right)\tag{Jensen's inequality}\\
    & \leq \phi^{-1}(\phi(\eps))=\eps \tag{by Eq. \ref{eq:pr_nr_err_erm}}
    % &\leq \Exp_{\dist \sim \metadist}\left[\phi^{-1}\left(\Pr[(\feat_1,\lab_1),\ldots ,(\feat_{\wit},\lab_{\wit})\sim \dist^{\wit}]{(\hat{\rep}(\feat_1),\lab_1), \dots, (\hat{\rep}(\feat_{\wit}),\lab_{\wit})\text{ is not realizable by }\perclass}\right)\right] \tag{by Eq. \ref{eq:phi_assumption}}\\
    % & \leq \phi^{-1} \left(\Exp_{\dist \sim \metadist}\left[\Pr[(\feat_1,\lab_1),\ldots ,(\feat_{\wit},\lab_{\wit})\sim \dist^{\wit}]{(\hat{\rep}(\feat_1),\lab_1), \dots, (\hat{\rep}(\feat_{\wit}),\lab_{\wit})\text{ is not realizable by }\perclass}\right]\right)\tag{Jensen's inequality}\\
    % & \leq \phi^{-1}(\phi(\eps))=\eps \tag{by Eq. \ref{eq:pr_nr_err_erm}}
\end{align*}

The above bound implies: 
$$\rer(\metadist, \hat{\rep}, \perclass) \leq \varepsilon = \min_{\rep \in \repclass}\rer\paren{\metadist, \rep, \perclass} + \varepsilon\,.$$

As a result, 
using $\users = O\left(\frac{\VC(\realpclass_{\wit, \perclass, \repclass})+\log(1/\delta)}{\phi(\varepsilon)}\right)$ tasks and $\wit$ samples from each task,  we have found a representation function $\hat{\rep}$ that has the desired error bound for metalearning with probability $1-\delta$. Hence, the proof is complete.
\end{proof}

%\maryamnote{Does PAC learning imply the existence of an algorithm? Can we say we can find ERM for an uncountable infinite size class?}\adamnote{Good point. Maybe we can just say "outputting the representation $\hat h$ that maximizes the number of realizable tasks" or something similarly direct.}

We can now prove \Cref{thm:met-samples} by combining the results of \Cref{lem:real-to-error} and \Cref{lem:num-tasks}.
 \begin{proof}[Proof of Theorem \ref{thm:met-samples}] 
Let $$\phi(\varepsilon) = \frac{1}{2} \left[\frac{\wit \varepsilon}{16e\cdot \max(\VC(\perclass),\wit)\ln(16/\varepsilon)}\right]^{\wit}.$$ 
By differentiating $\phi$ twice we see that it is a strictly increasing convex function in $\varepsilon$, for $\eps \in (0,1)$ and $\wit \geq 1$.

Applying \Cref{lem:real-to-error}, we obtain that for every representation $\hat\rep$ and distribution $P$ if \\$\min_{\per \in \perclass}\Pr[(\feat, \lab) \sim \dist]{\per(\hat{\rep}(\feat))\neq y}>0$, then
\begin{align*}
    \phi\biggl(\min_{\per \in \perclass} & \Pr[(\feat, \lab) \sim \dist]{\per(\hat{\rep}(\feat))\neq y}\biggr) \\
        &\leq   \Pr[(\feat_1,\lab_1),\ldots ,(\feat_{\wit},\lab_{\wit})\sim \dist^{\wit}]{(\hat{\rep}(\feat_1),\lab_1), \dots, (\hat{\rep}(\feat_{\wit}),\lab_{\wit})\text{ is not realizable by }\perclass}.
\end{align*}

 This satisfies the assumptions of \Cref{lem:num-tasks} which says that we can metalearn $(\repclass, \perclass)$ with $\users= O\left(\frac{\VC(\realpclass_{\wit, \perclass, \repclass})+\ln(1/\delta)}{\phi(\varepsilon)}\right)$ tasks and $\wit$ samples per task. This concludes our proof.
 \end{proof}

Our next result analyzes metalearning with fewer tasks and more samples per task.
Formally, \Cref{thm:met-real-samples} show that we can metalearn $(\repclass, \perclass)$ for a meta-realizable $\metadist$ with $\Tilde{O}(\VC(\realpclass_{\samples, \perclass, \repclass})/\eps)$ tasks and $\Tilde{O}(\VC(\perclass)/\eps)$ samples per task.
\fi
\newcommand{\thmMetRealSamples}{Let $\repclass$ be a class of representation functions $\rep: \featdom \to \mapdom$ and $\perclass$ be a class of specialized classifiers $\per: \mapdom \to \bits$. For every $\eps, \delta \in (0,1)$, we can $(\varepsilon, \delta)$-properly metalearn $(\repclass, \perclass)$ in the realizable case with $\users$ tasks and $\samples$ samples per task when 
$\users = O\left(\frac{\VC(\realpclass_{\samples, \perclass, \repclass})\ln(1/\eps)+\ln(1/\delta)}{\eps}\right)\quad $ and $\quad\samples = O\left(\frac{\VC(\perclass) \ln(1/\eps)}{\eps}\right) .$}
 \begin{thm}
\label{thm:met-real-samples}
\thmMetRealSamples
\end{thm}
\if \coltshort = 0 
\begin{proof}
    First, we set $\eps_1:=\eps/3$. Set the number of tasks $$\users := O\left(\frac{\VC(\realpclass_{\samples, \perclass, \repclass})\ln(1/\eps_1)+\ln(1/\delta)}{\eps_1}\right)$$.
    Each task $j$ has $\samples $ samples $S_j = \{(\feat_i^{(j)}, \lab_i^{(j)})\}_{i \in [\samples]}$, we construct a dataset $Z$ of $\users$ points $\zeta_j = ((\feat_1^{(j)},\lab_1^{(j)},\ldots, \feat_{\samples}^{(j)}, \lab_{\samples}^{(j)}),+1)$, one for each task $j$, as in \Cref{fig:newDataset}. Each $\zeta_j$ is drawn i.i.d.\ from data distribution $\cD$ (where we first draw $\dist$ from $\metadist$ and then $\samples$ points from $\dist$). 
    Since we are in the realizable case, there exists a representation $\rep$ such that $\realp_\rep$ returns $+1$ for every sample drawn from $\cD$. By Fact~\ref{fact:pac-pd} we have that for $\hat{\rep} = \arg \min_{\rep \in \repclass} \er(Z,\realp_\rep)$ with probability at least $1-\delta$ over dataset $Z$
\begin{align*}
   \er(\cD, \realp_{\hat{\rep}})\leq \frac{\eps}{3}.
\end{align*}

Fix a distribution $\dist$ in the support of $\metadist$. We start by assuming that $\rer(\dist, \hat{\rep}, \perclass) > \eps/3$. By \Cref{fact: vc-gen} for a dataset $S$ of $$\samples := \frac{24}{\eps} \left(\VC(\perclass) \ln(48/\eps) + \ln(4)\right)$$ samples drawn from $\dist$, with probability at least $1/2$ over $S$ all specialized classifiers $\per \in \perclass$ with $\er(\dist, \per \circ \hat{\rep}) >\eps/3$ have $\er(S, \per \circ \hat{\rep})>0$. Therefore,  
\begin{align*}
    \Pr[S \sim \dist^\samples]{\min_{\per \in \perclass} \er(S, \per\circ\hat{\rep})> 0} =
    \Pr[S \sim \dist^\samples]{\realp_{\hat{\rep}}(S) \neq -1}   \geq \frac{1}{2}.
\end{align*}
By Markov's inequality, we see that
\begin{align*}
    \Pr[\dist \sim \metadist]{\Pr[S\sim \dist^n]{\realp_{\hat{\rep}}(S) \neq +1} \geq 1/2} \leq 2 \Exp_{\dist \sim \metadist}\left[\Pr[S\sim \dist^n]{\realp_{\hat{\rep}}(S) \neq +1}\right] = 2~\er(\cD, \realp_{\hat{\rep}})\leq \frac{2\eps}{3}.
\end{align*}
So far we have shown that for a fixed $\dist$ if $\rer(\dist, \hat{\rep}, \perclass) > \eps/3$, then $\Pr[S\sim \dist^n]{\realp_{\hat{\rep}}(S) \neq +1} \geq 1/2$. As a result, we have that
\begin{align*}
    \Pr[\dist \sim \metadist]{\rer(\dist, \hat{\rep}, \perclass) > \eps/3} \leq   \Pr[\dist \sim \metadist]{\Pr[S\sim \dist^n]{\realp_{\hat{\rep}} \neq +1} \geq 1/2} \leq   \frac{2\eps}{3}.
\end{align*}
We can use this to bound the meta-error of representation $\hat{\rep}$ as follows. We see that with probability at least $1-\delta$ over the datasets of the $\users$ tasks
\begin{align*}
    \rer(\metadist, \hat{\rep}, \perclass) &=\Exp_{\dist \sim \metadist}\left[ \min_{\per \in \perclass} \Pr[(x,y) \sim \dist]{ \per(\hat{\rep}(x)) \neq y}\right]\\
    & \leq \frac{\eps}{3} + \Pr[\dist \sim \metadist]{\min_{\per \in \perclass} \Pr[(\feat, \lab) \sim \dist]{\per(\hat{\rep}(\feat))\neq \lab} > \eps/3} \\
    &\leq \frac{\eps}{3} + 2\frac{\eps}{3} = \eps.
\end{align*}
This concludes our proof.
\end{proof}
\fi

\ifnum \coltshort = 0
\subsection{Sample and Task Complexity Bounds for the Agnostic Case}
\label{sec:metalearning_agnostic}

In this section, we consider metalearning when the underlying metadistribution is not realizable. Our proof uses the pseudodimension of the class of empirical error functions to bound the number of tasks needed to metalearn. %Specifically, in \Cref{thm:agn-met-samples} we show that we can metalearn $(\repclass, \perclass)$ with $\Tilde{O}(\pd((\erindclass_{\samples, \perclass, \repclass })/\eps^2)$ tasks and $\Tilde{O}(\VC(\perclass)/\eps^2)$ samples per task.

\begin{thm}[Restatement of \Cref{thm:agn-met-samples-intro}]
\label{thm:agn-met-samples}
\thmagnmet
%Let $\repclass$ be a class of representation functions $\rep: \featdom \to \mapdom$ and $\perclass$ be a class of specialized classifiers $\per: \mapdom \to \bits$. Then, we can $(\varepsilon, \delta)$-metalearn $(\repclass, \perclass)$ with $\users \geq O\left(\frac{\pd(\erindclass_{\samples, \perclass, \repclass })\ln(1/\varepsilon)+\ln(1/\delta)}{\varepsilon^2}\right)$ tasks and $\samples = O\left( \frac{\VC(\perclass)+\ln(1/\varepsilon)}{\varepsilon^2}\right)$ samples per task.
\end{thm}
\begin{proof} 
We begin by setting our parameters: Set $\eps_1 \coloneqq \eps/3$ and $\eps_2 \coloneqq \eps/3$. Let the number of tasks be the following for a sufficiently large constant in the $O$ notation:
\begin{equation} \label{eq:uniform_convergence_pdim}
    \users \coloneqq O\left( \frac{\pd(\erindclass_{\samples, \perclass, \repclass })\ln(1/\eps_1)+\ln(1/\delta))}{\eps_1^2}\right)\,.
\end{equation} 
Suppose for each task $j \in [t]$ has $\samples $ samples $S_j =\{(\feat_i^{(j)}, \lab_i^{(j)})\}_{i \in [\samples]}$, we construct a dataset $Z$ of $\users$ points $\zeta_j = ((\feat_1^{(j)},\lab_1^{(j)},\ldots, \feat_{\samples}^{(j)}, \lab_{\samples}^{(j)}))$, one for each task $j$. Each $\zeta_j$ is drawn i.i.d.\ from the data distribution $\cD$ for which we first draw $\dist$ from $\metadist$ and then $\samples$ points from $\dist$. By Fact~\ref{fact:pac-pd}, we have that for $\hat{\rep} = \arg \min_{\rep \in \repclass} \frac{1}{\users}\sum_{j=1}^\users\erind_{\rep}(\zeta_j)$ with probability at least $1-\delta$:
\begin{align*}
    \Exp_{\zeta \sim \cD}\left[\erind_{\hat{\rep}}(\zeta)\right]\leq \min_{\rep \in \repclass}\Exp_{\zeta \sim \cD}\left[\erind_\rep (\zeta)\right] + \frac{\varepsilon}{3}\,.
\end{align*}
Fix a task distribution $\dist$ and a representation $\rep$. 
% For any given set of $n$ samples $\zeta$ and
For any fixed specialized classifier $\per' \in \perclass$, we have: 
$$\Exp_{(\feat_1,\lab_1),\ldots,(\feat_n, \lab_n) \sim \dist^n} \left[
\min_{\per \in \perclass} \frac{1}{\samples}\sum_{i \in [\samples]} \ind\{\per(\rep(\feat_i)) \neq \lab_i\} 
\right]
\leq 
\Exp_{(\feat_1,\lab_1),\ldots,(\feat_n, \lab_n) \sim \dist^n} \left[ 
\frac{1}{\samples}\sum_{i \in [\samples]} \ind\{\per'(\rep(\feat_i)) \neq \lab_i\}\right]\,.$$
Since the inequality holds for any $\per'$ in $\perclass$, it holds when we take minimum over all $\per'$. Hence, we obtain:
\begin{equation}\label{eq:swap_min}
    \Exp_{(\feat_1,\lab_1),\ldots,(\feat_n, \lab_n) \sim \dist^n} \left[ \min_{\per \in \perclass} \frac{1}{\samples}\sum_{i \in [\samples]} \ind\{\per(\rep(\feat_i)) \neq \lab_i\}\right] \leq \min_{\per \in \perclass} \Exp_{(\feat_1,\lab_1),\ldots,(\feat_n, \lab_n) \sim \dist^n}\left[\frac{1}{\samples}\sum_{i \in [\samples]} \ind\{\per(\rep(\feat_i)) \neq \lab_i\}\right]\,.
\end{equation}  
Next, we use the above inequality to continue bounding the expectation of  $\erind_{\hat{\rep}}$:
%Therefore, we have that with probability at least $1-\delta$ over the randomness of our constructed sample set of $Z = \left\{\zeta_j\right\}_{j \in [t]}$
\begin{align*}
    \Exp_{\zeta \sim \cD}\left[\erind_{\hat{\rep}}(\zeta)\right]&\leq
    \min_{\rep \in \repclass}\Exp_{\zeta \sim \cD}\left[\erind_\rep (\zeta)\right] + \frac{\varepsilon}{3} \tag{Eq.~\eqref{eq:uniform_convergence_pdim}}
    \\
    &=\min_{\rep \in \repclass} \Exp_{\dist \sim \metadist}\left[\Exp_{(\feat_1,\lab_1),\ldots,(\feat_n, \lab_n) \sim \dist^n}\left[\min_{\per \in \perclass}\frac{1}{n}\sum_{i=1}^n \ind \{\per(\rep(\feat_i))\neq \lab_i\}\right]\right] +\frac{\varepsilon}{3} \\
    &\leq
    \min_{\rep \in \repclass} \Exp_{\dist \sim \metadist}\left[\min_{\per \in \perclass}\Exp_{(\feat_1,\lab_1),\ldots,(\feat_n, \lab_n) \sim \dist^n}\left[\frac{1}{n}\sum_{i=1}^n \ind \{\per(\rep(\feat_i))\neq \lab_i\}\right]\right] + \frac{\varepsilon}{3}
    \tag{Eq.~\ref{eq:swap_min}}
    \\
    &= \min_{\rep \in \repclass} \rer(\metadist, \rep, \perclass)+\frac{\varepsilon}{3}.
\end{align*}
Consider an arbitrary task distribution $\dist$. Suppose we have a dataset $S$ of $\samples$ labeled sample from $\dist$ where $n$ is the following with a sufficiently large constant in the $O$ notation:
$$\samples \coloneqq O\left( \frac{\VC(\perclass)+\ln(1/\eps_2)}{\eps_2^2}\right)\,.$$ 
Since $n$ is sufficiently large, we have uniform convergence of the empirical error for all $\per \in \perclass$ by \Cref{fact:pac-vc}. Furthermore, the specialized classifier $\hat{\per}$ minimizing the empirical error over $S$, (i.e. $\hat{\per} = \arg \min_{\per \in \perclass} \er(S, \per \circ\hat{\rep})$) must have low true error as well. Therefore, with probability $1-\frac{\eps}{3}$ over the randomness in $S$, we get:

\begin{align*}
    \rer(\dist, \hat{\rep}, \perclass) \leq \Pr[(\feat, \lab) \sim \dist]{ \hat{\per}(\hat{\rep}(\feat))\neq \lab} \leq \rer(S, \hat{\rep}, \perclass)+\frac{\varepsilon}{3}\,.
\end{align*}
Thus, if we take the expectation over the dataset $S$, we see that:
\begin{align*}
    \rer(\dist, \hat{\rep},\perclass)& =\Exp_{S \sim \dist^\samples}\left[\rer(\dist, \hat{\rep},\perclass)\right] \\
    &= \Exp_{S \sim \dist^\samples}\left[\rer(\dist, \hat{\rep},\perclass)\left|\rer(\dist, \hat{\rep},\perclass) \leq \rer(S, \hat{\rep}, \perclass)+\frac{\varepsilon}{3} \right.\right]\\
    &\quad\times \Pr[S \sim \dist^\samples]{\rer(\dist, \hat{\rep},\perclass) \leq \rer(S, \hat{\rep}, \perclass)+\frac{\varepsilon}{3}}\\
    &\quad+ \Exp_{S \sim \dist^\samples}\left[\rer(\dist, \hat{\rep},\perclass)\left|\rer(\dist, \hat{\rep},\perclass) > \rer(S, \hat{\rep}, \perclass)+\frac{\varepsilon}{3}\right.\right]\\
    &\quad\times \Pr[S \sim \dist^\samples]{\rer(\dist, \hat{\rep},\perclass) >\rer(S, \hat{\rep}, \perclass)+\frac{\varepsilon}{3}}\\
    &\leq \left(\Exp_{S \sim \dist^\samples}\left[\rer(S, \hat{\rep}, \perclass)\right] +\frac{\eps}{3}\right)\cdot 1 +1 \cdot \frac{\varepsilon}{3} \\
    &= \Exp_{S \sim \dist^\samples}\left[\rer(S, \hat{\rep}, \perclass)\}\right] +\frac{2\varepsilon}{3}.
\end{align*}
Now, we take the expectation over $\dist \sim \metadist$ and obtain that:
\begin{align*}
    \rer(\metadist, \hat{\rep}, \perclass) &\leq  \Exp_{\dist \sim \metadist}\left[\Exp_{S \sim \dist^\samples}\left[\rer(S, \hat{\rep}, \perclass)\}\right]\right] +\frac{2\varepsilon}{3}\\
    & = \Exp_{\zeta \sim \cD}\left[\erind_{\hat{\rep}}(\zeta)\right]+\frac{2\varepsilon}{3}\,.
\end{align*}
Note that in the last line above, the $S$ dataset drawn from a random $\dist$ can be viewed as a random data set according to $\cD$. Combining with the bound we have derived earlier for the expected value of $\erind_{\hat{\rep}}$, we get the following that holds with probability at least $1-\delta$:
\begin{align*}
    \rer(\metadist, \hat{\rep}, \perclass) \leq \min_{\rep \in \repclass} \rer(\metadist, \rep, \perclass) + \varepsilon
    \,.
\end{align*}
Hence, the proof is complete. 
\end{proof}

\subsection{Sample and Task Complexity Bounds for General Metalearning}

For general metalearning we obtain the following corollary of Theorem~\ref{thm: meta-to-mtl} by applying Theorem~\ref{thm:mtl-smpls} for agnostic multitask learning. It is important to note that this method does not return a representation $\rep\in \repclass$, but a more general specialization algorithm that uses the datasets of the already seen tasks to learn a classifier for the new task. 

\begin{cor}
\label{cor:met-vc}
    For any $(\repclass, \perclass)$, any $\eps>0$, constant $c>0$ and any $\users, \samples$, $(\repclass, \perclass)$ is $(\eps,2/c)$-meta-learnable in the agnostic case with $\users$ tasks, $\samples$ samples per task and $\samples$ specialization samples when
    \[
    \samples \users = O\left(\frac{\VC(\perclass^{\otimes \users+1}\circ \repclass)\ln(1/\eps)}{\eps^2}\right).
    \]
\end{cor}
\begin{proof}
    By Theorem~\ref{thm: meta-to-mtl} we can $(\eps/c, \eps/c)$-multitask learn $(\repclass, \perclass)$ for $\users+1$ tasks and $\samples$ per task when \[\samples \users+ \samples= \frac{c'c^2\VC(\perclass^{\otimes \users+1}\circ \repclass)\ln(c/\eps)}{\eps^2},\] for a constant $c'>0$. As a result, the reduction of Theorem~\ref{thm: meta-to-mtl} implies that we can $(\eps, 2/c)$-metalearn $(\repclass, \perclass)$ for $\users$ tasks, $\samples$ samples per task and $\samples$ specialization samples when $$\samples \users = O\left(\frac{\VC(\perclass^{\otimes \users+1}\circ \repclass)\ln(1/\eps)}{\eps^2}\right).$$ This concludes the proof.
\end{proof}

\else
Next, we consider metalearning when the underlying metadistribution is not realizable, and show that the pseudodimension of the empirical error predicate (Definition~\ref{def:empirical_error_function}) bounds the number of tasks needed. %Specifically, in \Cref{thm:agn-met-samples} we show that we can metalearn $(\repclass, \perclass)$ with $\Tilde{O}(\pd((\erindclass_{\samples, \perclass, \repclass })/\eps^2)$ tasks and $\Tilde{O}(\VC(\perclass)/\eps^2)$ samples per task.

\begin{thm}[Restatement of Theorem~\ref{thm:agn-met-samples-intro}]
\label{thm:agn-met-samples}
\thmagnmet
%Let $\repclass$ be a class of representation functions $\rep: \featdom \to \mapdom$ and $\perclass$ be a class of specialized classifiers $\per: \mapdom \to \bits$. Then, we can $(\varepsilon, \delta)$-metalearn $(\repclass, \perclass)$ with $\users \geq O\left(\frac{\pd(\erindclass_{\samples, \perclass, \repclass })\ln(1/\varepsilon)+\ln(1/\delta)}{\varepsilon^2}\right)$ tasks and $\samples = O\left( \frac{\VC(\perclass)+\ln(1/\varepsilon)}{\varepsilon^2}\right)$ samples per task.
\end{thm}
\fi

\ifnum \coltshort = 0
\section{Bounds for Halfspaces over Linear Representations}
\else
\section{Metalearning of Halfspaces over Linear Representations}
\fi
\label{sec:lin}

In this section, we focus on \if \coltshort = 0 multitask learning and\fi metalearning of halfspaces with a shared linear representation:
$$
\repclass_{\datadim,\repdim} = \left\{\rep \mid \rep(\mathbf{\feat}) = \mathbf{B} \mathbf{\feat}, \mathbf{B} \in \R^{\repdim\times \datadim}\right\} \,,
 \text{ and }\hspace{3pt}
\perclass_{\repdim} = \left\{\per\mid  \per(\mathbf{\map}) = \textrm{sign}(\mathbf{a}\cdot \mathbf{\map} - w),  \mathbf{a} \in \R^{\repdim}, w \in \R \right\}\,.$$
More precisely, the representation is a linear projection that maps vectors in $d$ dimensions to vectors in $k$ dimensions. The specialized classifiers are halfspaces over the $k$-dimensional representation.

%We consider both realizable and agnostic cases. 
\if \coltshort = 0
We use our general results from Sections \ref{sec:mtl} and \ref{sec:meta} to bound the number of tasks and the number of samples per task we need to\ifnum\coltshort=0multitask learn and \fi metalearn $(\repclass_{\datadim,\repdim},\perclass_{\repdim})$. Our\ifnum\coltshort=0 multitask learning result (\Cref{thm:mtl-vc-hl}) relies on bounding the VC dimension of class $\perclass_{\repdim}^{\otimes \users}\circ \repclass_{\datadim,\repdim}$.\fi
Our results for metalearning (\Cref{cor:met-lin-real1} for realizable learning and \Cref{cor:met-lin-agn} for agnostic) rely on bounding the VC dimension of the class of realizability predicates and the pseudodimension of the class of empirical error functions. 
\else 
We use our results for general classes from Section \ref{sec:meta} to bound the number of tasks and the number of samples per task we need to properly metalearn $(\repclass_{\datadim,\repdim},\perclass_{\repdim})$. 
Our results (Corollary~\ref{cor:met-lin-real1} for realizable learning and Corollary~\ref{cor:met-lin-agn} for agnostic) rely on bounding the VC dimension of the class of realizability predicates and the pseudodimension of the class of empirical error functions. 
\ifnum \coltshort = 1 See \Cref{sec:hspaces_linear_bounds_app} for proofs and extensions.\fi 
    % The complete proofs of our results for this section\ifnum\coltshort=1~as well as analogous results for multitask learning\fi~are contained in \Cref{sec:hspaces_linear_bounds_app}.
\ifnum \coltshort = 0
\subsection{VC Dimension of  \texorpdfstring{$\perclass_{\repdim}^{\otimes \users}\circ \repclass_{\datadim,\repdim}$}{Lg}}

The general bounds of \Cref{lem:vc-bounds} give us that the VC dimension of $\perclass_{\repdim}^{\otimes \users}\circ \repclass_{\datadim, \repdim}$ is in the range
$$\max(\repdim \users + \users, \datadim+1) \leq \VC(\perclass_{\repdim}^{\otimes \users}\circ \repclass_{\datadim, \repdim}) \leq \datadim \users + \users.$$ 
We know the VC dimension of the class of composite functions  $\perclass_{\repdim}\circ \repclass_{\datadim,\repdim}$ because $\perclass_{\repdim}\circ \repclass_{\datadim,\repdim}$ and $\perclass_{\datadim}$, the class of $d$-dimensional halfspaces, are the same class.

In \Cref{thm:mtl-vc-hl} we characterize the VC dimension of class $\perclass_{\repdim}^{\otimes \users}\circ \repclass_{\datadim,\repdim}$ up to a constant. 
The bound we give matches the intuition from counting the number of parameters,
which yields $\datadim\repdim+\repdim\users + \users$ when $\users > \repdim$ and $\datadim\users + \users$ when $\users \leq \repdim$. 

\begin{thm}[Restatement of \Cref{thm:mtl-vc-hl-intro}]
\label{thm:mtl-vc-hl}
    \thmMtlVcHLText
\end{thm}

The proof for the upper bound we use Warren's Theorem, which we state here as a lemma:

\begin{lem}[Warren's Theorem,~\cite{Warren68}]
\label{fct:vc-polynomials}
Let $p_1,\ldots, p_{n}$ be real polynomials in $v$ variables, each of degree at most $\ell \geq 1$. If $n \geq v$, then the number of distinct sequences $\textrm{sign}(p_1(x)), \ldots,$ $ \textrm{sign}(p_n(x))$ for all $x$ does not exceed $(4e\ell n/v)^v$. In particular, if $\ell \geq 2$ and $n \geq 8v\log_2 \ell$, then the number of distinct sequences of $+1,-1$ taken by $\textrm{sign}(p_1(x)), \ldots, $ $\textrm{sign}(p_n(x))$ is less than $2^n$.
\end{lem}

\begin{proofsk}[Proof Sketch (\Cref{thm:mtl-vc-hl})]
    For $\users \leq \repdim$ tasks we can see that $\perclass_\repdim^{\otimes \users}\circ \repclass_{\datadim,\repdim} = \perclass_{\datadim}^{\otimes \users}$ and, thus, $\VC\left(\perclass_{\repdim}^{\otimes \users}\circ \repclass_{\datadim,\repdim}\right )  = \datadim \users +\users$. For more details about this case see the full proof in \Cref{sec:mtl-vc-hl-proof}.
    
    The novelty of the proof for this theorem lies in getting an upper bound for the case where we have more than $\repdim$ tasks. First, we rewrite functions $\conc \in \perclass_{\repdim}^{\otimes \users}\circ \repclass_{\datadim,\repdim}$ as $\conc(j,\mathbf{\feat}) = \textrm{sign}(\mathbf{a}_j\mathbf{B}\mathbf{\feat}-w_j)$.
    Observe that this is equivalent to
    \begin{align*}
        g(j, \mathbf{\feat}) = \textrm{sign}(\mathbf{e_j}^T \mathbf{A} \mathbf{B} \mathbf{\feat} - \mathbf{e_j}^T\mathbf{w}),
    \end{align*}
    where $\mathbf{e}_j \in \{0,1\}^{\users}$ is the one-hot encoding of $j$ and 
    \[
    \mathbf{A} = \begin{pmatrix}
        \mathbf{a}_1^T \\
        \vdots \\
        \mathbf{a}_{\users}^T
    \end{pmatrix} \textrm{ and } \mathbf{w} = \begin{pmatrix}
        w_1 \\
        \vdots \\
        w_{\users}
    \end{pmatrix}.
    \]
    Every combination of $\mathbf{A} \in \R^{\users\times \repdim}, \mathbf{B} \in R^{\repdim \times \datadim}$ and $\mathbf{w} \in \R^{\users}$ gives us a specific labeling function $\conc$. 
    Let $\samples \geq 8(\users \repdim + \repdim \datadim+\users)$. 
    Take a dataset $((j_1,\mathbf{\feat}_1),\ldots,(j_n, \mathbf{\feat}_{\samples})) \in ([\users] \times \R^{\datadim})^{\samples}$.
    We will show that $\perclass_{\repdim}^{\otimes \users}\circ \repclass_{\datadim,\repdim}$ does not shatter this data set.
    For each $i\in[\samples]$, we define a polynomial $p_i(\mathbf{A},\mathbf{B}, \mathbf{w}) = \mathbf{e}_{j_i}^T \mathbf{A} \mathbf{B}\mathbf{\feat}_i - \mathbf{e}_{j_i}^T\mathbf{w}$.
    Each of these is a degree-2 polynomial in $\users \repdim + \repdim \datadim+\users$ variables. By construction, $g(j_i,\mathbf{x}_i) =\sgn(p_i(\mathbf{A},\mathbf{B},\mathbf{w}))$. By \Cref{fct:vc-polynomials}, $\VC\left(\perclass_{\repdim}^{\otimes \users}\circ \repclass_{\datadim,\repdim}\right ) \leq 8 (\datadim \repdim +\repdim\users + \users)$.
    
    By \Cref{lem:vc-bounds} we have $\VC(\perclass_{\repdim}^{\otimes \users}\circ \repclass_{\datadim,\repdim}) \geq \repdim \users + \users$. Additionally, for $\users > \repdim$ any dataset shattered by $\perclass_\repdim^{\otimes \repdim}\circ \repclass_{\datadim,\repdim}$ can be shattered by $\perclass_\repdim^{\otimes \users}\circ \repclass_{\datadim,\repdim}$. As a result, $\VC(\perclass_\repdim^{\otimes \users}\circ \repclass_{\datadim,\repdim}) \geq \VC(\perclass_\repdim^{\otimes \repdim}\circ \repclass_{\datadim,\repdim})$. We proved above that $\VC(\perclass_{\repdim}^{\otimes \repdim}\circ \repclass_{\datadim,\repdim}) = \datadim \repdim + \repdim$.  Therefore, $\VC(\perclass_{\repdim}^{\otimes \users}\circ \repclass_{\datadim,\repdim}) \geq \Omega (\datadim \repdim + \repdim \users).$
\end{proofsk}

Recall that we can reduce metalearning to multitask learning and obtain the task and sample complexity in \Cref{cor:met-vc}. Therefore, the characterization of the VC dimension of class $\perclass_{\repdim}^{\otimes \users}\circ \repclass_{\datadim,\repdim}$ implies the following bound for the task and sample complexity of metalearning.
\begin{cor}
\label{cor:imp_meta_hlr}
    For $\eps>0$ and constant $c > 0$, we can $(\eps, c/2)$-metalearn $(\repclass_{\datadim,\repdim}, \perclass_{\repdim})$ with $\users$ tasks and $\samples$ samples per task when \[
    \samples = O\left( \frac{ \repdim \ln(1/\eps)}{\eps^2}\right) \text{ and }\users \geq \datadim.
    \]
\end{cor}
\begin{proof}
    Based on Theorem~\ref{thm:mtl-vc-hl} when $\users \geq \datadim$, we know that $\VC(\perclass_{\repdim}^{\otimes \users}\circ \repclass_{\datadim,\repdim}) = O(\repdim \users)$. Hence, by Corollary~\ref{cor:met-vc} we can metalearn $(\repclass_{\datadim,\repdim}, \perclass_{\repdim})$ up to error $\eps$ when $\samples = O\left( \frac{ \repdim \ln(1/\eps)}{\eps^2}\right)$.
\end{proof}

\fi
\ifnum \coltshort = 0
\subsection{Bounding the VC Dimension of Boolean Functions of Polynomials} 
\fi

Before presenting our results for proper metalearning, we provide a lemma that generalizes Warren's theorem (Lemma~\ref{fct:vc-polynomials}). 
We construct a parameterized class of hypotheses based on a given Boolean function $g$ whose inputs are signs of polynomials. 
% For each hypothesis the set of polynomials is determined by the class parameter of that hypothesis. 
% We define this class more formally in the statement of our lemma. 
Appealing to Warren's theorem, we show that the VC dimension of this class cannot be too large. 
% While this particular form of this hypothesis class may seem too specific, it matches exactly what we need to bound the VC dimension of realizability predicates later. 
% Our lemma states the following:

%First, we prove \Cref{lem:vc-bool-pol}, which extends Warren's theorem (\Cref{fct:vc-polynomials}) to Boolean functions of signs of polynomials.
\newcommand{\lemVCBoolPol}{
     Let $w, \ell, d\in\mathbb{N}$ with $\ell\ge 2$ and fix a domain of features $\featdom$ and a function $g: \{\pm 1\}^w \to \{\pm 1\}$.
For each $x\in\featdom$, let $p_x^{(1)},\ldots,p_x^{(w)}$ be degree-$\ell$ polynomials in $d$ variables.
Consider the hypothesis class $\class$ consisting of, for all $\mathbf{v}\in \R^d$, the functions $c_{\mathbf{v}}:\featdom\to\{\pm 1\}$ defined as
\begin{align}
     c_{\mathbf{v}}(x) := g(\sgn(p_x^{(1)}(\mathbf{v})),\ldots,\sgn(p_x^{(w)}(\mathbf{v}))).
\end{align}
We have $\mathrm{VC}(\class)\le 8 d \log_2 (\ell w)$.
}
\begin{lem}
 \label{lem:vc-bool-pol}
\lemVCBoolPol
% 
% as the functions $c_{\mathbf{v}}:\featdom\to\{\pm 1\}$ for all $\mathbf{v} \in \R^d$ defined in the following way:
%  \begin{align}
%      c_{\mathbf{v}}(x) := g(\sgn(p_x^1(\mathbf{v})),\ldots,\sgn(p_x^w(\mathbf{v})))
%  \end{align}
%  has $\mathrm{VC}(\class)\le 8 d \log_2 (\ell w)$.
 \end{lem}
 \ifnum \coltshort = 0
\begin{proof}
    We bound the growth function of $\class$.  
    Fix $n$ datapoints $x_1, \ldots, x_n \in \featdom$. By the definition of hypothesis class $\class$, the sequence $c_{\mathbf{v}}(x_1),$ $ \ldots, c_{\mathbf{v}}(x_n)$ is exactly the sequence
    $$
        g\bigl(\sign (p_{x_1}^{(1)}(\mathbf{v})), \ldots, \sign (p_{x_1}^{(w)}(\mathbf{v}))\bigr),\ldots, g\bigl(\sign(p_{x_n}^{(1)}(\mathbf{v})), \ldots, \sign (p_{x_n}^{(w)}(\mathbf{v}))\bigr)
    $$ 
    By \Cref{fct:vc-polynomials}, we know that the number of distinct sequences 
    $$
        \sign (p_{x_1}^{(1)}(\mathbf{v})), \ldots, \sign (p_{x_1}^{(w)}(\mathbf{v})),\ldots, \sign(p_{x_n}^{(1)}(\mathbf{v})), \ldots, \sign (p_{x_n}^{(w)}(\mathbf{v}))
    $$ 
    for all $\mathbf{v} \in \R^d$ does not exceed  $(4e\ell w n/d)^d$. The number of distinct outputs of a function is upper bounded by the number of distinct inputs it gets. As we saw above, the number of distinct inputs is at most $(4e\ell w n/d)^d$ and, thus, $|\{(c_{\mathbf{v}}(x_1), \ldots, c_{\mathbf{v}}(x_n))\mid \mathbf{v} \in \R^d\}| \leq (4e\ell wn/d)^d$. Since this holds for fixed $x_1, \ldots, x_n \in \featdom$, the growth function of $\class$ is at most $(4e\ell wn/d)^d$.

% \gb{Can we just cite standard ''growth bound'' -> ''VC bound'' calculation?} 
% \maryaminline{Good point, but I am not sure where to find the statement of that lemma with the exact constants. [Sauer's lemma is the other direction].  (I am too lazy to find it frankly :P )}
    To bound the VC dimension of $\class$ it suffices to show that for $n > 8d\log_2(\ell w)$ points, the growth function $\growth_{\class}(n)$ is less than $2^n$. If $n > 8d\log_2(\ell w)$, then $n > 8 d$ because $\ell \geq 2$ and $ w \geq 1$. We know that when $n/d > 8$, $\log_2(4e) < n/(2d)$ and $\log_2(n/d) < 3n/(8d)$. Therefore, we see that 
    \[
        \log_2\left(\frac{4e\ell w n}{d}\right) = \log_2(4e) + \log_2(\ell w) + \log_2\left(\frac{n}{d}\right) <  \frac{n}{2d}+\frac{n}{8d} + \frac{3n}{8d} = \frac{n}{d}.
    \]
   Combining the above steps, we have proven that for $n> 8\log_2(\ell w)$, $\growth_{\class}(n) \leq (4e\ell wn/d)^d < 2^n$.
\end{proof}
\fi
\if \coltshort = 0
\subsection{VC Dimension of the Realizability Predicate Class}
\fi
Our central result on halfspaces over linear representations in the realizable setting is Theorem~\ref{thm:vc-real-lin}, which bounds the VC dimension of the realizability predicates (see Definition~\ref{def:realizability_predicate}). 
This VC dimension bound, combined with our metalearning bounds for general classes in \Cref{sec:meta}, immediately yields our statement about metalearning presented in Corollary~\ref{cor:met-lin-real1}.

\newcommand{\thmVCRealLin}{
     For every $\datadim, \repdim, \samples$ with $\samples\ge \repdim+2$, 
        $\VC(\realpclass_{\samples,\perclass_{\repdim},\repclass_{\datadim, \repdim}}) \leq O(d\repdim n+dk\log(\repdim n))$.
}
\begin{thm}
    \thmVCRealLin
    \label{thm:vc-real-lin}
\end{thm}

\newcommand{\textMetLinReal}{
    We can $(\varepsilon, \delta)$-\ktext{properly} metalearn $(\repclass_{\datadim,\repdim}, \perclass_\repdim)$ in the realizable case with $\users$ tasks and $\samples$ samples per task when 
     $\users  
    =\left(\datadim\repdim^2+\ln\left(1/\delta\right)\right) \cdot O\left(\frac{\ln(1/\eps)}{\eps}\right)^{\repdim+2}$ and $\samples = \repdim+2$, or
    $\users = O\left(\frac{\datadim \repdim^2 \ln^2(1/\eps)}{\eps^2}+\frac{\ln(1/\delta)}{\eps}\right)$ and $\samples = O\left(\frac{\repdim \ln(1/\eps)}{\eps}\right) .$
}

\begin{cor}[Detailed version of Corollary~\ref{cor:met-lin-real1}]
    \label{cor:met-lin-real1}\adamnote{Maybe incorporate corollary of general reduction here, too.}
    \textMetLinReal
\end{cor}

 \begin{proof}[of Corollary~\ref{cor:met-lin-real1}]
% \gb{Where do we use Kirchberger?}\kb{we need it in the corollary because it's the witness size.}\maryaminline{I moved it to the proof}
For the class of halfspaces, the following result of Kirchberger gives the exact %we have this clever observation by 
%Kirchberger indicating the
\nrcc:
\begin{fact}[%Kirchberger %'s Theorem 
\cite{Kirchberger1903}]\label{fact:kirchberger}
    Every dataset of size at least $\repdim+2$ that is not realizable by $k$-dimensional halfspaces has a non-realizability certificate of size $\repdim+2$.  Therefore, $\NRC(\perclass_{\repdim}) = k+2$.
\end{fact}

Theorem~\ref{thm:vc-hs} states that the VC dimension of the linear separators for $k$-dimensional points is $k+1$. Using these facts, Theorem~\ref{thm:met-samples}, and Theorem~\ref{thm:vc-real-lin}, we obtain the first task and sample complexities in the statement of the corollary.
Similarly, the second line of complexities is derived when we use Theorem~\ref{thm:met-real-samples} instead of Theorem~\ref{thm:met-samples}. 
\end{proof}

% \subsection{Proof of \Cref{{thm:vc-real-lin}}}

\ifnum \coltshort = 0
 Given a dataset $D =\{(\mathbf{\map}_i,\lab_i)\}_{i \in [\samples]}$ of $\samples$ points in $\R^{\repdim}\times \bits$, we define $\mathbf{Z}$ as the matrix whose $i$-th row is the $\repdim+1$-dimensional vector $\mathbf{z}_i' = \lab_i(\mathbf{\map}_i \| 1)$, for $i \in [\samples]$. 
 For $I\subseteq [\samples]$ a set of indices, we define $\mathbf{Z}_I$ to be the matrix whose rows are the vectors $\mathbf{z}_i'$ for $i$ in $I$.

To prove \Cref{{thm:vc-real-lin}}, we establish a set of conditions on a data set which allow us to check for separability. 
These conditions can be expressed via a limited number of low-degree polynomials, which allows us to apply Warren's Theorem.

Our first lemma in this section equates separability of a data set with the existence of a special subset of points.
If these points are ``on the margin'' of a linear separator defined by vector $\mathbf{a}$, then we know that $\mathbf{a}$ is a strict separator.
 \begin{lem}
    Dataset $D=\{(\mathbf{\map}_i,\lab_i)\}_{i \in [\samples]}$ is strictly separable if and only if there exists a subset $I$ of linearly independent rows of $\mathbf{Z}$ such that for all $\mathbf{a} \in \R^{\repdim+1}$ if 
    % for all $i \in I, \mathbf{\map}_i' \cdot\mathbf{a} = 1$, then for all $i \in [\samples]$, $\mathbf{\map}_i' \cdot \mathbf{a} \geq 1$.
     $\mathbf{\map}_i' \cdot\mathbf{a} = 1$ for all $i \in I$, then $\mathbf{\map}_i' \cdot \mathbf{a} \geq 1$ for all $i \in [\samples]$.
    \label{lem:strict-sep}
\end{lem}

\begin{proof}
    By definition, $D$ is strictly separable iff there exists a linear separator $\mathbf{a} \in \R^{\repdim+1}$ such that for all $i \in [\samples]$ we have $\lab_i(\mathbf{\map}_i \| 1)\cdot \mathbf{a} = \mathbf{\map}_i' \cdot \mathbf{a} \geq 1$. We define $A = \{\mathbf{a}\mid \mathbf{\map}_i' \cdot \mathbf{a} \geq 1, \forall i \in [\samples]\}$ as the set of all linear separators of $D$ and, for any $\mathbf{a}$, $I_{\mathbf{a}}$ as the set of tight constraints for $\mathbf{a}$, i.e., $I_{\mathbf{a}} = \{i\mid \mathbf{\map}_i' \cdot \mathbf{a}= 1\}$. 
    Note that $A$ and $I_{\mathbf{a}}$ might be empty.

    $\Leftarrow)$ 
    Assume that $I\subseteq [\samples]$ defines a subset of linearly independent rows such that for all $\mathbf{a} \in \R^{\repdim+1}$ if $\mathbf{\map}_i' \cdot\mathbf{a} = 1$ for all $i\in I$, then $\mathbf{\map}_i' \cdot \mathbf{a} \geq 1$ for all $i \in [\samples]$. 
    We find a vector $\mathbf{a}$ that makes the constraints in $I$ tight, that is, solve the linear system $\mathbf{Z}_I a = \mathbf{1}$, where $\mathbf{1}$ is the $|I|$-dimensional all-ones vector.
    Since $\mathbf{Z}_I$ has linearly independent rows, this system has at least one solution 
    $\hat{\mathbf{a}} = \mathbf{Z}_I^+ \mathbf{1}$, where $\mathbf{Z}_I^+ = \mathbf{Z}_I^T(\mathbf{Z}_I\mathbf{Z}_I^T)^{-1}$ is the pseudoinverse.
    Thus, for all $i\in I$ we have $\mathbf{\map}_i'\cdot \hat{\mathbf{a}}=1$, which implies $\mathbf{\map}_i' \cdot \hat{\mathbf{a}} \geq 1$ for all $i \in [\samples]$ by our assumption on $I$.
    Therefore, the points in $D$ are strictly separable. 
    
    $\Rightarrow$)
    Assume that $D$ is strictly separable.
    Then, by \Cref{lem:full_rank_separator}, there exists a strict separator $\mathbf{a}$ such that $\mathrm{rank}(\mathbf{Z}_{I_{\mathbf{a}}})=\mathrm{rank}(\mathbf{Z})$ and $\mathbf{Z} \cdot \mathbf{a} \ge \mathbf{1}$, where the inequality holds entry-wise.
    If needed, we can remove indices from $I_{\mathbf{a}}$ to produce $I$, a subset with the same rank but linearly independent rows.
    
    Now suppose that $\mathbf{a}'$ satisfies $\mathbf{\map}_i' \cdot \mathbf{a}' = 1$ for all $i\in I$.
    In other words, $\mathbf{a}'$ satisfies $\mathbf{Z}_{I} \cdot \mathbf{a}'= \mathbf{1}$.
    This means we can write $\mathbf{a}'=\mathbf{a}+\mathbf{u}$, where $\mathbf{u}$ is in the right nullspace of $\mathbf{Z}_{I}$.
    Since $\mathbf{Z}$ and $\mathbf{Z}_{I}$ share a rowspace, they also share a right nullspace.
    Thus $\mathbf{Z} \cdot \mathbf{a}' = \mathbf{Z} \cdot \mathbf{a} \ge \mathbf{1}$, where the inequality holds entry-wise.
    This completes the proof.
\end{proof}

The proof of \Cref{lem:strict-sep} uses the following lemma, which says that every strictly separable data set admits a separator whose set of tight constraints is full rank.
\begin{lem}\label{lem:full_rank_separator}
    For a data set $D=\{(\mathbf{\map}_i,y_i)\}_{i\in[\samples]}$, let $\mathbf{Z}$ be its associated matrix and, for a vector $\mathbf{a}$, let $I_{\mathbf{a}}\subseteq [\samples]$ be the set of tight constraints (i.e., the largest set $I$ such that $\mathbf{Z}_I \cdot \mathbf{a} = \mathbf{1}$).
    If $D$ is strictly separable, then there exists a vector $\mathbf{a}^*$ such that $\mathbf{\map}_i\cdot \mathbf{a}^*\ge 1$ for all $i\in [\samples]$ and $\mathrm{rank}(\mathbf{Z})=\mathrm{rank}(\mathbf{Z}_{I_{\mathbf{a}^*}})$.
\end{lem}
\begin{proof}
    \newcommand{\iao}{I_{\mathbf{a}_0}}
    \newcommand{\ziao}{\mathbf{Z}_{\iao}}

    $D$ is strictly separable, so by definition there exists a vector $\mathbf{a}_0$ such that $\mathbf{\map}_i \cdot \mathbf{a}_0 \ge 1$ for all $i\in [\samples]$.
    Suppose $\mathrm{rank}(\mathbf{Z}) > \mathrm{rank}(\ziao)$, since otherwise we are done.
     We will construct another separator $\mathbf{a}_1$ that satisfies $\mathrm{rank}(\ziao) < \mathrm{rank}(\mathbf{Z}_{I_{\mathbf{a}_1}})$.
     Since $\mathrm{rank}(\mathbf{Z})$ is finite, repeating this process will yield a separator $\mathbf{a}^*$ with $\mathrm{rank}(\mathbf{Z}_{I_{\mathbf{a}^*}}) = \mathrm{rank}(\mathbf{Z})$.
     
    If $\mathbf{a}$ is a solution to $\ziao \cdot \mathbf{a} = \mathbf{1}$, we can write it as $\mathbf{a}'=\mathbf{a}_0 + \mathbf{u}$, where $\mathbf{u}$ is in the right nullspace of $\ziao$.
    Since the rank of $\ziao$ is strictly less than the rank of $\mathbf{Z}$, there exists a vector $\mathbf{v}$ that lies in the right nullspace of $\ziao$ but \emph{not} in the right nullspace of $\mathbf{Z}$.
    Our separator $\mathbf{a}_1$ will be of the form $\mathbf{a}_0 + c\cdot \mathbf{v}$ for some real value $c$.
    Note that halfspaces of this form keep the constraints in $\iao$ tight: by construction we have $\mathbf{\map}_i' \cdot (\mathbf{a}_0 + c\cdot \mathbf{v}) = 1$ for all $i\in \iao$.

    Let $I^\perp\subseteq \bar{I}_{\mathbf{a}_0}$ be the subset of rows which are not in the right rowspace of $\ziao$.
    This set is nonempty, since $\mathrm{rank}(\ziao)<\mathrm{rank}(\mathbf{Z})$.
    For all $i\in I^\perp$, let $m_i(c) = \langle\mathbf{\map}_i',\mathbf{a}_0\rangle + c\cdot \langle \mathbf{\map}_i', \mathbf{u}\rangle$.
    Each of these is a linear function in $c$ and, by the definition of $\iao$ and the fact that $\mathbf{a}_0$ is a linear separator, we see that $m_i(0)>1$ for all $i\in I^\perp$.

     For each $i\in I^\perp$, there exists an interval $[L_i,R_i]$ such that, if $c\in [L_i,R_i]$, then $m_i(c)\ge 1$, i.e., point $i$ lies on the correct side of $\mathbf{a}_0 + c\cdot \mathbf{v}$.
    Because $m_i(0)>1$ for all $i\notin\iao$, we see that $L_i < 0 < R_i$.
     The intersection of these intervals, $[\max L_i, \min R_i]$, is nonempty.
     For any $c$ in the intersection, $\mathbf{a}_0 + c\cdot \mathbf{v}$ is a strict separator.
     (This holds for $i\in I^\perp$ by construction; if $i\notin I^\perp$ then $i$ is in the rowspace of $\ziao$ and we have $\mathbf{\map}_i' \cdot (\mathbf{a}_0 + c\cdot \mathbf{v}) = \mathbf{\map}_i'\cdot \mathbf{a}_0\ge 1$, as $\mathbf{v}$ is in the nullspace of $\ziao$.)

    It remains to select $c$ so that at least one additional constraint is tight. 
    This is easy: if $\max L_i$ is finite then we have $m_i (\max L_i)=1$.
    The same holds for $\min R_i$.
    Since we know at least one of $\max L_i$ and $\min R_i$ are finite, we can select a finite one as our value of $c^*$.

    We have constructed a vector $\mathbf{a}_1 = \mathbf{a}_0 + c^*\cdot \mathbf{v}$ that strictly separates $D$ and makes constraint $i^*$ tight for some $i^*\in I^\perp$, where $I^\perp$ is the set of constraints not in the rowspace of $\ziao$.
    Thus, $\mathrm{rank}(\ziao)<\mathrm{rank}(\mathbf{Z}_{I_{\mathbf{a}_1}})$.
    This concludes the proof.
\end{proof}

A simple corollary to \Cref{lem:strict-sep} says that, on separable data sets, the special subset of points allows us to identify a specific separator.
\begin{cor}
   Dataset $D=\{(\mathbf{\map}_i,\lab_i)\}_{i \in [\samples]}$ is strictly separable if and only if there exists a subset of points $I$ for which
\begin{enumerate}
    \item $\mathbf{Z}_I$ is full rank, and
    \item for all $i \in[\samples]$, we have that $\mathbf{\map}_i' \cdot \hat{\mathbf{a}} \geq 1$, where $\hat{\mathbf{a}} = \mathbf{Z}_I^{+}\mathbf{1}_{|I|}$.
\end{enumerate}
\label{cor:sep-cond}
\end{cor}
\begin{proof}
    $\Rightarrow$) 
    Assume $D$ is strictly separable.
    By \Cref{lem:strict-sep} we know that there exists a subset $I$ of linearly independent rows such that for all $\mathbf{a} \in \R^{\repdim+1}$ if  $\mathbf{\map}_i' \cdot\mathbf{a} = 1$ for all $i \in I$ then  $\mathbf{\map}_i' \cdot \mathbf{a} \geq 1$ for all $i \in [\samples]$. Since the rows in subset $I$ are linearly independent, $Z_I$ is full rank. 
    By construction, $\hat{\mathbf{a}}$ satisfies $Z_I \cdot \hat{\mathbf{a}} = \mathbf{1}_{|I|}$. 
    Thus $\mathbf{z}_i' \cdot \hat{\mathbf{a}}\geq 1$ for all $i\in[\samples]$.
    
    $\Leftarrow$) 
    Assume that such a subset $I$ exists.
    We see that $\hat{\mathbf{a}}$ strictly separates $D$.
\end{proof}

We now show how to express this characterization of separability in the language of polynomials. 
With this lemma in hand, the proof of \Cref{{thm:vc-real-lin}} will be a direct application of \Cref{lem:vc-bool-pol}, our extension of Warren's Theorem. 
\begin{lem}
   % Given a dataset $D = \{(\feat_i,\lab_i)\}_{i \in [\samples]}$, the $(\samples, \perclass_{\repdim})$-realizability predicate $r_h$ is a Boolean function of $(\samples+1) \cdot 2^{\samples}$ signs of polynomials in the variables of $\rep$ of degree $4(\repdim+1)$.
   Let $w=(\samples+1)\cdot 2^{\samples}$.
   For a data set $D=\{(\feat_i,\lab_i)\}_{i\in[\samples]}$, there exists a Boolean function $g:\{\pm 1\}^w\to\{\pm 1\}$ and a list of polynomials $p_D^1(h),\ldots, p_D^w(h)$, each of degree at most $4(k+1)$, such that we can express the $(\samples, \perclass_{\repdim})$-realizability predicate $r_h$ as
   \begin{align*}
       r_h(D) = g\left( \sign\left(p_D^{(1)}(h)\right),\ldots, \sign\left(p_D^{(w)}(h)\right) \right).
   \end{align*}
   \label{lem:realp-pol}
\end{lem}
\begin{proof}
    A representation $\rep$ induces a labeled dataset in the representation space $D_{\rep} = \{(\rep(\feat_i),\lab_i)\}_{i \in [\samples]}$.
    By \Cref{cor:sep-cond}, we can check whether $D_{\rep}$ is linearly separable by checking whether any of the $\sum_{i=1}^{\repdim+1}\binom{\samples}{i} \leq 2^{\samples}$ subsets of $D_{\rep}$ of size at most $\repdim+1$ satisfies the two conditions of the corollary. 
    In the remainder of the proof, we fix a subset $I$ and construct a Boolean function $g_I$ over signs of polynomials that checks if $I$ satisfies these conditions.
    This function will use at most $\samples+1$ polynomials, each of degree at most $4(k+1)$.
    The proof is finished by taking $g$ to be the OR of these functions for each subset $I$.

     We write $D_{\rep}$ as a matrix $\mathbf{Z}$ whose $i$-th row is the $\repdim+1$ dimensional vector $\mathbf{z}_i' = \lab_i (h(\feat_i)\|1)$. 
     % For every non-empty subset of at most $\repdim+1$ points $I$ 
     We construct a polynomial $p_I^{(0)}(h)$ that is \emph{negative} iff $Z_I$ is full rank (so that $\sign~p_I^{(0)}=+1$ indicates rank deficiency).
     For each $i\in[\samples]$, we construct a polynomial $p_I^{(i)}(h)$ that, when $Z_I$ is full rank, is nonnegative iff $\map_i'\cdot\hat{\mathbf{a}} \ge 1$,
     % for all $i \in[\samples]$, we have that $\mathbf{\map}_i' \cdot \hat{\mathbf{a}} \geq 1$, 
     where $\hat{\mathbf{a}} = \mathbf{Z}_I^{+}\mathbf{1}_{|I|}$.
     We then take $g_I$ to be
     \begin{align*}
         g_I(h) := \left(\lnot~\sign~p_I^{(0)}(h)\right) \wedge \left(\bigwedge_{i\in[\samples]} \sign ~ p_I^{(i)}(h)\right).
     \end{align*}
    Recall that we interpret $+1$ as logical ``true.''     

    Matrix $\mathbf{Z}_I$ is full rank if and only if $\mathrm{det}(\mathbf{Z}_I\mathbf{Z}_I^T) \neq 0$. Therefore, we set $p_I^{(0)}(h) = -\mathrm{det}(\mathbf{Z}_I\mathbf{Z}_I^T)^2$, which is a polynomial in $\repdim \times d$ variables of degree $4(\repdim+1)$. 
    To see this, observe that each entry in $\mathbf{Z}_I \mathbf{Z}_I^T$ is a polynomial of degree $2$ in the variables of $\rep$. Then, $\mathrm{det}(\mathbf{Z}_I \mathbf{Z}_I^T)$ is a polynomial of degree $2I$. Finally, $\mathrm{det}(\mathbf{Z}_I\mathbf{Z}_I^T)^2$ is a polynomial of degree $4I \leq 4(\repdim+1)$.
    Taking the negation gives us the desired polynomial $p_I^{(0)}$.

    We now check that $\map_i'\cdot \hat{\mathbf{a}} -1 \ge 0$ for a given $i$.
    Let $\Delta = \mathrm{det}(\mathbf{Z}_I\mathbf{Z}_I^T)$, a polynomial of degree $2(\repdim+1)$. 
    Then $\hat{\mathbf{a}} =  \mathbf{Z}_I^T(\mathbf{Z}_I\mathbf{Z}_I^T)^{-1}\mathbf{1}_{|I|} = \frac{\mathbf{Z}_I^T \mathrm{adj}(\mathbf{Z}_I\mathbf{Z}_I^T)\mathbf{1}_{|I|}}{\Delta}$, where $\mathrm{adj}$ denotes the adjugate matrix.
    We have $\map_i'\cdot \hat{\mathbf{a}} -1 \ge 0$ when $\Delta\neq 0$ and $\map_i'\cdot \Delta\hat{\mathbf{a}} - \Delta\ge 0$.
    Each entry in $\Delta \hat{\mathbf{a}}$ is a polynomial of degree $2|I|-1\le 2\repdim+1$, so setting $p_I^{(i)}(h)= \map_i'\cdot \Delta\hat{\mathbf{a}} - \Delta$ we have a polynomial of degree at most $2(\repdim+1)$ that is, when $\mathbf{Z}_I$ is full rank, is nonnegative iff the constraint is satisfied.
    % We will use $\hat{\mathbf{a}}_{\Delta} = \Delta\hat{\mathbf{a}}$, which is a polynomial of degree $2I-1\leq 2\repdim +1$. We check whether for every $i \notin I$, $\mathbf{z}_i'\cdot \hat{\mathbf{a}}_\Delta - \Delta \geq 0$.
    %
    % Thus, for every subset $I$ we check the conjunction of the signs of $m+1$ polynomials of degree $4(\repdim+1)$. We want to see if one of the subsets satisfies this condition. Therefore, we check the value of the OR function of the $2^{\samples}$ outputs. 
\end{proof}
% \gb{The above proof is pretty terse, and I vote to make it longer and easier to follow. But that's probably low-priority.}

\begin{proof}[Proof of \Cref{thm:vc-real-lin}]
    By \Cref{lem:realp-pol}, we have that the $(\samples, \perclass_{\repdim})$-realizability predicate is a Boolean function of $(\samples+1)\cdot 2^{\samples}$ signs of polynomials in $d\repdim$ variables of degree $4(\repdim+1)$. \Cref{lem:vc-bool-pol} gives us that the VC dimension of the class of realizability predicates $\realpclass_{\samples,\perclass_{\repdim},\repclass_{\datadim, \repdim}}$ is upper-bounded by $8d\repdim \log_2(4(\repdim+1) \samples 2^{\samples}) = O(d\repdim \samples+dk\log(\repdim \samples))$.
\end{proof}
\fi

\ifnum \coltshort = 0
\subsection{Pseudodimension of the Empirical Error Function Class}
\fi
We now turn to bounding the pseudodimension (Definition~\ref{def:pseudodimension}) of the class of empirical error functions (Definition~\ref{def:empirical_error_function}).\ifnum \coltshort =0 Recall our earlier statement: \else ~This bound on the pseudodimension immediately gives the statement about the task complexity of metalearning linear representations.\fi

\begin{thm}[Thm.~\ref{thm:pdim-acc-intro}, Restated]
    \thmpdimacc
  \label{thm:pdim-acc}
\end{thm}
\ifnum \coltshort=0
This bound on the pseudodimension immediately gives the following statement about the task complexity of metalearning linear representations.\fi
\begin{cor}
\ifnum\coltshort=0
    [Detailed version of Corollary~\ref{cor:met-lin-agn-intro}]
\fi
\label{cor:met-lin-agn}
    We can $(\varepsilon, \delta)$-properly metalearn $(\repclass_{\datadim,\repdim}, \perclass_\repdim)$ with $\users$ tasks and $n$ samples per task when 
    $\users = O\left(\frac{\datadim \repdim^2\ln(1/\varepsilon)+\datadim \repdim\ln(1/\varepsilon)^2}{\varepsilon^4} +\frac{\ln(1/\delta)}{\eps^2}\right) \quad\text{and}\quad \samples = O\left( \frac{\repdim+\ln(1/\varepsilon)}{\varepsilon^2}\right).$
\end{cor}
\ifnum \coltshort = 0
Our first step toward proving these results is the following lemma, which follows almost immediately from \Cref{cor:sep-cond}.
The proof uses the fact that a data set which can be classified with accuracy at least $\alpha$ can be classified perfectly if we change $1-\alpha$ labels. 
 \begin{lem}
    Dataset $D=\{(\mathbf{\map}_i,\lab_i)\}_{i \in [n]}$ can be linearly separated with accuracy $\alpha$ if and only if there exists a subset of points $I$ and a vector $\pmb{\sigma} \in \{\pm 1\}^{n}$ with $\sum_{i=1}^n \frac{\ind\{\sigma_i = +1\}}{n} = \alpha$ such that 
    \begin{enumerate}
        \item $\mathbf{Z}_I$ is full rank
        \item for all $i \in [n]$, we have that $(\mathbf{z}_i'\cdot \hat{\mathbf{a}}_{\pmb{\sigma}})\sigma_i\geq 1$, where $\hat{\mathbf{a}}_{\pmb{\sigma}} = Z_I^+\pmb{\sigma}_I$.
    \end{enumerate}
    \label{lem:acc-a}
\end{lem}
\begin{proof}
  % \gtext{By definition, a} dataset can be classified with accuracy $\alpha$ if  separate it after we flip $\alpha$ fraction of the labels. 
  By definition, dataset $D$ can be linearly separated with accuracy $\alpha$ if and only if there exists a vector $\pmb{\sigma} \in \{\pm 1\}^{n}$ with $\sum_{i=1}^n \frac{\ind\{\sigma_i = 1\}}{n} = \alpha$ such that dataset $D_{\sigma} = \{(\mathbf{\map}_i,\sigma_i\lab_i)\}_{i \in [n]}$ is strictly separable. Now, we can apply \Cref{cor:sep-cond} to dataset $D_{\sigma}$. Let $\hat{\mathbf{Z}}$ be the matrix whose $i$-th row is the vector $\hat{\mathbf{z}}'_i = \sigma_i y_i(\mathbf{z}_i\|1)$. 
  The rank of $\hat{\mathbf{Z}}_I$ remains the same as the rank of $\mathbf{Z}_I$ since every row is a scalar multiple of the corresponding row in $Z_I$. % by the corresponding $\sigma_i$. 
  Additionally, observe that $Z_I^+\pmb{\sigma}_I = \hat{\mathbf{Z}}_I^+ \mathbf{1}_{|I|}$.
  % Additionally, solving the linear system $\hat{\mathbf{Z}}_I \mathbf{a} = \mathbf{1}_{|I|}$ is equivalent to solving $\mathbf{Z}_I \mathbf{a} = \pmb{\sigma}_I$ because for all $i$ we have $\sigma_i^2 = 1$. This concludes our proof.
\end{proof}

The next lemma shows how to express statements about the empirical error function in the language of low-degree polynomials.
\begin{lem}
   Given a dataset $D = \{(\feat_i,\lab_i)\}_{i \in [n]}$ and an $\alpha \in [0,1]$, the predicate $\ind_{\pm}\{q_h(x_1,y_1,\ldots,$ $x_n,y_n) =\alpha\}$ is a Boolean function of $(n+1) \cdot 2^{2n}$ signs of polynomials in the variables of $\rep$ of degree $4(\repdim+1)$.
   \label{lem:empacc-pol}
\end{lem}
\begin{proof}
By \Cref{lem:acc-a} we have seen that $\ind_{\pm}\{q_h(x_1,y_1,\ldots, x_n,y_n)\ktext{=}
%\geq 
\alpha\} = \ind_{\pm}\{\exists \pmb{\sigma} \in \{\pm 1\}^{n} $ with $\sum_{i=1}^n \ind\{\sigma_i = +1\} = \alpha n \textrm{ s.t. } r_h(\feat_1,\sigma_1 y_1, \ldots, \feat_n, \sigma_n y_n) = +1\}$. Therefore, by \Cref{lem:realp-pol} for every value of $\pmb{\sigma}$ that has $\alpha$ fraction of ones we can check a Boolean function of $(n+1)\cdot 2^n$ signs of polynomials in the variables of $h$ of degree $4(k+1)$. 
There are no more than $2^n$ such vectors $\mathbf{\sigma}$.
In total, we need to evaluate $(n+1)\cdot 2^{2n}$ signs of polynomials.
\end{proof}

We are now ready to prove the final result in this section.
\begin{proof}[Proof of \Cref{thm:pdim-acc}]
% Let $Q_{n, \perclass, \repclass, \alpha}=\{g \mid g(\feat_1,\lab_1,\ldots,\feat_n, \lab_n) = \ind\{q_h(\feat_1, \lab_1, \ldots, \feat_n, \lab_n) = \alpha\}\}$. By \Cref{lem:empacc-pol} and \Cref{lem:vc-bool-pol} we know that for a fixed $\alpha \in [0,1]$ the VC dimension of class $Q_{n, \perclass, \repclass, \alpha}$ is $O(dkn +dk\log_2(nk))$. We notice that $\alpha$ can take $n+1$ different values. Thus, the pseudodimension of $Q_{n, \perclass, \repclass}$ is $O(dkn^2 +dkn\log_2(nk))$. 

    A well-known fact about pseudodimension (see, e.g.,~\cite{AnthonyB1999}) is that it equals the VC dimension of the class of subgraphs. 
    As a function $q_h\in Q_{n,\perclass,\repclass}$ maps data sets $D=\{(x_i,y_i)\}_{i\in[n]}$ to the interval $[0,1]$ (corresponding to error), the object $(D,\tau)$ lies in the subgraph of $q_h$ if $q_h(D) \ge \tau$.
    In our case,
    \begin{align}
        \mathrm{PDim}(Q_{n,\perclass,\repclass}) = \mathrm{VC}\left( \{ \ind_{\pm}\{q_h(D)\ge \tau\} \mid q_h \in Q_{n,\perclass,\repclass}\} \right).
    \end{align}
    To use the tools we have previously established, we will show that, for a fixed $(D,\tau)$, we can write the indicator $\ind_{\pm}\{q_h(D) \ge \tau\}$ as a Boolean function of signs of polynomials.

    But this is easy: the indicator is +1 exactly when there exists an $\alpha\ge \tau$ such that $q_h(D) = \alpha$, which we analyzed in \Cref{lem:empacc-pol}.
    We take an OR over the $\samples+1$ possible values of $\alpha \in \{0,1/\samples ,\ldots, 1\}$:
    \begin{align*}
        \ind_{\pm}\{q_h(D)\ge \tau\} 
            = \bigwedge_{\alpha}\bigl( \ind_{\pm}\{q_h(D) = \alpha~\text{and}~\alpha\ge \tau\}\bigr).
    \end{align*}
    (Recall that we interpret $+1$ as logical ``true.'')
    When $\alpha <\tau$, we can represent this as $\sign(-1)$, a degree-0 polynomial.
    When $\alpha\ge \tau$, we apply \Cref{lem:empacc-pol} to see that each term can be written as a Boolean function of $(n+1)\cdot 2^{2n}$ signs of polynomials in the variables of $h$ of degree $4(k+1)$.
    Together, we see that $\ind_{\pm}\{q_h(D)\ge\tau\}$ can be written as a Boolean function of at most $(n+1)^2 2^{2n}$ signs of polynomials, each of (at most) the same degree.
    By \Cref{lem:vc-bool-pol}, the VC dimension (and thus the pseudodimension of $Q_{n,\perclass,\repclass}$) is at most $8 dk \log_2 ( 4(k+1) \cdot (n+1)^2 2^{2n}) =O(dkn + dk\log (kn))$.
\end{proof}

\fi 

\addcontentsline{toc}{section}{References}%
\bibliographystyle{sty/alpha-no-nonsense.bst}%
%\newpage
\bibliography{refs}

\appendix
%\crefalias{section}{appendix} 
\crefalias{section}{appendix}

\ifnum \coltshort = 1
\section{Notation and Background}
\label{app:back}
In this paper we use bold lowercase letters for vectors, e.g. $\mathbf{a}$, and bold uppercase letters for matrices, e.g. $\mathbf{A}$. Given a $\datadim$-dimensional vector $\mathbf{a}$ and a value $b$, $(\mathbf{a}\|b)$ is a $\datadim+1$-dimensional vector, where we have appended value $b$ to vector $\mathbf{a}$. For distributions over labeled data points we use the regular math font, e.g. $D$, whereas for metadistributions (distributions over task data distributions) we use fraktur letters, e.g. $\mathfrak{D}$. We define the indicator function of a predicate $p$ as
\begin{align*}
    \ind\{p\} = \begin{cases}
        1 & \text{if }p\text{ is true} \\
        0 & \text{o.w.}
    \end{cases}.
\end{align*}
Occasionally, we use the alternative indicator $\ind_{\pm}\{p\}$, which takes value $+1$ if $p$ is true and $-1$ otherwise.
We also use the following sign function 
\begin{align*}
    \sgn(x) = \begin{cases}
        +1 & \text{if }x \geq 0 \\
        -1 & \text{o.w.}
    \end{cases}.
\end{align*}
Contrary to the convention in Boolean analysis, when necessary we interpret $+1$ as logical ``true'' and $-1$ as logical ``false.''

\mypar{VC dimension.} We start by recalling the standard definition of VC dimension and the Sauer-Shelah Lemma.
\begin{defn}\label{def:vc_dimension}
    Let $\mathcal{F}$ be a set of functions mapping from a domain $\mathcal{X}$ to $\{\pm 1\}$ and suppose that $X=(x_1,\ldots,x_n)\subseteq \mathcal{X}$.
    We say that $\mathcal{F}$ \emph{shatters} $X$ if, for all $b\in \{\pm 1\}^n$, there is a function $f_b\in \mathcal{F}$ with $f_b(x_i)=b_i$ for each $i\in[n]$.  
    
    The \emph{VC dimension} of $\mathcal{F}$, denoted $\mathrm{VC}(\mathcal{F})$, is the size of the largest set $X$ that is shattered by $\mathcal{F}$.
\end{defn}

VC dimension is closely related to the \emph{growth function}, which bounds the number of distinct labelings a hypothesis class can produce on any fixed data set.
\begin{defn}
    Let $\perclass$ be a class of functions $\per: \mathcal{Z}\to\{\pm 1\}$ and $S = \{z_1, \ldots,z_n\}$ be a set of points in $\mathcal{Z}$. The \emph{restriction} of $\perclass$ to $S$ is the set of functions 
    $$\perclass_S = \{(\per(z_1), \ldots, \per(z_n))\mid\per \in \perclass\}.$$
    The \emph{growth function} of $\perclass$, denoted $\growth_\perclass(n)$, is 
    $
        \growth_\perclass(n) := \sup_{S: |S|=n} \left|\perclass_S \right|.
    $
\end{defn}

The Sauer-Shelah Lemma, which bounds the growth function in terms of the VC dimension.
\begin{lem}%[\cite{sauer1972,shelah1972}]
    If $\VC(\perclass) = \vc$, then for every $\samples > \vc$, $\growth_{\perclass}(\samples) \leq (e\samples/\vc)^\vc.$
    \label{lem:sauer}
\end{lem}

\mypar{VC dimension and PAC learning.}  Now we recall the relationship between VC dimension and the sample complexity of distribution-free PAC learning.  Here we refer to the textbook notion of PAC learning without giving a formal definition. 
\begin{thm}
\label{fact:pac-vc}
    Let $\class$ be a hypothesis class of functions $f: \featdom \to \bits$ with $\VC(\class) = \vc < \infty$, then for every $\eps,\delta \in (0,1)$
    \begin{enumerate}
        \item $\class$ has uniform convergence with sample complexity $O(\frac{\vc+\ln(1/\delta)}{\varepsilon^2})$
        \item $\class$ is agnostic PAC learnable with sample complexity $O(\frac{\vc+\ln(1/\delta)}{\varepsilon^2})$
        \item $\class$ is PAC learnable with sample complexity $O(\frac{\vc\ln(1/\varepsilon)+\ln(1/\delta)}{\varepsilon})$.
    \end{enumerate}
\end{thm} 
While uniform convergence requires $O(1/\eps^2)$ samples, with just $O(\ln(1/\eps)/\eps)$ samples we will have the property that every hypothesis that has error $\eps$ on the distribution has non-zero error on the sample.
\begin{thm}
    Let $\class$ be a hypothesis class with VC dimension $\vc$. Let $D$ be a probability distribution over $\featdom\times \bits$. For any $\varepsilon,\delta > 0$ if we draw a sample $S$ from $D$ of size $n$ satisfying 
    \begin{align*}
        n \geq \frac{8}{\varepsilon}\left(\vc \ln{\left(\frac{16}{\varepsilon}\right)+\ln{\left(\frac{2}{\delta}\right)}}\right)
    \end{align*}
    then with probability at least $1-\delta$, all hypotheses $\per$ in $\perclass$ with $\er(D,f)>\varepsilon$ have $\er(S,f)>0$.
    \label{fact: vc-gen}
\end{thm}

\mypar{VC dimension of halfspaces.} For a significant part of this paper, we work with linear classifiers of the form $$\perclass_{\datadim} = \left\{\per \left| \per(\textbf{x}) = \sign(\mathbf{a}\cdot \mathbf{x}-w), \mathbf{a} \in \R^\datadim, w \in \R \right.\right\}.$$ We also consider the class of linear classifiers that pass through the origin $$\tilde{\perclass}_{\datadim} = \left\{\per \left| \per(\textbf{x}) = \sign(\mathbf{a}\cdot \mathbf{x}), \mathbf{a} \in \R^\datadim \right.\right\}.$$

\begin{thm}
    \label{thm:vc-hs}
    We have $\VC(\perclass_{\datadim}) = \datadim+1$ and $\VC(\tilde{\perclass}_{\datadim}) = \datadim$.
\end{thm}

\mypar{Pseudodimension.} For real-valued functions we use generalization bounds based on a generalization of VC dimension called the \emph{pseudodimension.}
\begin{defn}\label{def:pseudodimension}
    Let $\mathcal{F}$ be a set of functions mapping from a domain $\mathcal{X}$ to $\mathbb{R}$ and suppose that $X = (x_1,\ldots,x_n) \subseteq \mathcal{X}$.
    We say that $X$ is \emph{pseudoshattered} by $\mathcal{F}$ if there are real numbers $r_1,\ldots,r_n$ such that for each $b\in \{\pm 1\}^n$ there is a function $f_b\in\mathcal{F}$ with $\sign(f_b(x_i)-r_i)=b_i$ for each $i\in [n]$.
    
    The \emph{pseudodimension} of $\mathcal{F}$, denoted $\mathrm{Pdim}(\mathcal{F})$, is the size of the largest set $X$ that is pseudoshattered by $\mathcal{F}$.
\end{defn}

\begin{thm}%[\cite{AnthonyB1999}]
\label{fact:pac-pd}
    Let $\class$ be a hypothesis class of functions $f: \featdom \to \R$ with $\textrm{PDim}(\class) = \pseudod < \infty$, then $\class$ has uniform convergence with sample complexity $O(\frac{\pseudod\ln(1/\varepsilon)+\ln(1/\delta)}{\varepsilon^2})$.
\end{thm}
\fi

\ifnum \coltshort = 0
\section{Proofs from Sections \ref{sec:models},  \ref{sec:mtl} and \ref{sec:lin}}

\subsection{Proof of \texorpdfstring{Lemma~\ref{lem:balls-and-bins}}{Lg}}
\label{sec:balls-and-bins-app}
\begin{customlem}{\ref*{lem:balls-and-bins}}[Restated]
Suppose we draw $t$ samples from a uniform distribution over $[t]$. Let $F_i$ denote the frequency of each $i \in [t]$ in the sample set. With probability $1-\delta$, $\max_i F_i$ is at most $\frac{\ln ((t/\delta)^2) }{\ln \ln ((t/\delta)^2)}$
\end{customlem}
\begin{proof}
Using the Chernoff bound for any $\epsilon > 0$, given the probability that a sample becomes equal to $i$ is $1/t$, we can bound the probability that the number of samples $i$ being too large as follows:
        \begin{align*}
            & \Pr{\frac{F_i}{t} \geq (1 + \epsilon) \cdot \frac{1}{t}} \leq \left(\frac{e^\epsilon}{(1 + \epsilon)^{(1 + \epsilon)}}\right)^{(\frac{m}{n})} 
            = \left(\frac{e^\epsilon}{(1 + \epsilon)^{(1 + \epsilon)}}\right) 
            \\
            & \quaaad 
            = \left(\frac{1}{1 + \epsilon}\right)\left(\frac{e^\epsilon}{(1 + \epsilon)^\epsilon}\right)  
            %= \left(\frac{1}{1 + \epsilon}\right) e^{\epsilon(1 - \ln(1 + \epsilon))}
            \leq \exp
            \left(- \ln(1+\epsilon) + \epsilon - \epsilon \ln(1+\epsilon) \right)
            \\
            & \quaaad 
            \leq \exp\left(\frac{-\epsilon\, \ln \epsilon}{2} \right)
            \,.
        \end{align*} The last inequality above holds for any $\epsilon \geq 1$. Now, set $\epsilon$ as follows: 
        $$\epsilon \coloneqq \frac{\ln ((t/\delta)^2) }{\ln \ln ((t/\delta)^2)}\,.$$

        It is not hard to see that: 

        \begin{align*}
            \epsilon \cdot \ln \epsilon & = \frac{\ln  ((t/\delta)^2) }{\ln ((\ln  ((t/\delta)^2))^2)} \cdot \left( \ln \ln  ((t/\delta)^2) - \ln \ln ((\ln  ((t/\delta)^2))^2)\right) \\ & \leq \ln  ((t/\delta)^2) \cdot \left(\frac{\ln \ln  ((t/\delta)^2) }{2 \ln \ln  ((t/\delta)^2)}\right) < \ln  ((t/\delta)^2)\,.
        \end{align*}

        Therefore, we can conclude that:
        \begin{align*}
            \Pr{\frac{F_i}{t} \geq (1 + \epsilon) \cdot \frac{1}{t}} & \leq \exp\left(-\frac{\epsilon\ln\epsilon}{2}\right) \leq \exp\left(-\frac{\ln  ((t/\delta)^2)}{2}\right) \leq \frac{\delta}{t}\,.
        \end{align*}

        Thus, for a fixed $i$, we have shown that the probability that $F_i$ is greater than $\epsilon$ is bounded from above by  $\frac{\delta}{t}$. By the union bound, the probability that any frequency is greater than $\epsilon$ is at most:
        \begin{align*}
            & \Pr{\exists i : F_i > \epsilon} \leq t \cdot \sum_{i = 1}^{t} \Pr{\frac{F_i}{t} \geq \frac{(1+\epsilon)}{t}} \leq t \frac{\delta}{t} = \delta\,.
        \end{align*}
Hence, the proof is complete. 
\end{proof}

\subsection{Proof of \texorpdfstring{\Cref{thm:mtl-smpls}}{Lg}} \label{sec:mtl-smples-proof}
\begin{customthm}{\ref*{thm:mtl-smpls}}[Restated]
\thmmtlsmpls
\label{thm:mtl-smpls-app}
\begin{comment}
    For every $(\repclass,\perclass)$, and every $\eps,\delta > 0$, and every $\users, \samples$,
    \begin{enumerate}
        \item If $\samples \users \geq O(\frac{\VC(\perclass^{\otimes \users}\circ \repclass) \cdot \ln(1/\varepsilon)+\ln\left(1/\delta\right)}{\varepsilon})$, then $(\repclass, \perclass)$ is $(\eps,\delta)$-multitask learnable in the realizable case with $\users$ tasks and $\samples$ samples per task.

        \item If $\samples \users \geq O(\frac{\VC(\perclass^{\otimes \users}\circ \repclass) \cdot \ln\left(1/\delta\right)}{\varepsilon^2})$, then $(\repclass, \perclass)$ is $(\eps,\delta)$-multitask learnable with $\users$ tasks and $\samples$ samples per task.

        \item If $ \samples \users \leq \frac14 \cdot \VC(\perclass^{\otimes \users}\circ \repclass)$ then $(\repclass,\perclass)$ is \emph{not} $(1/8,1/8)$-multitask learnable with $\users$ tasks and $\samples$ samples per task, even in the realizable case.
    \end{enumerate}
    \end{comment}
\end{customthm}

\begin{proof}
We’ll start by proving part 1, which covers the realizable case.  We will omit the proof of part 2, which generalizes the realizable case in a standard way. Finally, we will prove part 3.

\medskip\emph{Proof of 1.}
For every task $j \in [\users]$ we have $\samples$ i.i.d.\ samples $S_j = \{(\feat_i^{(j)}, \lab_i^{(j)})\}_{i \in [\samples]}$ drawn from $\dist_j$. Our dataset is equivalent to dataset $S = \{(j, \feat_i^{(j)}, \lab_i^{(j)})\}_{i \in [\samples], j \in [\users]}$, where the $j$s have fixed values. In standard PAC learning we assume that all the datapoints are i.i.d. However, in our case the samples are independent but not identically distributed because different tasks have (potentially) different distributions.

As in standard PAC learning, our proof follows the ``double-sampling trick''. We want to bound the probability of bad event 
\begin{align*}
    B: \exists g \in \perclass^{\otimes}\circ\repclass \textrm{ s.t. } \er(S, g)=0,\textrm{ but }\frac{1}{\users} \sum_{j \in [\users]} \er(\dist_j, g(j,\cdot)) > \varepsilon.
\end{align*}

We consider an auxiliary dataset $\hat{S} = \{(j, \hat{\feat}_i^{(j)}, \hat{\lab}_i^{(j)})\}_{i \in [\samples], j \in [\users]}$, where $(\hat{\feat}_i^{(j)},  \hat{\lab}_i^{(j)})$ are drawn independently from $\dist_j$. In our proof we will bound the probability of $B$ by bounding the probability of event 
\begin{align*}
    B': \exists g \in \perclass^{\otimes}\circ\repclass \textrm{ s.t. } \er(S, g)=0,\textrm{ and }  \er(\hat{S}, g) > \frac{ \varepsilon}{2}.
\end{align*}

We can show that if $\samples \users > \frac{8}{\varepsilon}$, then $\Pr[S]{B} \leq 2 \Pr[S, \hat{S}]{B'}$. We can show this by applying a multiplicative Chernoff bound on $\er(\hat{S},g)$, which is an average of independent random variables. Due to this step, it suffices to bound $\Pr[S, \hat{S}]{B'}$.

We also define a third event $B''$ as follows. We give $S$ and $\hat{S}$ as inputs to randomized process \textit{Swap} which iterates over $j \in [\users]$ and $i \in [\samples]$ and at every step it swaps $(\feat_i^{(j)}, \lab_i^{(j)})$ with $(\hat{\feat}_i^{(j)}, \hat{\lab}_i^{(j)})$ with probability $1/2$. Let $T$ and $\hat{T}$ be the two datasets this process outputs. We define event $$B'': \exists g \in \perclass^{\otimes \users}\circ \repclass \textrm{ s.t. } \er(T,g) = 0 \textrm{ and } \er(\hat{T}, g) > \frac{\varepsilon}{2}.$$

We see that $\Pr[S, \hat{S}, \textit{Swap}]{B''} = \Pr[S, \hat{S}]{B'}$. This happens because $T$, $\hat{T}$, $S$ and $\hat{S}$ are identically distributed. Thus, what we need to do now is bound $\Pr[S, \hat{S}, \textit{Swap}]{B''}$. 

 We start by showing that for a fixed $g$ $$\Pr[\textit{Swap}]{ \er(T,g) = 0 \textrm{ and } \er(\hat{T},g) > \frac{\varepsilon}{2}\mid S, \hat{S}} \leq 2^{-nt\varepsilon/2}.$$ Given $S$ and $\hat{S}$, $B''$ happens if for every $j \in [\users]$ and $i \in [\samples]$ $g$ predicts the label of $\feat_i^{(j)}$ or $\hat{\feat}_i^{(j)}$ correctly and makes $m > \varepsilon \samples \users/2$ mistakes overall. Additionally, all $m$ mistakes $g$ makes are in dataset $\hat{T}$. This means that \textit{Swap} assigns all these points to $\hat{T}$, which happens with probability $1/2^m \leq 1/2^{\varepsilon \samples \users/2}$. 

 Let $\left(\perclass^{\otimes \users}\circ \repclass\right) (S \cup \hat{S}) \subset \perclass^{\otimes \users}\circ \repclass$ be a set of hypotheses which contains one hypothesis for every labeling of $S\cup\hat{S}$. Then, 

\begin{align*}
    \Pr[S, \hat{S}, \textit{Swap}]{B''}& = \Exp_{S,\hat{S}}\left[\Pr[\textit{Swap}]{ \exists g \in  \perclass^{\otimes \users}\circ \repclass \textrm{ s.t. }\er(T,g) = 0 \textrm{ and } \er(\hat{T},g) > \frac{\varepsilon}{2}\mid S, \hat{S}}\right]\\
    &\leq \Exp_{S,\hat{S}}\left[\sum_{g \in\perclass^{\otimes \users}\circ \repclass (S \cup \hat{S}) }\Pr[\textit{Swap}]{ \er(T,g) = 0 \textrm{ and } \er(\hat{T},g) > \frac{\varepsilon}{2}\mid S, \hat{S}}\right]\\
    & \leq \growth_{\perclass^{\otimes \users}\circ \repclass}(2\samples \users) 2^{-\varepsilon \samples \users/2}.
\end{align*}

For the probability of bad event $B$ happening to be at most $\delta$, by the steps above we need $\samples \users \geq 2 \frac{\log_2 \growth_{\perclass^{\otimes \users}\circ \repclass}(2\samples \users) + \log_2(2/\delta)}{\varepsilon}$ samples in total. 

By Lemma \ref{lem:sauer},  for $\samples\users > \VC(\perclass^{\otimes \users}\circ \repclass)$,
$$\log_2\paren{\growth_{\perclass^{\otimes \users}\circ \repclass}(2nt)}
\leq 
\VC(\perclass^{\otimes \users}\circ \repclass)
\log_2\left(\frac{e 2\samples \users}{\VC(\perclass^{\otimes \users}\circ \repclass)}\right)
.$$
One can show that if $\samples \users \geq \frac{\VC(\perclass^{\otimes \users}\circ \repclass)}{\varepsilon}\log_2\left(\frac{2e}{\varepsilon}\right)$, then $\samples \users \geq \log_2\left(\frac{e 2\samples \users}{\VC(\perclass^{\otimes \users}\circ \repclass)}\right)\frac{\VC(\perclass^{\otimes \users}\circ \repclass)}{\varepsilon}.$ Therefore, we get that for $\samples \users \geq \frac{\VC(\perclass^{\otimes \users}\circ \repclass)\log_2(2e/\varepsilon)+\log_2(2/\delta)}{\varepsilon}$ samples in total the probability of bad event $B$ is at most $\delta$. 

\medskip\emph{Proof of 3.}
    Our goal is to construct distributions $\dist_1,\ldots, \dist_{\users}$ over $\featdom \times \{\pm 1\}$. We will first build their support and then define the probability distributions. 
    
    Let $\vc = \VC(\perclass^{\otimes \users}\circ \repclass)$. Since
    $4\samples \users \leq \vc$, there exists a dataset $S = \{(j_i,x_i)\}_{i \in [4\samples \users]}$ that can be shattered by $\perclass^{\otimes \users}\circ \repclass$. 
    Let $S_j = \{(j_i,x_i) \in S\mid j_i = j\}$. In general for every task $j$ the size of $S_j$, i.e.\ $|S_j|$, will be different. 
    We want to use $S$ to get a new dataset $S'$ that can be shattered by $\perclass^{\otimes \users}\circ \repclass$, but also has at least $2\samples$  
    points for every task. 
    For every $j \in [\users]$ we throw away points from $S_j$ until we have a multiple of $2\samples$. After doing this, we have thrown away at most $2\samples$ per task, which is at most $2\samples \users$ points in total. We can redistribute points so that we have at least $2\samples$ points per task as follows. For every task $j \in [\users]$, if there are no points in this task, there must be another task $k$ with at least $4\samples \users$ points. In this case, we move $\samples \users$ points from task $k$ to task $j$ by replacing $k$ with $j$ in $(k,x_i)$. 
    We denote
    this new dataset by $S'$ and the subset for task $j$, by $S_j'$. 
    
    We claim that $S'$
    , which has at least $2\samples$ points per task, can also be shattered by $\perclass^{\otimes \users}\circ \repclass$. To see this, suppose that $S$ was shattered by $\tilde{g}$. Now, if task $j$ did not lose all its points, the remaining points can be shattered by $\tilde{g}(j,\cdot)$. Otherwise, we are in the case where $j$'s points were initially assigned to task $k$, so they can be shattered by $\tilde{g}(k,\cdot)$.

    Let $\paren{\perclass^{\otimes \users}\circ \repclass}(S')$ be a set of hypotheses that contains one function for each labeling of $S'$. We choose a labeling function $g$ uniformly at random from $\paren{\perclass^{\otimes \users}\circ \repclass}(S')$. For all tasks $j \in [\users]$ we define $\dist_j$ 
    as the distribution of $(x,y)$ obtained by sampling $x$ uniformly from $S'_j$ (ignoring the task index in the sample) and labeling it according to $g$.

    Suppose that the (potentially randomized) learning algorithm $\alg$ returns $\hat{g}$ after seeing datasets $\hat{S}_1, \ldots, \hat{S}_\users$, where every $\hat{S}_j$ has $\samples$ points drawn i.i.d.\ from $\dist_j$. For a fixed task $j$, the probability that $\hat{g}$ makes a mistake on a new point $(x,y)$ drawn from $\dist_j$ is 
    \begin{align*}
        \Pr[g,\hat{S}_j, (x,y) \sim \dist_j]{\hat{g}(j,x) \neq y} & \geq \Pr[g,\hat{S}_j, (x,y)\sim \dist_j]{\hat{g}(j,x) \neq y \textrm{ and } x \notin \hat{S}_j}\\
        & = \Pr[g,\hat{S}_j, (x,y)\sim \dist_j]{x \notin \hat{S}}\Pr[g,\hat{S}, (x,y)\sim \dist_j]{\hat{g}(j,x) \neq y\mid x \notin \hat{S}_j}.
    \end{align*}
    
    We know that $\Pr[\hat{S}_j, (x,y)\sim \dist_j]{x \notin \hat{S}}\geq 1/2$ because $x$ is chosen uniformly at random out of more than $2\samples$ points that are in $S_j'$ and $\hat{S}_j$ has only $\samples$
    points. When $x \notin \hat{S}_j$ the algorithm has not seen the label that corresponds to this $x$ and, thus, $y=g(j,x)$ is independent of $\hat{g}(j,x)$. We have picked $g$ uniformly at random from a class with exactly one function per labeling, which means that for each $(j,x)$ we see $+1$ and $-1$ with equal probability. Therefore, $\Pr[g,\hat{S}_j, (x,y)\sim \dist_j]{\hat{g}(j,x) \neq y\mid x \notin \hat{S}_j} = 1/2$. Thus, $\Pr[g,\hat{S}_j, (x,y) \sim \dist_j]{\hat{g}(j,x) \neq y} \geq 1/4$.

    The average error is 
    \begin{align*}
        \frac{1}{\users} \sum_{j \in [\users]} \Pr[g,\hat{S}_j, (x,y) \sim \dist_j]{\hat{g}(j,x) \neq y} \geq \frac{1}{4}.
    \end{align*}
    In expectation over the labeling functions $g$ and the randomness of algorithm $\alg$, we have 
    \begin{align*}
        \Exp_{g, \alg}\left[\frac{1}{\users} \sum_{j \in [\users]} \Pr[\hat{S}_j, (x,y) \sim \dist_j]{\hat{g}(j,x) \neq y}\right] \geq \frac{1}{4}
    \end{align*}
    Hence, there exists a labeling function $g$ in $\perclass^{\otimes \users}\circ \repclass(S')$ such that $ \Exp_{\alg}\left[\frac{1}{\users} \sum_{j \in [\users]} \Pr[\hat{S}_j, (x,y) \sim \dist_j]{\hat{g}(j,x) \neq y}\right] \geq \frac{1}{4}$. For this $g$ we have that 
    \begin{align*}
        &\Exp_{\alg}\left[\frac{1}{\users} \sum_{j \in [\users]} \Pr[\hat{S}_j, (x,y) \sim \dist_j]{\hat{g}(j,x) \neq y}\right]\\
        &= \Exp_{\alg}\left[\frac{1}{\users} \sum_{j \in [\users]} \Exp_{\hat{S}_j}[\er(\dist_j,\hat{g}(j,\cdot))]\right]\\
        & = \Exp_{\alg,\hat{S}_1, \ldots, \hat{S}_\users}\left[\frac{1}{\users} \sum_{j \in [\users]}\er(\dist_j,\hat{g}(j,\cdot))\right]\\
        & = \Pr[\alg,\hat{S}_1, \ldots, \hat{S}_\users]{\frac{1}{\users} \sum_{j \in [\users]}\er(\dist_j,\hat{g}(j,\cdot))>\frac{1}{8}}\Exp_{\substack{\alg,\\\hat{S}_1, \ldots, \hat{S}_\users}}\left[\frac{1}{\users} \sum_{j \in [\users]}\er(\dist_j,\hat{g}(j,\cdot)) \left| \sum_{j \in [\users]}\er(\dist_j,\hat{g}(j,\cdot)) > \frac{1}{8}\right.\right]\\
        &+\Pr[\alg,\hat{S}_1, \ldots, \hat{S}_\users]{\frac{1}{\users} \sum_{j \in [\users]}\er(\dist_j,\hat{g}(j,\cdot))<\frac{1}{8}}\Exp_{\substack{\alg,\\\hat{S}_1, \ldots, \hat{S}_\users}}\left[\frac{1}{\users} \sum_{j \in [\users]}\er(\dist_j,\hat{g}(j,\cdot)) \left| \sum_{j \in [\users]}\er(\dist_j,\hat{g}(j,\cdot)) < \frac{1}{8}\right.\right]\\
        &\leq \Pr[\alg,\hat{S}_1, \ldots, \hat{S}_\users]{\frac{1}{\users} \sum_{j \in [\users]}\er(\dist_j,\hat{g}(j,\cdot))>\frac{1}{8}} + \frac{1}{8}
    \end{align*}
    Thus, we showed that there exist $\dist_1, \ldots, \dist_\users$ such that $\Pr[\alg,\hat{S}_1, \ldots, \hat{S}_\users]{\frac{1}{\users} \sum_{j \in [\users]}\er(\dist_j,\hat{g}(j,\cdot))>\frac{1}{8}} \geq \frac{1}{8}$.
\end{proof}

\subsection{Proof of  \texorpdfstring{\Cref{lem:vc-bounds}}{Lg}} \label{sec:vc-bounds-proof}
\begin{customlem}{\ref*{lem:vc-bounds}}[Restated]
\lemvcbounds
\label{lem:vc-bounds-app}
\begin{comment}
For any representation class $\repclass$ that contains a surjective function, any personalization class $\perclass$, and any $\users$,
\[
    \max \left\{\users \cdot \VC(\perclass), \VC(\perclass \circ \repclass) \right\} \leq \VC(\perclass^{\otimes \users}\circ \repclass) \leq \users \cdot \VC(\perclass \circ \repclass).
\]
The upper bound still holds even if $\repclass$ does not contain a surjective function.
\end{comment}
\end{customlem}
\begin{proof}
    We will show the two parts of the statement separately.

    \medskip\emph{Proof of the upper bound on VC dimension.}
    Assume that we have a dataset $X=  ((j_1, \feat_1), \ldots,$ $ (j_n,\feat_n))$ of size $n = \VC(\perclass^{\otimes \users}\circ \repclass)$ which can be shattered by $\perclass^{\otimes \users}\circ \repclass$. We split it into $\users$ disjoint datasets $X_1,\ldots,X_\users$, where $X_j = \{ \feat_i: (j,\feat_i) \in X\}$. Each one of these datasets can be shattered by $\perclass\circ \repclass$.
    
    Let $n_j$ be the size of dataset $X_j$. Then, we have that $n_j \leq \VC(\perclass \circ \repclass)$. As a result, we obtain $\VC(\perclass^{\otimes \users}\circ \repclass) = \sum_{j \in [\users]}n_j \leq \users\VC(\perclass \circ \repclass)$.
    
    \medskip\emph{Proof of the lower bound on VC dimension.}
    We will first show that $\VC(\perclass^{\otimes \users}\circ \repclass) \geq \VC(\perclass \circ \repclass)$. Suppose $X = (\feat_1, \ldots, \feat_n)$ is a dataset of size $n$ that can be shattered by class $\perclass\circ \repclass$. 
    % This means that for any labeling $(\lab_1, \ldots, \lab_n)$ there exist $\per^* \in \perclass$ and $\rep^* \in \repclass$ such that $\forall i \in [n]$ $\lab_i = \per^*(\rep^*(\feat_i))$. 
    Then, for any $j_1, \ldots, j_n \in [\users]$, the dataset $((j_1, \feat_1), \ldots, (j_n, \feat_n))$ can be shattered by $\perclass^{\otimes \users}\circ \repclass$. 
    To see this, fix a labeling $(\lab_1, \ldots, \lab_n)$ of $((j_1, \feat_1), \ldots, (j_n, \feat_n))$.
    Since we assumed $X$ could be shattered by $\perclass\circ \repclass$, there exists $\per^* \in \perclass$ and $\rep^*\in\repclass$ such that $\forall i \in [n]$, $\lab_i = \per^*(\rep^*(\feat_i))$. 
    Thus, there is a function in $\perclass^{\otimes t}\circ \repclass$ (namely, the one with  representation $\rep^*$ and all $t$ personalization functions equal to $\per^*$) that realizes this labeling.
    Therefore, $\VC(\perclass^{\otimes \users}\circ \repclass) \geq \VC(\perclass \circ \repclass)$.

    Next, we will prove that $\VC(\perclass^{\otimes \users}\circ \repclass) \geq \users  \VC(\perclass)$. Let $(\map_1, \ldots, \map_n) \in \mapdom^n$ be a dataset that $\perclass$ can shatter. Since there exists an $\rep \in \repclass$ whose image is $\mapdom$, there exist $(\feat_1, \ldots, \feat_n) \in \featdom^n$ such that $\forall i \in [n]$ $\rep(\feat_i)=\map_i$. 
    We now construct a new dataset $\cup_{j \in [\users]}\left\{(j,\feat_1)\ldots, (j,\feat_n)\right\}$, which has $n\users$ datapoints. 
    Our function class $\perclass^{\otimes \users}\circ \repclass$ can shatter this dataset. 
    To see this: for any labeling we split the dataset to $\users$ parts according to the value of $j$, use $\rep$ to get $(\map_1, \ldots, \map_n)$ for each part. We then label each part using an $f \in \perclass$. This means that there exists a dataset of size $\users  \VC(\perclass)$ that $\perclass^{\otimes \users}\circ \repclass$ can shatter. Thus, $\VC(\perclass^{\otimes \users}\circ \repclass) \geq \users  \VC(\perclass)$.
\end{proof}
\fi
\ifnum \coltshort = 0
\subsection{Proof of  \texorpdfstring{\Cref{thm:mtl-vc-hl}}{Lg}}\label{sec:mtl-vc-hl-proof}

\begin{customthm}{\ref*{thm:mtl-vc-hl}}[Restated]
\thmMtlVcHLText
\begin{comment}
 Fix $\users, \datadim, \repdim \in \N$. Then 
\[
\VC\left(\perclass_{\repdim}^{\otimes \users}\circ \repclass_{\datadim,\repdim}\right )  = 
\begin{cases}
\datadim \users + \users, &\textrm{ if } \users \leq \repdim\\
\Theta(\datadim \repdim +\repdim\users),&\textrm{ if } \users > \repdim.
\end{cases}
\]
\end{comment}
\end{customthm}
\begin{proof}
    We first show that for $\users \leq \repdim$ tasks $\VC\left(\perclass_{\repdim}^{\otimes \users}\circ \repclass_{\datadim,\repdim}\right )  = \datadim \users +\users$. To lower-bound the VC dimension, we will show that there exists a dataset of size $\datadim\users+\users$ which can be shattered by $\perclass_{\repdim}^{\otimes \users}\circ \repclass_{\datadim,\repdim}$. We know that the VC dimension of the class of $\datadim$-dimensional thresholds $\perclass_{\datadim}$ is $\datadim+1$, which means that there exists a dataset $(\mathbf{\feat}_1, \ldots, \mathbf{\feat}_{\datadim+1})$ which can be shattered by $\perclass_{\datadim}$. 
    Consider the dataset $\cup_{j \in [\users]} \{(j,\mathbf{\feat}_1), \ldots, (j, \mathbf{\feat}_{\datadim+1})\}$. This dataset can be shattered by $\perclass_{\repdim}^{\otimes \users}\circ \repclass_{\datadim,\repdim}$. To see this, fix a labeling $\mathbf{\lab} \in \bits^{\users \times (\datadim+1)}$, where $\lab_{j,i}$ is the label of datapoint $(j, \mathbf{\feat}_{i})$. For a $j \in [\users]$ we know that there exist $\mathbf{b}_j \in \R^{\datadim}$ and $w_j \in \R$ such that for all $i \in [\datadim+1]$ we have $\textrm{sign}(\mathbf{b}_j  \mathbf{\feat}_{i}-w_j) = \lab_{j,i}$. Since $ \users \leq \repdim$, we set $\mathbf{B}$ to be the matrix with rows $b_j^T$ for the first $\users$ rows and all zeros everywhere else and $\mathbf{a}_j \in \R^{\repdim}$ to be the one-hot encoding of $j$. We pick $\rep(\mathbf{\feat}) = \mathbf{B}\mathbf{\feat}$ and $\per_j(\mathbf{\map}) = \textrm{sign}(\mathbf{a}_j\mathbf{\map}-w_j)$. Therefore, we have that for all $j \in [\users]$ and $i \in [\datadim+1]$, $\lab_{j,i} = \per_j(h(\mathbf{\feat}_i))$. As a result, 
    \[\VC(\perclass_{\repdim}^{\otimes \users}\circ \repclass_{\datadim,\repdim})\geq \datadim\users+\users. \]

    For the upper bound \Cref{lem:vc-bounds} says that $\VC(\perclass_{\repdim}^{\otimes \users}\circ \repclass_{d,k})\leq \users\VC(\perclass_{\repdim}\circ \repclass_{\datadim,\repdim})$. Furthermore, since the composition of linear functions is linear the class of composite functions $\perclass_{\repdim}\circ \repclass_{\datadim,\repdim}$ and $\perclass_{\datadim}$ are the same. Hence,
    \[
    \VC(\perclass_{\repdim}^{\otimes \users}\circ \repclass_{\datadim,\repdim})\leq \users\VC(\perclass_{\datadim}) = \datadim\users+\users.
    \] 
    
    We now look at the case where we have more than $\repdim$ tasks. For the upper bound we rewrite functions $\conc \in \perclass_{\repdim}^{\otimes \users}\circ \repclass_{\datadim,\repdim}$ as $\conc(j,\mathbf{\feat}) = \textrm{sign}(\mathbf{a}_j\mathbf{B}\mathbf{\feat}-w_j)$.
    Observe that this is equivalent to
    \begin{align*}
        g(j, \mathbf{\feat}) = \textrm{sign}(\mathbf{e_j}^T \mathbf{A} \mathbf{B} \mathbf{\feat} - \mathbf{e_j}^T\mathbf{w}),
    \end{align*}
    where $\mathbf{e}_j \in \{0,1\}^{\users}$ is the one-hot encoding of $j$ and 
    \[
    \mathbf{A} = \begin{pmatrix}
        \mathbf{a}_1^T \\
        \vdots \\
        \mathbf{a}_{\users}^T
    \end{pmatrix} \textrm{ and } \mathbf{w} = \begin{pmatrix}
        w_1 \\
        \vdots \\
        w_{\users}
    \end{pmatrix}.
    \]
    Every combination of $\mathbf{A} \in \R^{\users\times \repdim}, \mathbf{B} \in R^{\repdim \times \datadim}$ and $\mathbf{w} \in \R^{\users}$ gives us a specific labeling function $\conc$. 
    Let $\samples \geq 8(\users \repdim + \repdim \datadim+\users)$. 
    Take a dataset $((j_1,\mathbf{\feat}_1),\ldots,(j_n, \mathbf{\feat}_{\samples})) \in ([\users] \times \R^{\datadim})^{\samples}$. %, where $\mathbf{e}_i$ is the one-hot encoding of $j_i$ for all $i \in [\samples]$. 
    We will show that $\perclass_{\repdim}^{\otimes \users}\circ \repclass_{\datadim,\repdim}$ does not shatter this data set.
    For each $i\in[\samples]$, we define a polynomial $p_i(\mathbf{A},\mathbf{B}, \mathbf{w}) = \mathbf{e}_{j_i}^T \mathbf{A} \mathbf{B}\mathbf{\feat}_i - \mathbf{e}_{j_i}^T\mathbf{w}$.
    Each of these is a degree-2 polynomial in $\users \repdim + \repdim \datadim+\users$ variables. By construction, $g(j_i,\mathbf{x}_i) =\sgn(p_i(\mathbf{A},\mathbf{B},\mathbf{w}))$. By \Cref{fct:vc-polynomials}, $\VC\left(\perclass_{\repdim}^{\otimes \users}\circ \repclass_{\datadim,\repdim}\right ) \leq 8 (\datadim \repdim +\repdim\users + \users)$.
    
    By \Cref{lem:vc-bounds} we have that $\VC(\perclass_{\repdim}^{\otimes \users}\circ \repclass_{\datadim,\repdim}) \geq \repdim \users + \users$. Additionally, for $\users > \repdim$ any dataset that is shattered by $\perclass_\repdim^{\otimes \repdim}\circ \repclass_{\datadim,\repdim}$ can be shattered by $\perclass_\repdim^{\otimes \users}\circ \repclass_{\datadim,\repdim}$. As a result, $\VC(\perclass_\repdim^{\otimes \users}\circ \repclass_{\datadim,\repdim}) \geq \VC(\perclass_\repdim^{\otimes \repdim}\circ \repclass_{\datadim,\repdim})$. We proved above that $\VC(\perclass_{\repdim}^{\otimes \repdim}\circ \repclass_{\datadim,\repdim}) = \datadim \repdim + \repdim$.  Therefore, 
    \[
        \VC(\perclass_{\repdim}^{\otimes \users}\circ \repclass_{\datadim,\repdim}) \geq \Omega (\datadim \repdim + \repdim \users). 
        %\qedhere
    \]
\end{proof}
\fi
\section{Basics on \NRCC}
\label{sec:nrcc}

For halfspaces $\perclass_\repdim$, the \nrcc{} is exactly $\VC(\perclass_{\repdim}) +1$. In general, though, \nrcc{}  (Definition~\ref{def:non-realizability-complexity}) can be arbitrarily larger or smaller than VC dimension. 

\begin{lem}[VC dimension and \nrcc]
    \label{lem:nrcc-to-vc}
    For every integer $\ell\geq 2$, 
    \begin{enumerate}
        \item There exists a class $\perclass_{\ell, \VC}$ with VC dimension $\ell$ and \nrcc{} 2.
        \item There exists a class $\perclass_{\ell,NR}$ with \nrcc{} $\ell$ and VC dimension 1.
        \item It is possible to add one more function to the class $\perclass_{\ell,NR}$ above and reduce its \nrcc{} to 2. This complexity measure is therefore not monotone under inclusion.
    \end{enumerate} 
\end{lem}

\begin{proof}
    Fix $\ell\geq 2$. Let $\perclass_{\ell,\VC}$ be the class of all functions  $f:\set{0,1,...,\ell}\to \bits$ such that $f(0)=1$. This class has VC dimension $\ell$ since shatters the set $\{1,...,\ell\}$ but not the entire domain. However, it has very short certificates of nonrealizability: if a set cannot be realized, it must either contain the same value labeled with different points, or the example  $(0,-1)$. Either way, there is a subset of 1 or 2 points that is not realizable. 

    Let  $\perclass_{\ell,NR}$ denote the class of point functions on $[\ell]$; that is, the functions $f_x$ that take the value 1 at \emph{exactly} one point $x \in [\ell]$ (and -1 elsewhere). This class has VC dimension 1. However, the set $S=\set{(x,-1): x \in [\ell]}$ is unrealizable (since no value is labeled with 1), but all of its  subsets of size $\ell-1$
    are realizable (by $f_x$, where $(x,-1)$ is the point that was removed from $S$). 

    Now, we can extend $\perclass_{\ell,NR}$ to a larger class by adding the constant $-1$ function, so that it now consists of the functions $f:[\ell]\to\bits$ that take the value 1 at \emph{at most} one point $x \in [\ell]$. This class has \nrcc{} 2, since the any unrealizable set must either have a point labeled twice with different values, or two labeled points $(x,1)$ and $(y,1)$ for $x\neq y$. 
\end{proof}
\ifnum \coltshort = 1
\section{Multitask Learning and Metalearning Models}
In this section we introduce two other learning models we also consider in this work, one for multitask learning and one for general (improper) metalearning.
\label{sec:models_red}
\subsection{The Multitask Learning Model}
In multitask learning, we pool data together from $\users$ related tasks with the aim of finding one classifier per task so that the average test error per task is low. When these tasks are related, we may need fewer samples per task than if we learn them separately, because samples from one task inform us about the distribution of another task. In this work, we consider classifiers that use a shared representation to map features to an intermediate space in which the specialized classifiers are defined.  We want to achieve low error on tasks that are related by a single shared representation that can be specialized to obtain a low-error classifier for most tasks.

\begin{defn}[Multitask learning]\label{def:multitask_learning}
Let $\repclass$ be a class of representation functions $\rep: \featdom \to \mapdom$ and $\perclass$ be a class of specialized classifiers $\per: \mapdom \to \bits$. We say that $(\repclass,\perclass)$ is \emph{distribution-free $(\varepsilon, \delta)$-multitask learnable for $\users$ tasks with $\samples$ samples per task} if there exists an algorithm $\alg$ such that for every $\users$ probability distributions $\dist_1, \ldots, \dist_{\users}$ over $\featdom \times \bits$, for every $\rep \in \repclass$ and every $\per_1, \ldots, \per_{\users} \in \perclass$, given $\samples$ i.i.d. samples from each $\dist_i$, returns  hypothesis $\conc:[\users]\times \featdom \to \bits$ such that with probability at least $1-\delta$ over the randomness of the samples and the algorithm 
\begin{align*}
    \frac{1}{\users}\sum_{j\in [\users]}\er(\dist_j, g(j,\cdot)) \leq \min_{\rep \in \repclass, \per_1, \ldots, \per_{\users} \in \perclass} \frac{1}{\users}\sum_{j\in [\users]}\er(\dist_j, \per_j\circ \rep)+ \varepsilon.
\end{align*}

If there exists an algorithm $\alg$ such that the same guarantee holds except that we only quantify over all distributions $\dist_1,\dots,\dist_{\users}$ such that
\begin{align*}
    \min_{\rep \in \repclass, \per_1, \ldots, \per_{\users} \in \perclass} \frac{1}{\users}\sum_{j\in [\users]}\er(\dist_j, \per_j\circ \rep) = 0
\end{align*}
then we say that $(\repclass,\perclass)$ is \emph{distribution-free $(\varepsilon, \delta)$-multitask learnable for $\users$ tasks with $\samples$ samples per task in the realizable case}.

For brevity, we will typically omit the term ``distribution-free,'' which applies to all of the results in this paper.
\end{defn}

The natural approach to learning $(\repclass,\perclass)$ is to learn a classifier that, given the index of a task and the features of a sample, computes the representation of the sample and then labels it using the task-specific specialized classifier. We call the class of these classifiers $\perclass^{\otimes \users}\circ \repclass$.

\begin{defn}
Let $\repclass$ be a class of representations $\rep: \featdom \to \mapdom$ and $\per$ be a class of specialized classifiers $\per: \mapdom \to \bits $. We define the class of 
% specialized classifiers 
composed classifiers
for multitask learning with $\users$ tasks as 
$$\perclass^{\otimes \users}\circ \repclass = \left\{\conc: [\users]\times \featdom \to \bits \mid \exists \rep \in \repclass, \per_1, \ldots, \per_{\users} \in \perclass \textrm{ s.t. } \conc(j,x) = \per_j(\rep(x))\right\}.$$
\end{defn}

\section{Metalearning}
\label{sec:meta_app}

In this section, we illustrate the techniques used to achieve the results in \Cref{sec:meta} via the special case of monotone thresholds applied to 1-dimensional representations (\Cref{sec:real-tech}). This special case corresponds to a natural setting where the representation assigns a real-valued score to each example, but the threshold for converting that score into a binary label may vary from task to task. We then restate the main theorems and provide their complete proofs for the realizable (\Cref{sec:metalearn-realizable-app}) and the agnostic cases (\Cref{sec:metalearning_agnostic}), respectively. Finally, we include a sample and task complexity bound for general metalearning (\Cref{sec:sample-gen-meta}).

\subsection{Sample and Task Complexity Bounds for the Realizable Case}

\label{sec:metalearn-realizable-app}

When the metadistribution is meta-realizable, Theorems~\ref{thm:met-samples} and \ref{thm:met-real-samples} bound the number of tasks and samples per task we need to metalearn. Their proofs follow the structure described in \Cref{sec:real-tech}.  

\begin{customthm}{\ref*{thm:met-samples}}[Restated]
\label{thm:met-samples-app}
\thmMetSamples

\end{customthm}

We first prove Lemma~\ref{lem:real-to-error} and Lemma~\ref{lem:num-tasks}, which we use in the proof of Theorem~\ref{thm:met-samples-app}.

\begin{customlem}{\ref*{lem:real-to-error}}[Restated]
     \label{lem:real-to-error-app}
    Let $\perclass$ be the class of specialized classifiers $\per: \mapdom \to \bits$ with $\NRC(\perclass) = \wit$. 
    %Assume there exists a positive integer $\wit$ such that every data set $S$ of size at least $\wit$ \adaminline{Use $\NRC(\perclass)$ notation} that is not realizable by $\perclass$ has a non-realizability witness of size $\wit$.
    Fix an arbitrary distribution $\distint$ over $\mapdom\times\bits$.
    If $\er(\distint, \perclass) > 0$, then $$ p_{nr}(\distint, \perclass, \wit)  \ge  \frac 1 2
    \left(\left[\frac{\wit \cdot \er(\distint, \perclass)}{16e\cdot \vc \ln{(16/\er(\distint, \perclass))}}\right]^{\wit}\right),$$ where $v = \max(\VC(\perclass),\wit)$.
\end{customlem}

\begin{proof}
Let $g(\eps)  = \frac{16}{\eps}\VC(\perclass)\ln(16/\eps)$. 
% Let $g(\eps)  = \frac{8}{\eps}[\VC(\perclass)\ln(16/\eps)+\ln(4)]$. 
Function $g$ is continuous and strictly decreasing in $\eps$ for $\eps \in (0,1]$. %Therefore, its inverse $g^{-1}$ exists.
The proof analyzes two cases

In case one, if %$\er(\distint,\perclass) < g^{-1}(\wit)$, then 
$\wit > g(\er(\distint,\perclass))$, then by Theorem \ref{fact: vc-gen} the probability that a dataset of $\wit$ points drawn from $\distint$ is not realizable by $\perclass$ is  
\begin{align*}
    p_{nr}(\distint, \perclass, \wit) 
    &= \Pr[S_\wit\sim \distint^{\wit}]{S_\wit\text{ is not realizable by }\perclass} \\
    &= \Pr[S_\wit\sim \distint^{\wit}]{\min_{\per \in \perclass}\er(S_ \wit,\per)>0} \geq \frac{1}{2},
\end{align*}
because $\er(\distint,\perclass) >0$ and $\wit\ge \frac{8}{\eps}\left[\mathrm{VC}(\mathcal{F})\ln (16/\eps)+\ln (2/\delta)\right]$ for $\eps=\er(\distint,\perclass)$ and $\delta=\frac 1 2$.
This is stronger than the claimed lower bound.
To see this, recall $\vc \ge \wit$ and observe
\begin{align*}
    \frac 1 2 \left(\frac{\wit \cdot \er(\distint, \perclass)}{16e\cdot \vc \ln{(16/\er(\distint, \perclass))}}\right)^{\wit}
    % &\le  \left[\frac{\wit \cdot \er(\distint, \perclass)}{16e\cdot \textcolor{black}{\wit} \ln{(16/\er(\distint, \perclass))}}\right]^{\wit} \\
    &\le \frac 1 2
    \left(\frac{\er(\distint, \perclass)}{16e\cdot \ln{(16/\er(\distint, \perclass))}}\right)^{\wit} \\
    &\le \frac 1 2 \left(\frac{1}{16e\cdot \ln{(16)}}\right)^{\wit},
    % &\le \frac 1 2.
\end{align*}
which is less than $\frac{1}{2}$.

%If $\er(\distint,\perclass) \geq g^{-1}(\wit)$
In case two, suppose $\wit \leq g(\er(\distint,\perclass))$.
Drawing $\wit$  i.i.d.\ samples from $\distint$ is equivalent to drawing a larger dataset $S_n =\{(\map_i,\lab_i)\}_{i \in [n]}$ from $\distint$ of some size $n\geq m$, set later in the proof, and picking a uniformly random subset $S_\wit\subseteq S_n$ of size $\wit$.
% $\wit$ points uniformly at random without replacement from $S_n$. 
More formally, 
\begin{align*}
     p_{nr}(\distint, \perclass, \wit) &=
    \Pr[S_\wit \sim \distint^{\wit}]{S_\wit \text{ is not realizable}} \\
    % &= \Pr[\substack{(\map_1,\lab_1),\ldots ,(\map_{n},\lab_{n})\sim \distint^{n}\\(i_1, \ldots, i_{\wit}) \sim \text{Unif}{\left(\binom{n}{m}\right)}}]{(\map_{i_1},\lab_{i_1}), \dots, (\map_{i_\wit},\lab_{i_\wit}) \text{ is not realizable} }.
    &= \Pr[\substack{S_n\sim \distint^n \\ S_\wit \sim \binom{S_n}{\wit}}]{S_\wit \text{ is not realizable} }.
\end{align*}
Furthermore, we notice that  if $S_n$ is labeled correctly by some $\per \in \perclass$, then $\per$ labels $S_\wit$ correctly, too.
Hence, 
\begin{align}
    p_{nr}(\distint, \perclass, \wit) &= \Pr{S_\wit \text{ is not realizable} } \nonumber \\
        &= \Pr{S_\wit \text{ is not realizable}  \mid S_n \text{ is not realizable} } \cdot \Pr{S_n \text{ is not realizable}}. \label{eq:not_realizable_breakdown}
\end{align}
We know that if $S_n$ is not realizable then there exists a non-realizability certificate of size $m$.
Since there are $\binom{n}{\wit}$ subsets, $S_\wit$ is exactly this certificate with probability at least $1/\binom{n}{\wit}$.
 % We notice that  if $S_n$ is labeled correctly by an $\per \in \perclass$, then the same $\per$ labels $(\map_{i_1},\lab_{i_1}), \dots, (\map_{i_\wit},\lab_{i_\wit})$ correctly, too. Hence%, by the law of total probability
% \begin{align*}
%     &\Pr[\substack{(\map_{1},\lab_{1}),\ldots ,(\map_{i_n},\lab_{n})\sim \distint^{n}\\(i_1, \ldots, i_{\wit}) \sim \text{Unif}{\left(\binom{n}{m}\right)}}]{(\map_{i_1},\lab_{i_1}), \dots, (\map_{i_\wit},\lab_{i_\wit})\text{ is not realizable}} =\\
%     &\Pr[\substack{(\map_{1},\lab_{1}),\ldots ,(\map_{n},\lab_{n})\sim \distint^{n}\\(i_1, \ldots, i_{\wit}) \sim \text{Unif}{\left(\binom{n}{m}\right)}}]{(\map_{i_1},\lab_{i_1}), \dots, (\map_{i_\wit},\lab_{i_\wit})\text{ is not realizable} \Big| S_n \text{ is not realizable}} \Pr[S_n \sim \distint^n]{S_n \text{ is not realizable}}.
% \end{align*}

We now provide a lower bound on the probability that $S_n$ is not realizable.
We set 
$n := \frac{16}{\er(\distint,\perclass)}\VC(\perclass)\ln(16/\er(\distint,\perclass))$,
% $n := \frac{8}{\er(\distint,\perclass)}[\VC(\perclass)\ln(16/\er(\distint,\perclass))+\ln(4)]$,
which satisfies $n\geq m$ by hypothesis. Then, by Theorem \ref{fact: vc-gen} the probability that dataset $S_n$ is not realizable by $\perclass$ is  
\begin{align*}
    \Pr[S_n\sim \distint^{n}]{S_n\text{ is not realizable by }\perclass} = 
    \Pr[S_n\sim \distint^{n}]{\min_{\per \in \perclass}\er(S_n,\per)>0} \geq \frac{1}{2},
\end{align*}
again because $\er(\distint,\perclass) >0$ and $n$ is sufficiently large.

Thus, continuing from \Cref{eq:not_realizable_breakdown} and using a bound on the binomial coefficient, we see that $p_{nr}(\distint, \perclass, \wit) \ge \frac{1}{2\binom{n}{\wit}} \geq \frac{1 }{2(\frac{en}{\wit})^{\wit}}$.
%
% We know that if $S_n$ is not realizable then there exists a non-realizability certificate of size $m$. Thus,
% \begin{align*}
%     &\Pr[\substack{S_n\sim \distint^{n}\\(i_1, \ldots, i_{\wit}) \sim \text{Unif}{\left(\binom{n}{m}\right)}}]{(\map_{i_1},\lab_{i_1}), \dots, (\map_{i_\wit},\lab_{i_\wit})\text{ is not realizable by }\perclass \Big| S_n\text{ is not realizable by }\perclass} \geq\\
%     &\Pr[\substack{S_n\sim \distint^{n}\\(i_1, \ldots, i_{\wit}) \sim \text{Unif}{\left(\binom{n}{m}\right)}}]{\{(\map_{i_1},\lab_{i_1}), \dots, (\map_{i_\wit},\lab_{i_\wit}) \text{ is a n-r certificate } \Big| S_n\text{ is not realizable by }\perclass}\geq \frac{1}{\binom{n}{\wit}}.
% \end{align*}
% Combining the steps above, we see that
% \begin{align*}
%     &p_{\text{nr}}(\distint,\perclass,\wit)  =\\&\Pr[\substack{(\map_{1},\lab_{1}),\ldots ,(\map_{n},\lab_{n})\sim \distint^{n}\\(i_1, \ldots, i_{\wit}) \sim \text{Unif}{\left(\binom{n}{m}\right)}}]{(\map_{i_1},\lab_{i_1}), \dots, (\map_{i_\wit},\lab_{i_\wit})\text{ is not realizable}
%     \Big|
%     S_n \text{ is not realizable}} \Pr[S_n \sim \distint^n]{S_n \text{ is not realizable}}\geq\\
%     & \frac{1}{2\binom{n}{\wit}} \geq \frac{1 }{2(\frac{en}{\wit})^{\wit}}.
% \end{align*}
Plugging in $n=\frac{16}{\er(\distint,\perclass)}\VC(\perclass)\ln(16/\er(\distint,\perclass))$, we get that
\begin{align}
    p_{\text{nr}}(\distint,\perclass,\wit) \geq \frac{1}{2}\left[\frac{m \cdot \er(\distint,\perclass)}{16e\cdot \VC(\perclass)\ln{(16/\er(\distint,\perclass))}}\right]^m.
\end{align}
As $\VC(\perclass)\le \vc$, this is stronger than the claim in the lemma.
This concludes the proof.
\end{proof}

\begin{customlem}{\ref*{lem:num-tasks}}[Restated]
\lemNumTasks
 \label{lem:num-tasks-app}
\end{customlem}

\begin{proof}
%Recall that $\metadist$ is a metadistribution over $\cP$. 
Our goal is to show that there exists a shared representation that achieves small representation error for $\metadist$.
%For the rest of this proof, we formalize the argument we mentioned above. 
The proof proceeds in two stages.
First, we use our bound on the VC dimension of $\realpclass_{\wit, \perclass, \repclass}$ to show that we can find a representation $\hat h$ that, when applied to data from a new task, admits a perfect specialized classifier with high probability.
Second, we connect this to $\rer(\metadist,\hat{\rep}, \perclass)$, the error of $\hat \rep$ on the meta-distribution.

Recall our approach for monotone thresholds in \Cref{sec:real-tech}:
for each task, we receive a data set $S^{(j)}=(\feat_1^{(j)},\lab_1^{(j)},\ldots, \feat_{\wit}^{(j)}, \lab_{\wit}^{(j)})$ and construct a single ``data point'' $\zeta_j$:
\begin{align*}
    \zeta_j \defeq (S^{(j)}, +1)\,.
\end{align*}
We set the label to ``$+1$'' because, by the meta-realizability assumption, there exists an $h^*$ such that $r_{h^*}(S^{(j)})=+1$ for all $j$.
Since each task distribution $\dist_j$ is drawn independently from $\metadist$, these $\zeta_j$ observations are drawn i.i.d.\ from some distribution $\cD$.
For an illustration of this process, see Figure~\ref{fig:newDataset}.

We find an $\hat h$ such that $r_{\hat h}$ has zero error on the dataset $\zeta_1,\ldots,\zeta_t$.
Because we set $t = \Theta\left(\frac{\VC(\realpclass_{\wit, \perclass, \repclass})+\ln(1/\delta)}{\phi(\varepsilon)}\right)$, by Theorem~\ref{fact:pac-vc} we know that $r_{\hat h}$ generalizes.
That is, with probability at least $1-\delta$, we have
\begin{equation}
    \Pr[\zeta=(S,+1)\sim \cD]{\realp_{\hat{\rep}}(S)\neq +1}\leq \phi(\varepsilon)\,.
    \label{eq:pr_nr_err_erm}
\end{equation}

For a data set $S=((\feat_1,\lab_1),\ldots ,(\feat_{\wit},\lab_{\wit}))$ and representation $\rep$, define $S_h$ to be \newline$ ((\rep(\feat_1),\lab_1),\ldots ,(\rep(\feat_{\wit}),\lab_{\wit}))$.
Recall that by the assumption of this lemma, for $\hat{\rep}$ and all $\dist$ in the support of $\metadist$, if $\rer(\dist, \hat{\rep}, \perclass)>0$ we have:
\begin{equation}
\phi\left(\rer(\dist, \hat{\rep}, \perclass)\right) \leq  \Pr[S \sim \dist^{\wit}]{S_{\hat h} \text{ is not realizable by }\perclass}. 
% \phi\left(\rer(\dist, \hat{\rep}, \perclass)\right) \leq  \Pr[(\feat_1,\lab_1),\ldots ,(\feat_{\wit},\lab_{\wit})\sim \dist^{\wit}]{(\hat{\rep}(\feat_1),\lab_1), \dots, (\hat{\rep}(\feat_{\wit}),\lab_{\wit})\text{ is not realizable by }\perclass}. 
\end{equation}
Note that $\phi$ is a strictly increasing function and, thus, it has an inverse function $\phi^{-1}$ that is also strictly increasing. Therefore, the above bound implies that
\begin{equation} \label{eq:phi_assumption}
\rer(\dist, \hat{\rep}, \perclass)\leq  \phi^{-1}\left( \Pr[S\sim \dist^{\wit}]{S_{\hat h}\text{ is not realizable by }\perclass}\right). 
% \rer(\dist, \hat{\rep}, \perclass)\leq  \phi^{-1}\left( \Pr[(\feat_1,\lab_1),\ldots ,(\feat_{\wit},\lab_{\wit})\sim \dist^{\wit}]{(\hat{\rep}(\feat_1),\lab_1), \dots, (\hat{\rep}(\feat_{\wit}),\lab_{\wit})\text{ is not realizable by }\perclass}\right). 
\end{equation}

Now we are ready to bound the meta-error of $\hat{\rep}$: $\rer(\metadist, \hat{\rep}, \perclass)$. We start by bounding $\phi\left(\rer(\metadist, \hat{\rep}, \perclass)\right)$. Since $\phi$ is convex, $\phi^{-1}$ is concave and we can apply Jensen's inequality. We get that:
\begin{align*}
    \rer(\metadist,\hat{\rep}, \perclass) 
    &= \Exp_{\dist \sim \metadist}\left[\rer(\dist, \hat{\rep}, \perclass)\right]\\
    & \leq \Exp_{\dist \sim \metadist}\left[\rer(\dist, \hat{\rep}, \perclass)\mid\rer(\dist, \hat{\rep}, \perclass)>0\right]\\
    &\leq \Exp_{\dist \sim \metadist}\left[\phi^{-1}\left(\Pr[S\sim \dist^{\wit}]{S_{\hat h}\text{ is not realizable by }\perclass}\right)\right] \tag{by Eq. \ref{eq:phi_assumption}}\\
    & \leq \phi^{-1} \left(\Exp_{\dist \sim \metadist}\left[\Pr[S\sim \dist^{\wit}]{S_{\hat h}\text{ is not realizable by }\perclass}\right]\right)\tag{Jensen's inequality}\\
    & \leq \phi^{-1}(\phi(\eps))=\eps \tag{by Eq. \ref{eq:pr_nr_err_erm}}
    % &\leq \Exp_{\dist \sim \metadist}\left[\phi^{-1}\left(\Pr[(\feat_1,\lab_1),\ldots ,(\feat_{\wit},\lab_{\wit})\sim \dist^{\wit}]{(\hat{\rep}(\feat_1),\lab_1), \dots, (\hat{\rep}(\feat_{\wit}),\lab_{\wit})\text{ is not realizable by }\perclass}\right)\right] \tag{by Eq. \ref{eq:phi_assumption}}\\
    % & \leq \phi^{-1} \left(\Exp_{\dist \sim \metadist}\left[\Pr[(\feat_1,\lab_1),\ldots ,(\feat_{\wit},\lab_{\wit})\sim \dist^{\wit}]{(\hat{\rep}(\feat_1),\lab_1), \dots, (\hat{\rep}(\feat_{\wit}),\lab_{\wit})\text{ is not realizable by }\perclass}\right]\right)\tag{Jensen's inequality}\\
    % & \leq \phi^{-1}(\phi(\eps))=\eps \tag{by Eq. \ref{eq:pr_nr_err_erm}}
\end{align*}

The above bound implies: 
$$\rer(\metadist, \hat{\rep}, \perclass) \leq \varepsilon = \min_{\rep \in \repclass}\rer\paren{\metadist, \rep, \perclass} + \varepsilon\,.$$

As a result, 
using $\users = O\left(\frac{\VC(\realpclass_{\wit, \perclass, \repclass})+\log(1/\delta)}{\phi(\varepsilon)}\right)$ tasks and $\wit$ samples from each task,  we have found a representation function $\hat{\rep}$ that has the desired error bound for metalearning with probability $1-\delta$. Hence, the proof is complete.
\end{proof}

%\maryamnote{Does PAC learning imply the existence of an algorithm? Can we say we can find ERM for an uncountable infinite size class?}\adamnote{Good point. Maybe we can just say "outputting the representation $\hat h$ that maximizes the number of realizable tasks" or something similarly direct.}

We can now prove Theorem~\ref{thm:met-samples-app} by combining the results of Lemma~\ref{lem:real-to-error-app} and Lemma~\ref{lem:num-tasks-app}.
 \begin{proof}[Proof of Theorem \ref{thm:met-samples-app}] 
Let $$\phi(\varepsilon) = \frac{1}{2} \left[\frac{\wit \varepsilon}{16e\cdot \max(\VC(\perclass),\wit)\ln(16/\varepsilon)}\right]^{\wit}.$$ 
By differentiating $\phi$ twice we see that it is a strictly increasing convex function in $\varepsilon$, for $\eps \in (0,1)$ and $\wit \geq 1$.

Applying Lemma~\ref{lem:real-to-error-app}, we obtain that for every representation $\hat\rep$ and distribution $P$ if \\$\min_{\per \in \perclass}\Pr[(\feat, \lab) \sim \dist]{\per(\hat{\rep}(\feat))\neq y}>0$, then
\begin{align*}
    \phi\biggl(\min_{\per \in \perclass} & \Pr[(\feat, \lab) \sim \dist]{\per(\hat{\rep}(\feat))\neq y}\biggr) \\
        &\leq   \Pr[(\feat_1,\lab_1),\ldots ,(\feat_{\wit},\lab_{\wit})\sim \dist^{\wit}]{(\hat{\rep}(\feat_1),\lab_1), \dots, (\hat{\rep}(\feat_{\wit}),\lab_{\wit})\text{ is not realizable by }\perclass}.
\end{align*}

 This satisfies the assumptions of Lemma~\ref{lem:num-tasks-app} which says that we can metalearn $(\repclass, \perclass)$ with $\users= O\left(\frac{\VC(\realpclass_{\wit, \perclass, \repclass})+\ln(1/\delta)}{\phi(\varepsilon)}\right)$ tasks and $\wit$ samples per task. This concludes our proof.
 \end{proof}

Our next result analyzes metalearning with fewer tasks and more samples per task.
Formally, Theorem~\ref{thm:met-real-samples} shows that we can metalearn $(\repclass, \perclass)$ for a meta-realizable $\metadist$ with $\Tilde{O}(\VC(\realpclass_{\samples, \perclass, \repclass})/\eps)$ tasks and $\Tilde{O}(\VC(\perclass)/\eps)$ samples per task.

 \begin{customthm}{\ref*{thm:met-real-samples}}[Restated]
\label{thm:met-real-samples-app}
\thmMetRealSamples
\end{customthm}
 
\begin{proof}
    First, we set $\eps_1:=\eps/3$. Set the number of tasks $$\users := O\left(\frac{\VC(\realpclass_{\samples, \perclass, \repclass})\ln(1/\eps_1)+\ln(1/\delta)}{\eps_1}\right).$$
    Each task $j$ has $\samples $ samples $S_j = \{(\feat_i^{(j)}, \lab_i^{(j)})\}_{i \in [\samples]}$, we construct a dataset $Z$ of $\users$ points $\zeta_j = ((\feat_1^{(j)},\lab_1^{(j)},\ldots, \feat_{\samples}^{(j)}, \lab_{\samples}^{(j)}),+1)$, one for each task $j$, as in \Cref{fig:newDataset}. Each $\zeta_j$ is drawn i.i.d.\ from data distribution $\cD$ (where we first draw $\dist$ from $\metadist$ and then $\samples$ points from $\dist$). 
    Since we are in the realizable case, there exists a representation $\rep$ such that $\realp_\rep$ returns $+1$ for every sample drawn from $\cD$. By Theorem~\ref{fact:pac-pd} we have that for $\hat{\rep} = \arg \min_{\rep \in \repclass} \er(Z,\realp_\rep)$ with probability at least $1-\delta$ over dataset $Z$
\begin{align*}
   \er(\cD, \realp_{\hat{\rep}})\leq \frac{\eps}{3}.
\end{align*}

Fix a distribution $\dist$ in the support of $\metadist$. We start by assuming that $\rer(\dist, \hat{\rep}, \perclass) > \eps/3$. By Theorem~\ref{fact: vc-gen} for a dataset $S$ of $$\samples := \frac{24}{\eps} \left(\VC(\perclass) \ln(48/\eps) + \ln(4)\right)$$ samples drawn from $\dist$, with probability at least $1/2$ over $S$ all specialized classifiers $\per \in \perclass$ with $\er(\dist, \per \circ \hat{\rep}) >\eps/3$ have $\er(S, \per \circ \hat{\rep})>0$. Therefore,  
\begin{align*}
    \Pr[S \sim \dist^\samples]{\min_{\per \in \perclass} \er(S, \per\circ\hat{\rep})> 0} =
    \Pr[S \sim \dist^\samples]{\realp_{\hat{\rep}}(S) \neq -1}   \geq \frac{1}{2}.
\end{align*}
By Markov's inequality, we see that
\begin{align*}
    \Pr[\dist \sim \metadist]{\Pr[S\sim \dist^n]{\realp_{\hat{\rep}}(S) \neq +1} \geq 1/2} \leq 2 \Exp_{\dist \sim \metadist}\left[\Pr[S\sim \dist^n]{\realp_{\hat{\rep}}(S) \neq +1}\right] = 2~\er(\cD, \realp_{\hat{\rep}})\leq \frac{2\eps}{3}.
\end{align*}
So far we have shown that for a fixed $\dist$ if $\rer(\dist, \hat{\rep}, \perclass) > \eps/3$, then $\Pr[S\sim \dist^n]{\realp_{\hat{\rep}}(S) \neq +1} \geq 1/2$. As a result, we have that
\begin{align*}
    \Pr[\dist \sim \metadist]{\rer(\dist, \hat{\rep}, \perclass) > \eps/3} \leq   \Pr[\dist \sim \metadist]{\Pr[S\sim \dist^n]{\realp_{\hat{\rep}} \neq +1} \geq 1/2} \leq   \frac{2\eps}{3}.
\end{align*}
We can use this to bound the meta-error of representation $\hat{\rep}$ as follows. We see that with probability at least $1-\delta$ over the datasets of the $\users$ tasks
\begin{align*}
    \rer(\metadist, \hat{\rep}, \perclass) &=\Exp_{\dist \sim \metadist}\left[ \min_{\per \in \perclass} \Pr[(x,y) \sim \dist]{ \per(\hat{\rep}(x)) \neq y}\right]\\
    & \leq \frac{\eps}{3} + \Pr[\dist \sim \metadist]{\min_{\per \in \perclass} \Pr[(\feat, \lab) \sim \dist]{\per(\hat{\rep}(\feat))\neq \lab} > \eps/3} \\
    &\leq \frac{\eps}{3} + 2\frac{\eps}{3} = \eps.
\end{align*}
This concludes our proof.
\end{proof}

\subsection{Sample and Task Complexity Bounds for the Agnostic Case}
\label{sec:metalearning_agnostic}

In this section, we consider metalearning in the agnostic case. 

\begin{customthm}{\ref*{thm:agn-met-samples}}[Restated]
\label{thm:agn-met-samples-app}
\thmagnmet
%Let $\repclass$ be a class of representation functions $\rep: \featdom \to \mapdom$ and $\perclass$ be a class of specialized classifiers $\per: \mapdom \to \bits$. Then, we can $(\varepsilon, \delta)$-metalearn $(\repclass, \perclass)$ with $\users \geq O\left(\frac{\pd(\erindclass_{\samples, \perclass, \repclass })\ln(1/\varepsilon)+\ln(1/\delta)}{\varepsilon^2}\right)$ tasks and $\samples = O\left( \frac{\VC(\perclass)+\ln(1/\varepsilon)}{\varepsilon^2}\right)$ samples per task.
\end{customthm}
\begin{proof} 
We begin by setting our parameters: Set $\eps_1 \coloneqq \eps/3$ and $\eps_2 \coloneqq \eps/3$. Let the number of tasks be the following for a sufficiently large constant in the $O$ notation:
\begin{equation} \label{eq:uniform_convergence_pdim}
    \users \coloneqq O\left( \frac{\pd(\erindclass_{\samples, \perclass, \repclass })\ln(1/\eps_1)+\ln(1/\delta))}{\eps_1^2}\right)\,.
\end{equation} 
Suppose for each task $j \in [t]$ has $\samples $ samples $S_j =\{(\feat_i^{(j)}, \lab_i^{(j)})\}_{i \in [\samples]}$, we construct a dataset $Z$ of $\users$ points $\zeta_j = ((\feat_1^{(j)},\lab_1^{(j)},\ldots, \feat_{\samples}^{(j)}, \lab_{\samples}^{(j)}))$, one for each task $j$. Each $\zeta_j$ is drawn i.i.d.\ from the data distribution $\cD$ for which we first draw $\dist$ from $\metadist$ and then $\samples$ points from $\dist$. By Theorem~\ref{fact:pac-pd}, we have that for $\hat{\rep} = \arg \min_{\rep \in \repclass} \frac{1}{\users}\sum_{j=1}^\users\erind_{\rep}(\zeta_j)$ with probability at least $1-\delta$:
\begin{align*}
    \Exp_{\zeta \sim \cD}\left[\erind_{\hat{\rep}}(\zeta)\right]\leq \min_{\rep \in \repclass}\Exp_{\zeta \sim \cD}\left[\erind_\rep (\zeta)\right] + \frac{\varepsilon}{3}\,.
\end{align*}
Fix a task distribution $\dist$ and a representation $\rep$. 
% For any given set of $n$ samples $\zeta$ and
For any fixed specialized classifier $\per' \in \perclass$, we have: 
$$\Exp_{(\feat_1,\lab_1),\ldots,(\feat_n, \lab_n) \sim \dist^n} \left[
\min_{\per \in \perclass} \frac{1}{\samples}\sum_{i \in [\samples]} \ind\{\per(\rep(\feat_i)) \neq \lab_i\} 
\right]
\leq 
\Exp_{(\feat_1,\lab_1),\ldots,(\feat_n, \lab_n) \sim \dist^n} \left[ 
\frac{1}{\samples}\sum_{i \in [\samples]} \ind\{\per'(\rep(\feat_i)) \neq \lab_i\}\right]\,.$$
Since the inequality holds for any $\per'$ in $\perclass$, it holds when we take minimum over all $\per'$. Hence, we obtain:
\begin{equation}\label{eq:swap_min}
    \Exp_{(\feat_1,\lab_1),\ldots,(\feat_n, \lab_n) \sim \dist^n} \left[ \min_{\per \in \perclass} \frac{1}{\samples}\sum_{i \in [\samples]} \ind\{\per(\rep(\feat_i)) \neq \lab_i\}\right] \leq \min_{\per \in \perclass} \Exp_{(\feat_1,\lab_1),\ldots,(\feat_n, \lab_n) \sim \dist^n}\left[\frac{1}{\samples}\sum_{i \in [\samples]} \ind\{\per(\rep(\feat_i)) \neq \lab_i\}\right]\,.
\end{equation}  
Next, we use the above inequality to continue bounding the expectation of  $\erind_{\hat{\rep}}$:
%Therefore, we have that with probability at least $1-\delta$ over the randomness of our constructed sample set of $Z = \left\{\zeta_j\right\}_{j \in [t]}$
\begin{align*}
    \Exp_{\zeta \sim \cD}\left[\erind_{\hat{\rep}}(\zeta)\right]&\leq
    \min_{\rep \in \repclass}\Exp_{\zeta \sim \cD}\left[\erind_\rep (\zeta)\right] + \frac{\varepsilon}{3} \tag{Eq.~\eqref{eq:uniform_convergence_pdim}}
    \\
    &=\min_{\rep \in \repclass} \Exp_{\dist \sim \metadist}\left[\Exp_{(\feat_1,\lab_1),\ldots,(\feat_n, \lab_n) \sim \dist^n}\left[\min_{\per \in \perclass}\frac{1}{n}\sum_{i=1}^n \ind \{\per(\rep(\feat_i))\neq \lab_i\}\right]\right] +\frac{\varepsilon}{3} \\
    &\leq
    \min_{\rep \in \repclass} \Exp_{\dist \sim \metadist}\left[\min_{\per \in \perclass}\Exp_{(\feat_1,\lab_1),\ldots,(\feat_n, \lab_n) \sim \dist^n}\left[\frac{1}{n}\sum_{i=1}^n \ind \{\per(\rep(\feat_i))\neq \lab_i\}\right]\right] + \frac{\varepsilon}{3}
    \tag{Eq.~\ref{eq:swap_min}}
    \\
    &= \min_{\rep \in \repclass} \rer(\metadist, \rep, \perclass)+\frac{\varepsilon}{3}.
\end{align*}
Consider an arbitrary task distribution $\dist$. Suppose we have a dataset $S$ of $\samples$ labeled sample from $\dist$ where $n$ is the following with a sufficiently large constant in the $O$ notation:
$$\samples \coloneqq O\left( \frac{\VC(\perclass)+\ln(1/\eps_2)}{\eps_2^2}\right)\,.$$ 
Since $n$ is sufficiently large, we have uniform convergence of the empirical error for all $\per \in \perclass$ by Theorem~\ref{fact:pac-vc}. Furthermore, the specialized classifier $\hat{\per}$ minimizing the empirical error over $S$, (i.e. $\hat{\per} = \arg \min_{\per \in \perclass} \er(S, \per \circ\hat{\rep})$) must have low true error as well. Therefore, with probability $1-\frac{\eps}{3}$ over the randomness in $S$, we get:

\begin{align*}
    \rer(\dist, \hat{\rep}, \perclass) \leq \Pr[(\feat, \lab) \sim \dist]{ \hat{\per}(\hat{\rep}(\feat))\neq \lab} \leq \rer(S, \hat{\rep}, \perclass)+\frac{\varepsilon}{3}\,.
\end{align*}
Thus, if we take the expectation over the dataset $S$, we see that:
\begin{align*}
    \rer(\dist, \hat{\rep},\perclass)& =\Exp_{S \sim \dist^\samples}\left[\rer(\dist, \hat{\rep},\perclass)\right] \\
    &= \Exp_{S \sim \dist^\samples}\left[\rer(\dist, \hat{\rep},\perclass)\left|\rer(\dist, \hat{\rep},\perclass) \leq \rer(S, \hat{\rep}, \perclass)+\frac{\varepsilon}{3} \right.\right]\\
    &\quad\times \Pr[S \sim \dist^\samples]{\rer(\dist, \hat{\rep},\perclass) \leq \rer(S, \hat{\rep}, \perclass)+\frac{\varepsilon}{3}}\\
    &\quad+ \Exp_{S \sim \dist^\samples}\left[\rer(\dist, \hat{\rep},\perclass)\left|\rer(\dist, \hat{\rep},\perclass) > \rer(S, \hat{\rep}, \perclass)+\frac{\varepsilon}{3}\right.\right]\\
    &\quad\times \Pr[S \sim \dist^\samples]{\rer(\dist, \hat{\rep},\perclass) >\rer(S, \hat{\rep}, \perclass)+\frac{\varepsilon}{3}}\\
    &\leq \left(\Exp_{S \sim \dist^\samples}\left[\rer(S, \hat{\rep}, \perclass)\right] +\frac{\eps}{3}\right)\cdot 1 +1 \cdot \frac{\varepsilon}{3} \\
    &= \Exp_{S \sim \dist^\samples}\left[\rer(S, \hat{\rep}, \perclass)\}\right] +\frac{2\varepsilon}{3}.
\end{align*}
Now, we take the expectation over $\dist \sim \metadist$ and obtain that:
\begin{align*}
    \rer(\metadist, \hat{\rep}, \perclass) &\leq  \Exp_{\dist \sim \metadist}\left[\Exp_{S \sim \dist^\samples}\left[\rer(S, \hat{\rep}, \perclass)\}\right]\right] +\frac{2\varepsilon}{3}\\
    & = \Exp_{\zeta \sim \cD}\left[\erind_{\hat{\rep}}(\zeta)\right]+\frac{2\varepsilon}{3}\,.
\end{align*}
Note that in the last line above, the $S$ dataset drawn from a random $\dist$ can be viewed as a random data set according to $\cD$. Combining with the bound we have derived earlier for the expected value of $\erind_{\hat{\rep}}$, we get the following that holds with probability at least $1-\delta$:
\begin{align*}
    \rer(\metadist, \hat{\rep}, \perclass) \leq \min_{\rep \in \repclass} \rer(\metadist, \rep, \perclass) + \varepsilon
    \,.
\end{align*}
Hence, the proof is complete. 
\end{proof}

\ktext{
\subsection{Sample and Task Complexity Bounds for General Metalearning}
\label{sec:sample-gen-meta}

For general metalearning we obtain the following corollary of Theorem~\ref{thm: meta-to-mtl} by applying Theorem~\ref{thm:mtl-smpls}. It is important to note that this method does not return a representation $\rep\in \repclass$, but a more general specialization algorithm that uses the datasets of the already seen tasks to learn a classifier for the new task. 

\begin{cor}
\label{cor:met-vc}
    For any $(\repclass, \perclass)$, any $\eps>0$, constant $c>0$ and any $\users, \samples$, $(\repclass, \perclass)$ is $(\eps,2/c)$-meta-learnable in the agnostic case with $\users$ tasks, $\samples$ samples per task and $\samples$ specialization samples when
    \[
    \samples \users = O\left(\frac{\VC(\perclass^{\otimes \users+1}\circ \repclass)\ln(1/\eps)}{\eps^2}\right).
    \]
\end{cor}
}

\section{Bounds for Halfspaces over Linear Representations}
\label{sec:hspaces_linear_bounds_app}
In this section, we provide results on multitask learning and metalearning of linear projections (as representations) and halfspaces (as specialized classifiers):
$$
\repclass_{\datadim,\repdim} = \left\{\rep \mid \rep(\mathbf{\feat}) = \mathbf{B} \mathbf{\feat}, \mathbf{B} \in \R^{\repdim\times \datadim}\right\} \,,
 \text{ and }\hspace{3pt}
\perclass_{\repdim} = \left\{\per\mid  \per(\mathbf{\map}) = \textrm{sign}(\mathbf{a}\cdot \mathbf{\map} - w),  \mathbf{a} \in \R^{\repdim}, w \in \R \right\}\,.$$

\subsection{VC Dimension of \texorpdfstring{$\perclass_{\repdim}^{\otimes \users}\circ \repclass_{\datadim,\repdim}$}{Lg}}
\label{sec:vc_mtl_app}
The general bounds of Lemma~\ref{lem:vc-bounds} give us that the VC dimension of $\perclass_{\repdim}^{\otimes \users}\circ \repclass_{\datadim, \repdim}$ is in the range
$$\max(\repdim \users + \users, \datadim+1) \leq \VC(\perclass_{\repdim}^{\otimes \users}\circ \repclass_{\datadim, \repdim}) \leq \datadim \users + \users.$$ 
We know the VC dimension of the class of composite functions  $\perclass_{\repdim}\circ \repclass_{\datadim,\repdim}$ because $\perclass_{\repdim}\circ \repclass_{\datadim,\repdim}$ and $\perclass_{\datadim}$, the class of $d$-dimensional halfspaces, are the same class.

In Theorem~\ref{thm:mtl-vc-hl} we characterize the VC dimension of class $\perclass_{\repdim}^{\otimes \users}\circ \repclass_{\datadim,\repdim}$ up to a constant. 
The bound we give matches the intuition from counting the number of parameters,
which yields $\datadim\repdim+\repdim\users + \users$ when $\users > \repdim$ and $\datadim\users + \users$ when $\users \leq \repdim$. 

\begin{customthm}{\ref*{thm:mtl-vc-hl-intro}}[Restated]
\label{thm:mtl-vc-hl}
    \thmMtlVcHLText
\end{customthm}

In the proof for the upper bound we use Warren's Theorem, which we state here as a lemma:

\begin{lem}[Warren's Theorem,~\cite{Warren68}]
\label{fct:vc-polynomials}
Let $p_1,\ldots, p_{n}$ be real polynomials in $v$ variables, each of degree at most $\ell \geq 1$. If $n \geq v$, then the number of distinct sequences $\textrm{sign}(p_1(x)), \ldots,$ $\textrm{sign}(p_n(x))$ for all $x$ does not exceed $(4e\ell n/v)^v$. In particular, if $\ell \geq 2$ and $n \geq 8v\log_2 \ell$, then the number of distinct sequences of $+1,-1$ taken by $\textrm{sign}(p_1(x)), \ldots, $ $\textrm{sign}(p_n(x))$ is less than $2^n$.
\end{lem}

\begin{proof}
    We first show that for $\users \leq \repdim$ tasks $\VC\left(\perclass_{\repdim}^{\otimes \users}\circ \repclass_{\datadim,\repdim}\right )  = \datadim \users +\users$. To lower-bound the VC dimension, we will show that there exists a dataset of size $\datadim\users+\users$ which can be shattered by $\perclass_{\repdim}^{\otimes \users}\circ \repclass_{\datadim,\repdim}$. We know that the VC dimension of the class of $\datadim$-dimensional thresholds $\perclass_{\datadim}$ is $\datadim+1$, which means that there exists a dataset $(\mathbf{\feat}_1, \ldots, \mathbf{\feat}_{\datadim+1})$ which can be shattered by $\perclass_{\datadim}$. 
    Consider the dataset $\cup_{j \in [\users]} \{(j,\mathbf{\feat}_1), \ldots, (j, \mathbf{\feat}_{\datadim+1})\}$. This dataset can be shattered by $\perclass_{\repdim}^{\otimes \users}\circ \repclass_{\datadim,\repdim}$. To see this, fix a labeling $\mathbf{\lab} \in \bits^{\users \times (\datadim+1)}$, where $\lab_{j,i}$ is the label of datapoint $(j, \mathbf{\feat}_{i})$. For a $j \in [\users]$ we know that there exist $\mathbf{b}_j \in \R^{\datadim}$ and $w_j \in \R$ such that for all $i \in [\datadim+1]$ we have $\textrm{sign}(\mathbf{b}_j  \mathbf{\feat}_{i}-w_j) = \lab_{j,i}$. Since $ \users \leq \repdim$, we set $\mathbf{B}$ to be the matrix with rows $b_j^T$ for the first $\users$ rows and all zeros everywhere else and $\mathbf{a}_j \in \R^{\repdim}$ to be the one-hot encoding of $j$. We pick $\rep(\mathbf{\feat}) = \mathbf{B}\mathbf{\feat}$ and $\per_j(\mathbf{\map}) = \textrm{sign}(\mathbf{a}_j\mathbf{\map}-w_j)$. Therefore, we have that for all $j \in [\users]$ and $i \in [\datadim+1]$, $\lab_{j,i} = \per_j(h(\mathbf{\feat}_i))$. As a result, 
    \[\VC(\perclass_{\repdim}^{\otimes \users}\circ \repclass_{\datadim,\repdim})\geq \datadim\users+\users. \]

    For the upper bound Lemma~\ref{lem:vc-bounds} says that $\VC(\perclass_{\repdim}^{\otimes \users}\circ \repclass_{d,k})\leq \users\VC(\perclass_{\repdim}\circ \repclass_{\datadim,\repdim})$. Furthermore, since the composition of linear functions is linear the class of composite functions $\perclass_{\repdim}\circ \repclass_{\datadim,\repdim}$ and $\perclass_{\datadim}$ are the same. Hence,
    \[
    \VC(\perclass_{\repdim}^{\otimes \users}\circ \repclass_{\datadim,\repdim})\leq \users\VC(\perclass_{\datadim}) = \datadim\users+\users.
    \] 
    
    We now look at the case where we have more than $\repdim$ tasks. For the upper bound we rewrite functions $\conc \in \perclass_{\repdim}^{\otimes \users}\circ \repclass_{\datadim,\repdim}$ as $\conc(j,\mathbf{\feat}) = \textrm{sign}(\mathbf{a}_j\mathbf{B}\mathbf{\feat}-w_j)$.
    Observe that this is equivalent to
    \begin{align*}
        g(j, \mathbf{\feat}) = \textrm{sign}(\mathbf{e_j}^T \mathbf{A} \mathbf{B} \mathbf{\feat} - \mathbf{e_j}^T\mathbf{w}),
    \end{align*}
    where $\mathbf{e}_j \in \{0,1\}^{\users}$ is the one-hot encoding of $j$ and 
    \[
    \mathbf{A} = \begin{pmatrix}
        \mathbf{a}_1^T \\
        \vdots \\
        \mathbf{a}_{\users}^T
    \end{pmatrix} \textrm{ and } \mathbf{w} = \begin{pmatrix}
        w_1 \\
        \vdots \\
        w_{\users}
    \end{pmatrix}.
    \]
    Every combination of $\mathbf{A} \in \R^{\users\times \repdim}, \mathbf{B} \in R^{\repdim \times \datadim}$ and $\mathbf{w} \in \R^{\users}$ gives us a specific labeling function $\conc$. 
    Let $\samples \geq 8(\users \repdim + \repdim \datadim+\users)$. 
    Take a dataset $((j_1,\mathbf{\feat}_1),\ldots,(j_n, \mathbf{\feat}_{\samples})) \in ([\users] \times \R^{\datadim})^{\samples}$. %, where $\mathbf{e}_i$ is the one-hot encoding of $j_i$ for all $i \in [\samples]$. 
    We will show that $\perclass_{\repdim}^{\otimes \users}\circ \repclass_{\datadim,\repdim}$ does not shatter this data set.
    For each $i\in[\samples]$, we define a polynomial $p_i(\mathbf{A},\mathbf{B}, \mathbf{w}) = \mathbf{e}_{j_i}^T \mathbf{A} \mathbf{B}\mathbf{\feat}_i - \mathbf{e}_{j_i}^T\mathbf{w}$.
    Each of these is a degree-2 polynomial in $\users \repdim + \repdim \datadim+\users$ variables. By construction, $g(j_i,\mathbf{x}_i) =\sgn(p_i(\mathbf{A},\mathbf{B},\mathbf{w}))$. By Lemma~\ref{fct:vc-polynomials}, $\VC\left(\perclass_{\repdim}^{\otimes \users}\circ \repclass_{\datadim,\repdim}\right ) \leq 8 (\datadim \repdim +\repdim\users + \users)$.
    
    By Lemma~\ref{lem:vc-bounds} we have that $\VC(\perclass_{\repdim}^{\otimes \users}\circ \repclass_{\datadim,\repdim}) \geq \repdim \users + \users$. Additionally, for $\users > \repdim$ any dataset that is shattered by $\perclass_\repdim^{\otimes \repdim}\circ \repclass_{\datadim,\repdim}$ can be shattered by $\perclass_\repdim^{\otimes \users}\circ \repclass_{\datadim,\repdim}$. As a result, $\VC(\perclass_\repdim^{\otimes \users}\circ \repclass_{\datadim,\repdim}) \geq \VC(\perclass_\repdim^{\otimes \repdim}\circ \repclass_{\datadim,\repdim})$. We proved above that $\VC(\perclass_{\repdim}^{\otimes \repdim}\circ \repclass_{\datadim,\repdim}) = \datadim \repdim + \repdim$.  Therefore, 
    \[
        \VC(\perclass_{\repdim}^{\otimes \users}\circ \repclass_{\datadim,\repdim}) \geq \Omega (\datadim \repdim + \repdim \users). 
        %\qedhere
    \]
\end{proof}
\begin{comment}
\ktext{
Recall that we can reduce metalearning to (general) multitask learning and obtain the task and sample complexity in Corollary~\ref{cor:met-vc}. Therefore, the characterization of the VC dimension of class $\perclass_{\repdim}^{\otimes \users}\circ \repclass_{\datadim,\repdim}$ implies the following bound for the task and sample complexity of (general) metalearning.
\begin{cor}
    We can $(2\eps, \eps)$-metalearn $(\repclass_{\datadim,\repdim}, \perclass_{\repdim})$ with $\users$ tasks, $\samples$ samples per task and $\samples$ specialization samples when \[
    \samples \users = O\left( \frac{(\datadim \repdim+\repdim \users) \ln(1/\eps)}{\eps^4}\right) \text{and }\users \geq \repdim.
    \]
\end{cor}
}\end{comment}

\subsection{Bounding the VC Dimension of Boolean Functions of Polynomials} 
Here we restate and prove Lemma~\ref{lem:vc-bool-pol}.
\begin{customlem}{\ref*{lem:vc-bool-pol}}[Restated]
 \label{lem:vc-bool-pol-app}
\lemVCBoolPol
 \end{customlem}

\begin{proof}
    We bound the growth function of $\class$.  
    Fix $n$ datapoints $x_1, \ldots, x_n \in \featdom$. By the definition of hypothesis class $\class$, the sequence $c_{\mathbf{v}}(x_1),$ $ \ldots, c_{\mathbf{v}}(x_n)$ is exactly the sequence
    $$
        g\bigl(\sign (p_{x_1}^{(1)}(\mathbf{v})), \ldots, \sign (p_{x_1}^{(w)}(\mathbf{v}))\bigr),\ldots, g\bigl(\sign(p_{x_n}^{(1)}(\mathbf{v})), \ldots, \sign (p_{x_n}^{(w)}(\mathbf{v}))\bigr)
    $$ 
    By Lemma~\ref{fct:vc-polynomials}, we know that the number of distinct sequences 
    $$
        \sign (p_{x_1}^{(1)}(\mathbf{v})), \ldots, \sign (p_{x_1}^{(w)}(\mathbf{v})),\ldots, \sign(p_{x_n}^{(1)}(\mathbf{v})), \ldots, \sign (p_{x_n}^{(w)}(\mathbf{v}))
    $$ 
    for all $\mathbf{v} \in \R^d$ does not exceed  $(4e\ell w n/d)^d$. The number of distinct outputs of a function is upper bounded by the number of distinct inputs it gets. As we saw above, the number of distinct inputs is at most $(4e\ell w n/d)^d$ and, thus, $|\{(c_{\mathbf{v}}(x_1), \ldots, c_{\mathbf{v}}(x_n))\mid \mathbf{v} \in \R^d\}| \leq (4e\ell wn/d)^d$. Since this holds for fixed $x_1, \ldots, x_n \in \featdom$, the growth function of $\class$ is at most $(4e\ell wn/d)^d$.

    To bound the VC dimension of $\class$ it suffices to show that for $n > 8d\log_2(\ell w)$ points, the growth function $\growth_{\class}(n)$ is less than $2^n$. If $n > 8d\log_2(\ell w)$, then $n > 8 d$ because $\ell \geq 2$ and $ w \geq 1$. We know that when $n/d > 8$, $\log_2(4e) < n/(2d)$ and $\log_2(n/d) < 3n/(8d)$. Therefore, we see that 
    \[
        \log_2\left(\frac{4e\ell w n}{d}\right) = \log_2(4e) + \log_2(\ell w) + \log_2\left(\frac{n}{d}\right) <  \frac{n}{2d}+\frac{n}{8d} + \frac{3n}{8d} = \frac{n}{d}.
    \]
   Combining the above steps, we have proven that for $n> 8\log_2(\ell w)$, $\growth_{\class}(n) \leq (4e\ell wn/d)^d < 2^n$.
\end{proof}

\subsection{VC Dimension of the Realizability Predicate Class}

In this section we provide the full proof of Theorem~\ref{thm:vc-real-lin}.

\begin{customthm}{\ref*{thm:vc-real-lin}}[Restated]
    \thmVCRealLin
    \label{thm:vc-real-lin-app}
\end{customthm}

 Given a dataset $D =\{(\mathbf{\map}_i,\lab_i)\}_{i \in [\samples]}$ of $\samples$ points in $\R^{\repdim}\times \bits$, we define $\mathbf{Z}$ as the matrix whose $i$-th row is the $\repdim+1$-dimensional vector $\mathbf{z}_i' = \lab_i(\mathbf{\map}_i \| 1)$, for $i \in [\samples]$. 
 For $I\subseteq [\samples]$ a set of indices, we define $\mathbf{Z}_I$ to be the matrix whose rows are the vectors $\mathbf{z}_i'$ for $i$ in $I$.

To prove Theorem~\ref{thm:vc-real-lin-app}, we establish a set of conditions on a data set which allow us to check for separability. 
These conditions can be expressed via a limited number of low-degree polynomials, which allows us to apply Warren's Theorem.

Our first lemma in this section equates separability of a data set with the existence of a special subset of points.
If these points are ``on the margin'' of a linear separator defined by vector $\mathbf{a}$, then we know that $\mathbf{a}$ is a strict separator.
 \begin{lem}
    Dataset $D=\{(\mathbf{\map}_i,\lab_i)\}_{i \in [\samples]}$ is strictly separable if and only if there exists a subset $I$ of linearly independent rows of $\mathbf{Z}$ such that for all $\mathbf{a} \in \R^{\repdim+1}$ if 
    % for all $i \in I, \mathbf{\map}_i' \cdot\mathbf{a} = 1$, then for all $i \in [\samples]$, $\mathbf{\map}_i' \cdot \mathbf{a} \geq 1$.
     $\mathbf{\map}_i' \cdot\mathbf{a} = 1$ for all $i \in I$, then $\mathbf{\map}_i' \cdot \mathbf{a} \geq 1$ for all $i \in [\samples]$.
    \label{lem:strict-sep}
\end{lem}

\begin{proof}
    By definition, $D$ is strictly separable iff there exists a linear separator $\mathbf{a} \in \R^{\repdim+1}$ such that for all $i \in [\samples]$ we have $\lab_i(\mathbf{\map}_i \| 1)\cdot \mathbf{a} = \mathbf{\map}_i' \cdot \mathbf{a} \geq 1$. We define $A = \{\mathbf{a}\mid \mathbf{\map}_i' \cdot \mathbf{a} \geq 1, \forall i \in [\samples]\}$ as the set of all linear separators of $D$ and, for any $\mathbf{a}$, $I_{\mathbf{a}}$ as the set of tight constraints for $\mathbf{a}$, i.e., $I_{\mathbf{a}} = \{i\mid \mathbf{\map}_i' \cdot \mathbf{a}= 1\}$. 
    Note that $A$ and $I_{\mathbf{a}}$ might be empty.

    $\Leftarrow)$ 
    Assume that $I\subseteq [\samples]$ defines a subset of linearly independent rows such that for all $\mathbf{a} \in \R^{\repdim+1}$ if $\mathbf{\map}_i' \cdot\mathbf{a} = 1$ for all $i\in I$, then $\mathbf{\map}_i' \cdot \mathbf{a} \geq 1$ for all $i \in [\samples]$. 
    We find a vector $\mathbf{a}$ that makes the constraints in $I$ tight, that is, solve the linear system $\mathbf{Z}_I a = \mathbf{1}$, where $\mathbf{1}$ is the $|I|$-dimensional all-ones vector.
    Since $\mathbf{Z}_I$ has linearly independent rows, this system has at least one solution 
    $\hat{\mathbf{a}} = \mathbf{Z}_I^+ \mathbf{1}$, where $\mathbf{Z}_I^+ = \mathbf{Z}_I^T(\mathbf{Z}_I\mathbf{Z}_I^T)^{-1}$ is the pseudoinverse.
    Thus, for all $i\in I$ we have $\mathbf{\map}_i'\cdot \hat{\mathbf{a}}=1$, which implies $\mathbf{\map}_i' \cdot \hat{\mathbf{a}} \geq 1$ for all $i \in [\samples]$ by our assumption on $I$.
    Therefore, the points in $D$ are strictly separable. 
    
    $\Rightarrow$)
    Assume that $D$ is strictly separable.
    Then, by Lemma~\ref{lem:full_rank_separator}, there exists a strict separator $\mathbf{a}$ such that $\mathrm{rank}(\mathbf{Z}_{I_{\mathbf{a}}})=\mathrm{rank}(\mathbf{Z})$ and $\mathbf{Z} \cdot \mathbf{a} \ge \mathbf{1}$, where the inequality holds entry-wise.
    If needed, we can remove indices from $I_{\mathbf{a}}$ to produce $I$, a subset with the same rank but linearly independent rows.
    
    Now suppose that $\mathbf{a}'$ satisfies $\mathbf{\map}_i' \cdot \mathbf{a}' = 1$ for all $i\in I$.
    In other words, $\mathbf{a}'$ satisfies $\mathbf{Z}_{I} \cdot \mathbf{a}'= \mathbf{1}$.
    This means we can write $\mathbf{a}'=\mathbf{a}+\mathbf{u}$, where $\mathbf{u}$ is in the right nullspace of $\mathbf{Z}_{I}$.
    Since $\mathbf{Z}$ and $\mathbf{Z}_{I}$ share a rowspace, they also share a right nullspace.
    Thus $\mathbf{Z} \cdot \mathbf{a}' = \mathbf{Z} \cdot \mathbf{a} \ge \mathbf{1}$, where the inequality holds entry-wise.
    This completes the proof.
\end{proof}

The proof of lemma~\ref{lem:strict-sep} uses the following lemma, which says that every strictly separable data set admits a separator whose set of tight constraints is full rank.
\begin{lem}\label{lem:full_rank_separator}
    For a data set $D=\{(\mathbf{\map}_i,y_i)\}_{i\in[\samples]}$, let $\mathbf{Z}$ be its associated matrix and, for a vector $\mathbf{a}$, let $I_{\mathbf{a}}\subseteq [\samples]$ be the set of tight constraints (i.e., the largest set $I$ such that $\mathbf{Z}_I \cdot \mathbf{a} = \mathbf{1}$).
    If $D$ is strictly separable, then there exists a vector $\mathbf{a}^*$ such that $\mathbf{\map}_i\cdot \mathbf{a}^*\ge 1$ for all $i\in [\samples]$ and $\mathrm{rank}(\mathbf{Z})=\mathrm{rank}(\mathbf{Z}_{I_{\mathbf{a}^*}})$.
\end{lem}
\begin{proof}
    \newcommand{\iao}{I_{\mathbf{a}_0}}
    \newcommand{\ziao}{\mathbf{Z}_{\iao}}

    $D$ is strictly separable, so by definition there exists a vector $\mathbf{a}_0$ such that $\mathbf{\map}_i \cdot \mathbf{a}_0 \ge 1$ for all $i\in [\samples]$.
    Suppose $\mathrm{rank}(\mathbf{Z}) > \mathrm{rank}(\ziao)$, since otherwise we are done.
     We will construct another separator $\mathbf{a}_1$ that satisfies $\mathrm{rank}(\ziao) < \mathrm{rank}(\mathbf{Z}_{I_{\mathbf{a}_1}})$.
     Since $\mathrm{rank}(\mathbf{Z})$ is finite, repeating this process will yield a separator $\mathbf{a}^*$ with $\mathrm{rank}(\mathbf{Z}_{I_{\mathbf{a}^*}}) = \mathrm{rank}(\mathbf{Z})$.
     
    If $\mathbf{a}$ is a solution to $\ziao \cdot \mathbf{a} = \mathbf{1}$, we can write it as $\mathbf{a}'=\mathbf{a}_0 + \mathbf{u}$, where $\mathbf{u}$ is in the right nullspace of $\ziao$.
    Since the rank of $\ziao$ is strictly less than the rank of $\mathbf{Z}$, there exists a vector $\mathbf{v}$ that lies in the right nullspace of $\ziao$ but \emph{not} in the right nullspace of $\mathbf{Z}$.
    Our separator $\mathbf{a}_1$ will be of the form $\mathbf{a}_0 + c\cdot \mathbf{v}$ for some real value $c$.
    Note that halfspaces of this form keep the constraints in $\iao$ tight: by construction we have $\mathbf{\map}_i' \cdot (\mathbf{a}_0 + c\cdot \mathbf{v}) = 1$ for all $i\in \iao$.

    Let $I^\perp\subseteq \bar{I}_{\mathbf{a}_0}$ be the subset of rows which are not in the right rowspace of $\ziao$.
    This set is nonempty, since $\mathrm{rank}(\ziao)<\mathrm{rank}(\mathbf{Z})$.
    For all $i\in I^\perp$, let $m_i(c) = \langle\mathbf{\map}_i',\mathbf{a}_0\rangle + c\cdot \langle \mathbf{\map}_i', \mathbf{u}\rangle$.
    Each of these is a linear function in $c$ and, by the definition of $\iao$ and the fact that $\mathbf{a}_0$ is a linear separator, we see that $m_i(0)>1$ for all $i\in I^\perp$.

     For each $i\in I^\perp$, there exists an interval $[L_i,R_i]$ such that, if $c\in [L_i,R_i]$, then $m_i(c)\ge 1$, i.e., point $i$ lies on the correct side of $\mathbf{a}_0 + c\cdot \mathbf{v}$.
    Because $m_i(0)>1$ for all $i\notin\iao$, we see that $L_i < 0 < R_i$.
     The intersection of these intervals, $[\max L_i, \min R_i]$, is nonempty.
     For any $c$ in the intersection, $\mathbf{a}_0 + c\cdot \mathbf{v}$ is a strict separator.
     (This holds for $i\in I^\perp$ by construction; if $i\notin I^\perp$ then $i$ is in the rowspace of $\ziao$ and we have $\mathbf{\map}_i' \cdot (\mathbf{a}_0 + c\cdot \mathbf{v}) = \mathbf{\map}_i'\cdot \mathbf{a}_0\ge 1$, as $\mathbf{v}$ is in the nullspace of $\ziao$.)

    It remains to select $c$ so that at least one additional constraint is tight. 
    This is easy: if $\max L_i$ is finite then we have $m_i (\max L_i)=1$.
    The same holds for $\min R_i$.
    Since we know at least one of $\max L_i$ and $\min R_i$ are finite, we can select a finite one as our value of $c^*$.

    We have constructed a vector $\mathbf{a}_1 = \mathbf{a}_0 + c^*\cdot \mathbf{v}$ that strictly separates $D$ and makes constraint $i^*$ tight for some $i^*\in I^\perp$, where $I^\perp$ is the set of constraints not in the rowspace of $\ziao$.
    Thus, $\mathrm{rank}(\ziao)<\mathrm{rank}(\mathbf{Z}_{I_{\mathbf{a}_1}})$.
    This concludes the proof.
\end{proof}

A simple corollary to Lemma~\ref{lem:strict-sep} says that, on separable data sets, the special subset of points allows us to identify a specific separator.
\begin{cor}
   Dataset $D=\{(\mathbf{\map}_i,\lab_i)\}_{i \in [\samples]}$ is strictly separable if and only if there exists a subset of points $I$ for which
\begin{enumerate}
    \item $\mathbf{Z}_I$ is full rank, and
    \item for all $i \in[\samples]$, we have that $\mathbf{\map}_i' \cdot \hat{\mathbf{a}} \geq 1$, where $\hat{\mathbf{a}} = \mathbf{Z}_I^{+}\mathbf{1}_{|I|}$.
\end{enumerate}
\label{cor:sep-cond}
\end{cor}
\begin{proof}
    $\Rightarrow$) 
    Assume $D$ is strictly separable.
    By Lemma~\ref{lem:strict-sep} we know that there exists a subset $I$ of linearly independent rows such that for all $\mathbf{a} \in \R^{\repdim+1}$ if  $\mathbf{\map}_i' \cdot\mathbf{a} = 1$ for all $i \in I$ then  $\mathbf{\map}_i' \cdot \mathbf{a} \geq 1$ for all $i \in [\samples]$. Since the rows in subset $I$ are linearly independent, $Z_I$ is full rank. 
    By construction, $\hat{\mathbf{a}}$ satisfies $Z_I \cdot \hat{\mathbf{a}} = \mathbf{1}_{|I|}$. 
    Thus $\mathbf{z}_i' \cdot \hat{\mathbf{a}}\geq 1$ for all $i\in[\samples]$.
    
    $\Leftarrow$) 
    Assume that such a subset $I$ exists.
    We see that $\hat{\mathbf{a}}$ strictly separates $D$.
\end{proof}

We now show how to express this characterization of separability in the language of polynomials. 
With this lemma in hand, the proof of Theorem~\ref{thm:vc-real-lin-app} will be a direct application of Lemma~\ref{lem:vc-bool-pol-app}, our extension of Warren's Theorem. 
\begin{lem}
   % Given a dataset $D = \{(\feat_i,\lab_i)\}_{i \in [\samples]}$, the $(\samples, \perclass_{\repdim})$-realizability predicate $r_h$ is a Boolean function of $(\samples+1) \cdot 2^{\samples}$ signs of polynomials in the variables of $\rep$ of degree $4(\repdim+1)$.
   Let $w=(\samples+1)\cdot 2^{\samples}$.
   For a data set $D=\{(\feat_i,\lab_i)\}_{i\in[\samples]}$, there exists a Boolean function $g:\{\pm 1\}^w\to\{\pm 1\}$ and a list of polynomials $p_D^1(h),\ldots, p_D^w(h)$, each of degree at most $4(k+1)$, such that we can express the $(\samples, \perclass_{\repdim})$-realizability predicate $r_h$ as
   \begin{align*}
       r_h(D) = g\left( \sign\left(p_D^{(1)}(h)\right),\ldots, \sign\left(p_D^{(w)}(h)\right) \right).
   \end{align*}
   \label{lem:realp-pol}
\end{lem}
\begin{proof}
    A representation $\rep$ induces a labeled dataset in the representation space $D_{\rep} = \{(\rep(\feat_i),\lab_i)\}_{i \in [\samples]}$.
    By Corollary~\ref{cor:sep-cond}, we can check whether $D_{\rep}$ is linearly separable by checking whether any of the $\sum_{i=1}^{\repdim+1}\binom{\samples}{i} \leq 2^{\samples}$ subsets of $D_{\rep}$ of size at most $\repdim+1$ satisfies the two conditions of the corollary. 
    In the remainder of the proof, we fix a subset $I$ and construct a Boolean function $g_I$ over signs of polynomials that checks if $I$ satisfies these conditions.
    This function will use at most $\samples+1$ polynomials, each of degree at most $4(k+1)$.
    The proof is finished by taking $g$ to be the OR of these functions for each subset $I$.

     We write $D_{\rep}$ as a matrix $\mathbf{Z}$ whose $i$-th row is the $\repdim+1$ dimensional vector $\mathbf{z}_i' = \lab_i (h(\feat_i)\|1)$. 
     % For every non-empty subset of at most $\repdim+1$ points $I$ 
     We construct a polynomial $p_I^{(0)}(h)$ that is \emph{negative} iff $Z_I$ is full rank (so that $\sign~p_I^{(0)}=+1$ indicates rank deficiency).
     For each $i\in[\samples]$, we construct a polynomial $p_I^{(i)}(h)$ that, when $Z_I$ is full rank, is nonnegative iff $\map_i'\cdot\hat{\mathbf{a}} \ge 1$,
     % for all $i \in[\samples]$, we have that $\mathbf{\map}_i' \cdot \hat{\mathbf{a}} \geq 1$, 
     where $\hat{\mathbf{a}} = \mathbf{Z}_I^{+}\mathbf{1}_{|I|}$.
     We then take $g_I$ to be
     \begin{align*}
         g_I(h) := \left(\lnot~\sign~p_I^{(0)}(h)\right) \wedge \left(\bigwedge_{i\in[\samples]} \sign ~ p_I^{(i)}(h)\right).
     \end{align*}
    Recall that we interpret $+1$ as logical ``true.''     

    Matrix $\mathbf{Z}_I$ is full rank if and only if $\mathrm{det}(\mathbf{Z}_I\mathbf{Z}_I^T) \neq 0$. Therefore, we set $p_I^{(0)}(h) = -\mathrm{det}(\mathbf{Z}_I\mathbf{Z}_I^T)^2$, which is a polynomial in $\repdim \times d$ variables of degree $4(\repdim+1)$. 
    To see this, observe that each entry in $\mathbf{Z}_I \mathbf{Z}_I^T$ is a polynomial of degree $2$ in the variables of $\rep$. Then, $\mathrm{det}(\mathbf{Z}_I \mathbf{Z}_I^T)$ is a polynomial of degree $2I$. Finally, $\mathrm{det}(\mathbf{Z}_I\mathbf{Z}_I^T)^2$ is a polynomial of degree $4I \leq 4(\repdim+1)$.
    Taking the negation gives us the desired polynomial $p_I^{(0)}$.

    We now check that $\map_i'\cdot \hat{\mathbf{a}} -1 \ge 0$ for a given $i$.
    Let $\Delta = \mathrm{det}(\mathbf{Z}_I\mathbf{Z}_I^T)$, a polynomial of degree $2(\repdim+1)$. 
    Then $\hat{\mathbf{a}} =  \mathbf{Z}_I^T(\mathbf{Z}_I\mathbf{Z}_I^T)^{-1}\mathbf{1}_{|I|} = \frac{\mathbf{Z}_I^T \mathrm{adj}(\mathbf{Z}_I\mathbf{Z}_I^T)\mathbf{1}_{|I|}}{\Delta}$, where $\mathrm{adj}$ denotes the adjugate matrix.
    We have $\map_i'\cdot \hat{\mathbf{a}} -1 \ge 0$ when $\Delta\neq 0$ and $\map_i'\cdot \Delta\hat{\mathbf{a}} - \Delta\ge 0$.
    Each entry in $\Delta \hat{\mathbf{a}}$ is a polynomial of degree $2|I|-1\le 2\repdim+1$, so setting $p_I^{(i)}(h)= \map_i'\cdot \Delta\hat{\mathbf{a}} - \Delta$ we have a polynomial of degree at most $2(\repdim+1)$ that is, when $\mathbf{Z}_I$ is full rank, is nonnegative iff the constraint is satisfied.
    % We will use $\hat{\mathbf{a}}_{\Delta} = \Delta\hat{\mathbf{a}}$, which is a polynomial of degree $2I-1\leq 2\repdim +1$. We check whether for every $i \notin I$, $\mathbf{z}_i'\cdot \hat{\mathbf{a}}_\Delta - \Delta \geq 0$.
    %
    % Thus, for every subset $I$ we check the conjunction of the signs of $m+1$ polynomials of degree $4(\repdim+1)$. We want to see if one of the subsets satisfies this condition. Therefore, we check the value of the OR function of the $2^{\samples}$ outputs. 
\end{proof}
% \gb{The above proof is pretty terse, and I vote to make it longer and easier to follow. But that's probably low-priority.}

\begin{proof}[Proof of Theorem~\ref{thm:vc-real-lin-app}]
    By Lemma~\ref{lem:realp-pol}, we have that the $(\samples, \perclass_{\repdim})$-realizability predicate is a Boolean function of $(\samples+1)\cdot 2^{\samples}$ signs of polynomials in $d\repdim$ variables of degree $4(\repdim+1)$. Lemma~\ref{lem:vc-bool-pol-app} gives us that the VC dimension of the class of realizability predicates $\realpclass_{\samples,\perclass_{\repdim},\repclass_{\datadim, \repdim}}$ is upper-bounded by $8d\repdim \log_2(4(\repdim+1) \samples 2^{\samples}) = O(d\repdim \samples+dk\log(\repdim \samples))$.
\end{proof}

\subsection{Pseudodimension of the Empirical Error Function Class}

We can now also bound the pseudodimension (Definition~\ref{def:pseudodimension}) of the class of empirical error functions (Definition~\ref{def:empirical_error_function}).

\begin{customthm}{\ref*{thm:pdim-acc}}[Restated]
    \thmpdimacc
  \label{thm:pdim-acc-app}
\end{customthm}

Our first step toward proving these results is the following lemma, which follows almost immediately from Corollary~\ref{cor:sep-cond}.
The proof uses the fact that a data set which can be classified with accuracy at least $\alpha$ can be classified perfectly if we change $1-\alpha$ labels. 
 \begin{lem}
    Dataset $D=\{(\mathbf{\map}_i,\lab_i)\}_{i \in [n]}$ can be linearly separated with accuracy $\alpha$ if and only if there exists a subset of points $I$ and a vector $\pmb{\sigma} \in \{\pm 1\}^{n}$ with $\sum_{i=1}^n \frac{\ind\{\sigma_i = +1\}}{n} = \alpha$ such that 
    \begin{enumerate}
        \item $\mathbf{Z}_I$ is full rank
        \item for all $i \in [n]$, we have that $(\mathbf{z}_i'\cdot \hat{\mathbf{a}}_{\pmb{\sigma}})\sigma_i\geq 1$, where $\hat{\mathbf{a}}_{\pmb{\sigma}} = Z_I^+\pmb{\sigma}_I$.
    \end{enumerate}
    \label{lem:acc-a}
\end{lem}
\begin{proof}
  % \gtext{By definition, a} dataset can be classified with accuracy $\alpha$ if  separate it after we flip $\alpha$ fraction of the labels. 
  By definition, dataset $D$ can be linearly separated with accuracy $\alpha$ if and only if there exists a vector $\pmb{\sigma} \in \{\pm 1\}^{n}$ with $\sum_{i=1}^n \frac{\ind\{\sigma_i = 1\}}{n} = \alpha$ such that dataset $D_{\sigma} = \{(\mathbf{\map}_i,\sigma_i\lab_i)\}_{i \in [n]}$ is strictly separable. Now, we can apply Corollary~\ref{cor:sep-cond} to dataset $D_{\sigma}$. Let $\hat{\mathbf{Z}}$ be the matrix whose $i$-th row is the vector $\hat{\mathbf{z}}'_i = \sigma_i y_i(\mathbf{z}_i\|1)$. 
  The rank of $\hat{\mathbf{Z}}_I$ remains the same as the rank of $\mathbf{Z}_I$ since every row is a scalar multiple of the corresponding row in $Z_I$. % by the corresponding $\sigma_i$. 
  Additionally, observe that $Z_I^+\pmb{\sigma}_I = \hat{\mathbf{Z}}_I^+ \mathbf{1}_{|I|}$.
  % Additionally, solving the linear system $\hat{\mathbf{Z}}_I \mathbf{a} = \mathbf{1}_{|I|}$ is equivalent to solving $\mathbf{Z}_I \mathbf{a} = \pmb{\sigma}_I$ because for all $i$ we have $\sigma_i^2 = 1$. This concludes our proof.
\end{proof}

The next lemma shows how to express statements about the empirical error function in the language of low-degree polynomials.
\begin{lem}
   Given a dataset $D = \{(\feat_i,\lab_i)\}_{i \in [n]}$ and an $\alpha \in [0,1]$, the predicate $\ind_{\pm}\{q_h(x_1,y_1,\ldots, x_n,y_n) =\alpha\}$ is a Boolean function of $(n+1) \cdot 2^{2n}$ signs of polynomials in the variables of $\rep$ of degree $4(\repdim+1)$.
   \label{lem:empacc-pol}
\end{lem}
\begin{proof}
By Lemma~\ref{lem:acc-a} we have seen that $\ind_{\pm}\{q_h(x_1,y_1,\ldots, x_n,y_n)\ktext{=}
%\geq 
\alpha\} = \ind_{\pm}\{\exists \pmb{\sigma} \in \{\pm 1\}^{n}$ with $\sum_{i=1}^n \ind\{\sigma_i = +1\} = \alpha n \textrm{ s.t. } r_h(\feat_1,\sigma_1 y_1, \ldots, \feat_n, \sigma_n y_n) = +1\}$. Therefore, by Lemma~\ref{lem:realp-pol} for every value of $\pmb{\sigma}$ that has $\alpha$ fraction of ones we can check a Boolean function of $(n+1)\cdot 2^n$ signs of polynomials in the variables of $h$ of degree $4(k+1)$. 
There are no more than $2^n$ such vectors $\mathbf{\sigma}$.
In total, we need to evaluate $(n+1)\cdot 2^{2n}$ signs of polynomials.
\end{proof}

We are now ready to prove the final result in this section.
\begin{proof}[Proof of Theorem~\ref{thm:pdim-acc-app}]
% Let $Q_{n, \perclass, \repclass, \alpha}=\{g \mid g(\feat_1,\lab_1,\ldots,\feat_n, \lab_n) = \ind\{q_h(\feat_1, \lab_1, \ldots, \feat_n, \lab_n) = \alpha\}\}$. By \Cref{lem:empacc-pol} and \Cref{lem:vc-bool-pol} we know that for a fixed $\alpha \in [0,1]$ the VC dimension of class $Q_{n, \perclass, \repclass, \alpha}$ is $O(dkn +dk\log_2(nk))$. We notice that $\alpha$ can take $n+1$ different values. Thus, the pseudodimension of $Q_{n, \perclass, \repclass}$ is $O(dkn^2 +dkn\log_2(nk))$. 

    A well-known fact about pseudodimension (see, e.g.,~\cite{AnthonyB1999}) is that it equals the VC dimension of the class of subgraphs. 
    As a function $q_h\in Q_{n,\perclass,\repclass}$ maps data sets $D=\{(x_i,y_i)\}_{i\in[n]}$ to the interval $[0,1]$ (corresponding to error), the object $(D,\tau)$ lies in the subgraph of $q_h$ if $q_h(D) \ge \tau$.
    In our case,
    \begin{align}
        \mathrm{PDim}(Q_{n,\perclass,\repclass}) = \mathrm{VC}\left( \{ \ind_{\pm}\{q_h(D)\ge \tau\} \mid q_h \in Q_{n,\perclass,\repclass}\} \right).
    \end{align}
    To use the tools we have previously established, we will show that, for a fixed $(D,\tau)$, we can write the indicator $\ind_{\pm}\{q_h(D) \ge \tau\}$ as a Boolean function of signs of polynomials.

    But this is easy: the indicator is +1 exactly when there exists an $\alpha\ge \tau$ such that $q_h(D) = \alpha$, which we analyzed in Lemma~\ref{lem:empacc-pol}.
    We take an OR over the $\samples+1$ possible values of $\alpha \in \{0,1/\samples ,\ldots, 1\}$:
    \begin{align*}
        \ind_{\pm}\{q_h(D)\ge \tau\} 
            = \bigwedge_{\alpha}\bigl( \ind_{\pm}\{q_h(D) = \alpha~\text{and}~\alpha\ge \tau\}\bigr).
    \end{align*}
    (Recall that we interpret $+1$ as logical ``true.'')
    When $\alpha <\tau$, we can represent this as $\sign(-1)$, a degree-0 polynomial.
    When $\alpha\ge \tau$, we apply Lemma~\ref{lem:empacc-pol} to see that each term can be written as a Boolean function of $(n+1)\cdot 2^{2n}$ signs of polynomials in the variables of $h$ of degree $4(k+1)$.
    Together, we see that $\ind_{\pm}\{q_h(D)\ge\tau\}$ can be written as a Boolean function of at most $(n+1)^2 2^{2n}$ signs of polynomials, each of (at most) the same degree.
    By Lemma~\ref{lem:vc-bool-pol}, the VC dimension (and thus the pseudodimension of $Q_{n,\perclass,\repclass}$) is at most $8 dk \log_2 ( 4(k+1) \cdot (n+1)^2 2^{2n}) =O(dkn + dk\log (kn))$.
\end{proof}

\fi
\section{Bounds for Finite Specialized Classifiers and Representations} \label{sec:finite_classes}
In this section, we study multitask learning and metalearning for a finite class of representations $\repclass$ and a finite class of specialized classifiers $\perclass$. Specifically, we use our results from \Cref{sec:mtl} and \Cref{sec:meta} to provide bounds for the number of tasks and the number of samples per task needed to multitask learn and metalearn $(\repclass, \perclass)$. 
Recall that the VC dimension of a finite class $\cG$ is at most $\log_2|\cG|$.

By Theorem~\ref{thm:mtl-smpls}, we can $(\eps, \delta)$-multitask learn $(\repclass, \perclass)$ when the number of tasks $\users$ and the number of samples per task $\samples$ satisfy $\samples \users= O(\frac{\VC(\perclass^{\otimes \users}\circ \repclass)\cdot \ln(1/\eps) +\ln(1/\delta)}{\eps})$ in the realizable case and $\samples \users= O(\frac{\VC(\perclass^{\otimes \users}\circ \repclass)\cdot\ln(1/\delta)}{\eps^2})$ in the agnostic case.
For finite classes, we can bound the VC dimension of $\perclass^{\otimes \users}\circ \repclass$ directly by counting:
\begin{equation*}
    \VC(\perclass ^{\otimes \users}\circ \repclass) \leq \log_2|\repclass| + \users \log_2|\perclass|.
\end{equation*}
(For specific classes, the VC dimension can be significantly smaller.)

% In metalearning, we use the VC dimension of the realizability predicate class $\realpclass_{\samples, \perclass, \repclass}$ and the pseudodimension of the empirical error function class $\erindclass_{\samples, \perclass, \repclass}$ to compute the number of tasks we can use to learn in the realizable and the agnostic setting respectively. Since these two function classes have one function per representation $\rep \in \repclass$, when the size of the representation class is finite, they also have finite size of at most $|\repclass|$. 
% In the realizable case we have two results, one where the number of samples per task depends on the non-realizability certificate complexity and one where it depends on the VC dimension. For finite classes of specialized classifiers the non-realizability certificate complexity is upper bounded by the size of the class, but we do not know if it is smaller or larger than the VC dimension (see \Cref{sec:nrcc}). In \Cref{cor:met-fin-real1} we present the two results using the general upper bounds we have for the two quantities. The first result is useful for $\eps$ small enough so that $|\perclass|$ is smaller than $O(\frac{\log_2|\perclass|\ln(1/\eps)}{\eps})$.

To apply our general metalearning theorems, we need to bound the VC dimension of the realizability predicate class $\realpclass_{\samples, \perclass, \repclass}$, the pseudodimension of the empirical error function class $\erindclass_{\samples, \perclass, \repclass}$, and the non-realizability certificate complexity $\NRC(\perclass)$. 
All of these bounds are direct since the underlying classes are finite.

Corollaries~\ref{cor:met-fin-real1} and \ref{cor:met-fin-agn} present our results for metalearning in the realizable and agnostic settings, respectively.
% Note that for finite classes of specialized classifiers $\perclass$, the best possible general bound on $\NRC(\perclass)$ is $|\perclass|$ (\Cref{lem:nrcc-to-vc}).
Note that for finite classes of specialized classifiers $\perclass$, the best possible bound on $\NRC(\perclass)$ that depends only on $|\perclass|$ is simply $\NRC(\perclass)\le |\perclass|$. 
(Lemma~\ref{lem:nrcc-to-vc} shows that this is tight, but for many classes we have $\NRC(\perclass)\ll |\perclass|$.)
Since the number of tasks depends exponentially on  $\NRC(\perclass)$, the first result in Corollary~\ref{cor:met-fin-real1} is only useful for $\eps$ small enough so that $|\perclass|$ is smaller than $O(\frac{\log_2|\perclass|\ln(1/\eps)}{\eps})$. 
%\gb{I don't follow this last claim, since $|F|$ shows up in both $t$ and $n$.}\adaminline{I also find this confusing, and maybe incorrect. (1) The Corollary has two parts, one of which applies more generally; (2) are we sure that the best possible bound on NRC is $|\perclass|$? \Cref{lem:nrcc-to-vc} shows a class with $\NRC =|\perclass|$, but that class can be extended to a class with much smaller NRC and roughly the same size. }
%\gb{By ``first result,'' we mean the first line of Cor \ref{cor:met-fin-real1}.}

\begin{cor}
    \label{cor:met-fin-real1}
    If $\repclass, \perclass$ are finite, we can $(\varepsilon, \delta)$-metalearn $(\repclass, \perclass)$ in the realizable case with $\users$ tasks and $\samples$ samples per task when 
    
     $$\users  
    =\left(\log_2 |\repclass| +\ln\left(1/\delta\right)\right) \cdot O\left(\frac{\ln(1/\eps)}{\eps}\right)^{|\perclass|} \quad\text{and}\quad \samples = |\perclass|
    $$
    \begin{center}
        or
    \end{center}
    $$\users = O\left(\frac{ \log_2|\repclass| \cdot \ln(1/\eps)+\ln(1/\delta)}{\eps}\right)\quad\text{and}\quad \samples = O\left(\frac{\log_2|\perclass| \cdot \ln(1/\eps)}{\eps}\right) .$$
\end{cor}
\begin{proof}
    % Theorems \ref{thm:met-samples} and \ref{thm:met-real-samples} use several measures to characterize the sample and task complexity. 
    For finite $\perclass$, we have $\VC(\perclass) \leq \log_2|\perclass|$ and $\NRC(\perclass) \leq |\perclass|$.
    The second inequality holds because, if a dataset is not realizable by $\perclass$, then for every $\per \in \perclass$ there exists a data point that is not labeled correctly by $\per$. 
    The union of these at most $|\perclass|$ points is also not realizable by $\perclass$. 
    % In \Cref{def:nrcc-to-vc} we showed that there are cases where the non-realizability-certificate complexity of a finite class can be equal to the size of the class. 
    For finite $\repclass$, the class of realizability predicates $\realpclass_{\samples, \perclass, \repclass}$ has at most one realizability predicate per representation. Thus, it has VC dimension $\VC(\realpclass_{\samples, \perclass,\repclass}) \leq \log_2|\repclass|$. 
    The corollary follows by plugging these measures into the results of Theorems \ref{thm:met-samples} and \ref{thm:met-real-samples}.
\end{proof}
\begin{cor}
    \label{cor:met-fin-agn}
    If $\repclass, \perclass$ are finite, we can $(\varepsilon, \delta)$-metalearn $(\repclass, \perclass)$ with $\users$ tasks and $n$ samples per task when 
    $$\users = O\left(\frac{\log_2|\repclass| \cdot \ln(1/\varepsilon)+\ln(1/\delta)}{\eps^2}\right) \quad\text{and}\quad \samples = O\left( \frac{\log_2|\perclass|+\ln(1/\varepsilon)}{\varepsilon^2}\right).$$
\end{cor}
\begin{proof}
    % In the agnostic case result of  we need a bound for the $\VC(\perclass)$ and the $\pd(\erindclass_{\samples, \perclass, \repclass })$. 
    We have that $\VC(\perclass) \leq \log_2|\perclass|$. Additionally, we see that $\pd(\erindclass_{\samples, \perclass, \repclass })\leq \log_2|\repclass|$ because $\pd(\erindclass_{\samples, \perclass, \repclass })  = \mathrm{VC}\left( \{ \ind_{\pm}\{q_h(D)\ge \tau\} \mid q_h \in Q_{n,\perclass,\repclass}\} \right)$ and the size of $\erindclass_{\samples, \perclass, \repclass }$ is at most $|\repclass|$. By plugging these two bounds into Theorem~\ref{thm:agn-met-samples} we obtain the result.
\end{proof}

\end{document}